\documentclass{article}

\usepackage{microtype}
\usepackage{graphicx}
\usepackage{subfig}
\usepackage{booktabs}

\usepackage{hyperref}

\usepackage[accepted]{arxiv}

\usepackage{amsmath}
\usepackage{amssymb}
\usepackage{mathtools}
\usepackage{amsthm}

\usepackage[capitalize,noabbrev]{cleveref}

\usepackage{amsmath}
\usepackage{amssymb}
\usepackage{mathtools}
\usepackage{amsthm}
\usepackage{bbm}

\usepackage[capitalize,noabbrev]{cleveref}

\newtheoremstyle{myexample} 
    {\topsep}                    
    {\topsep}                    
    {\rm }                   
    {}                           
    {\bf }                   
    {.}                          
    {.5em}                       
    {}  

\newtheoremstyle{myremark} 
    {\topsep}                    
    {\topsep}                    
    {\rm}                        
    {}                           
    {\bf}                        
    {.}                          
    {.5em}                       
    {}  

\newtheorem{claim}{Claim}[section]
\newtheorem{lemma}[claim]{Lemma}

\newtheorem{theorem}[claim]{Theorem}
\newtheorem{proposition}[claim]{Proposition}

\newtheorem*{theorem*}{Theorem}

\def\M{\mathcal{M}}
\def\R{\mathbb{R}}

\def\bx{{\boldsymbol x}}
\def\bz{{\boldsymbol z}}
\def\v{{\boldsymbol v}}
\def\u{{\boldsymbol u}}

\def\0{{\boldsymbol 0}}
\def\1{{\boldsymbol 1}}
\def\x{{\boldsymbol x}}
\def\a{{\boldsymbol a}}
\def\b{{\boldsymbol b}}
\def\A{{\boldsymbol A}}
\def\B{{\boldsymbol B}}
\def\D{{\boldsymbol D}}
\def\G{{\boldsymbol G}}
\def\U{{\boldsymbol U}}
\def\V{{\boldsymbol V}}

\def\P{{\boldsymbol P}}

\def\W{{\boldsymbol W}}
\def\X{{\boldsymbol X}}
\def\Y{{\boldsymbol Y}}
\def\J{{\boldsymbol J}}
\def\bR{{\boldsymbol R}}
\def\bE{{\boldsymbol E}}
\def\Z{{\boldsymbol Z}}
\def\M{{\boldsymbol M}}

\def\u{{\boldsymbol u}}
\def\I{{\boldsymbol I}}

\def\1{{\boldsymbol 1}}
\def\m{{\boldsymbol m}}
\def\bSigma{{\boldsymbol \Sigma}}

\def\S{{\boldsymbol{S}}}

\def\r{{\boldsymbol r}}
\def\e{{\boldsymbol e}}
\def\m{{\boldsymbol m}}
\def\v{{\boldsymbol v}}
\def\u{{\boldsymbol u}}
\def\bQ{{\boldsymbol Q}}
\def\bR{{\boldsymbol R}}
\def\bE{{\boldsymbol E}}
\def\m{{\boldsymbol m}}

\def\Vtm{{\left(\V^\top\right)_\m}}
\def\Vtmt{{\left(\V^\top\right)_\m^\top}}

\def\Dh{{\hat{\D}}}
\def\bm{\bar{\b}}
\def\bh{\hat{\b}}
\def\Bm{\bar{\B}}
\def\Bh{\hat{\B}}
\def\Jh{\hat{\J}}

\usepackage{dsfont}
\def\E{{\mathbb{E}}}

\DeclareMathOperator*{\argmin}{arg\,min}
\newcommand{\tr}[1]{\mathrm{Tr}\left[ #1 \right]}

\newcommand{\norm}[1]{\left\lVert #1 \right\rVert}

\newcommand{\abs}[1]{\left\lvert#1\right\rvert}

\newcommand{\expbr}[1]{\mathrm{exp} \left( #1 \right)}

\newcommand{\opn}[1]{\left\lVert #1 \right\rVert_{op}}

\newcommand{\maxn}[1]{\left\lVert #1 \right\rVert_{max}}

\def\Em{{\E_\m}}
\def\alphr{{\alpha_{R}}}
\def\plr{{\log^{\alphr}(d)}}
\renewcommand{\P}[1]{\mathbb{P}\left(#1 \right)}
\DeclareMathOperator*{\Psu}{\mathbb{P}}
\newcommand{\Psub}[2]{\Psu_{#1}\left(#2 \right)}
\def\ber{\mathrm{Bern}(p)}
\def\ball{\mathcal{B}}

\newcommand{\pl}[1]{\mathrm{poly}_{#1}\mathrm{(log(d))}}
\newcommand{\Diag}[1]{\mathrm{Diag}\left(#1\right)}
\newcommand{\Om}[1]{O\left(#1\right)}

\usepackage[textsize=tiny]{todonotes}

\icmltitlerunning{Compression of Structured Data with Autoencoders: Provable Benefit of Nonlinearities and Depth}

\begin{document}

\twocolumn[

\icmltitle{Compression of Structured Data with Autoencoders: \\ Provable Benefit of Nonlinearities and Depth}

\icmlsetsymbol{equal}{*}

\begin{icmlauthorlist}
\icmlauthor{Kevin Kögler}{equal,yyy}
\icmlauthor{Alexander Shevchenko}{equal,yyy}
\icmlauthor{Hamed Hassani}{comp}
\icmlauthor{Marco Mondelli}{yyy}

\end{icmlauthorlist}

\icmlaffiliation{yyy}{ISTA, Klosterneuburg, Austria}
\icmlaffiliation{comp}{Department of Electrical and Systems Engineering, University of Pennsylvania, USA}

\icmlcorrespondingauthor{Kevin Kögler}{kevin.koegler@ist.ac.at}
\icmlcorrespondingauthor{Alexander Shevchenko}{alex.shevchenko@ist.ac.at}

\icmlkeywords{Machine Learning, ICML}

\vskip 0.3in
]

\printAffiliationsAndNotice{\icmlEqualContribution} 

\begin{abstract}
Autoencoders are a prominent model in many empirical branches of machine learning and lossy data compression. However, basic theoretical questions remain unanswered even in a shallow two-layer setting. In particular, to what degree does a shallow autoencoder capture the structure of the underlying data distribution? For the prototypical case of the 1-bit compression of \emph{sparse} Gaussian data, we prove that gradient descent converges to a solution that completely disregards the sparse structure of the input. Namely, the performance of the algorithm is the same as if it was compressing a Gaussian source -- with no sparsity.  For general data distributions, we give evidence of a phase transition phenomenon in the shape of the gradient descent minimizer, as a function of the data sparsity: below the critical sparsity level, the minimizer is a rotation taken uniformly at random (just like in the compression of non-sparse data); above the critical sparsity, the minimizer is the identity (up to a permutation). 
Finally, by exploiting a connection with approximate message passing algorithms, we show how to improve upon Gaussian performance for the compression of sparse data: adding a denoising function to a shallow architecture already reduces the loss provably, and a suitable multi-layer decoder leads to a further improvement.  
We validate our findings on 
image datasets, such as CIFAR-10 and MNIST. 
\end{abstract}

\section{Introduction}\label{section:intro}

Autoencoders have achieved remarkable performance  
in many machine learning areas, such as  generative modeling \cite{kingmaauto},  
inverse problems \cite{peng2020solving} and data compression \cite{balle2017endtoend, theis2017lossy, agustsson2017soft}. Motivated by this practical success, an active area of research is aimed at theoretically analyzing the performance of autoencoders to understand the quality and dynamics of representation learning when these architectures are trained with gradient methods.

Formally, 
consider the encoding of 
$\x\in\mathbb R^d$ given by 
\vspace{-.4em}
\begin{equation}\label{eq:shallow_encoding}
    \bz = \sigma(\B \x), \quad \B \in \mathbb{R}^{n\times d}, \quad \bz \in \mathbb{R}^n,
\end{equation}
where the non-linear activation $\sigma(\cdot)$ is applied component-wise. The ratio $r=n/d$ is referred to as the compression rate. For a shallow (two-layer) autoencoder, the decoding  consists of a single linear transformation $\A \in \mathbb{R}^{d\times n}$:
\vspace{-.4em}
\begin{equation}\label{eq:linear_decoding} 
    \hat{\bx}_{\boldsymbol{\Theta}}(\bx) = \A \bz = \A \sigma(\B \bx).
\end{equation}
The optimal set of parameters $\boldsymbol{\Theta} = \{\A,\B\}$ minimizes the mean-squared error (MSE)
\vspace{-.4em}
\begin{equation}\label{eq:mse}
    \mathcal{R}(\boldsymbol{\Theta}) := d^{-1}\E \left[\|\boldsymbol{x} - \hat{\boldsymbol{x}}_{\boldsymbol{\Theta}}(\boldsymbol{x})\|_2^2\right],
\end{equation}
where the expectation is taken over the data distribution $\x$.
The model described in \eqref{eq:linear_decoding} is a natural extension of \emph{linear} autoencoders ($\sigma(x) = \alpha \cdot x$ for some $\alpha \neq 0$), which were thoroughly studied over the past years \cite{kunin2019loss,gidel2019implicit, bao2020regularized}. In an effort to go beyond the linear setting, a number of recent works have considered the \emph{non-linear} model \eqref{eq:linear_decoding}. Specifically, \citet{pmlr-v162-refinetti22a, nguyen2021analysis} study the training dynamics under specific scaling regimes of the input dimension $d$ and the number of neurons $n$, which lead to either vanishing or diverging compression rates. \citet{shevchenko2023fundamental} focus on the  proportional regime in which $d$ and $n$ grow at the same speed, but their analysis relies heavily on Gaussian data assumptions. In contrast with Gaussian data that lacks any particular structure, real data often exhibits rich structural properties. For instance, images are inherently sparse, and this property has been exploited by various compression schemes such as jpeg. In this view, it is paramount to go beyond the analysis of unstructured Gaussian data and address the following fundamental questions:

\vspace{-.3em}

\begin{center}   
\begin{minipage}{0.48\textwidth}
\textit{Does gradient descent training of the two-layer autoencoder \eqref{eq:linear_decoding} capture the structure in the data? How does increasing the expressivity of the decoder impact the performance?
}
\end{minipage}
\end{center}

 To address these questions, we consider the compression of structured data via the non-linear autoencoder \eqref{eq:shallow_encoding} with $\sigma\equiv \mathrm{sign}$ (1-bit compressed sensing, \cite{boufounos20081}) and show how the data structure is captured by the architecture of the decoder. Let us explain the choice of $\sigma  \equiv \mathrm{sign}$. Apart from the connection to classical information and coding theory \cite{CoverThomas}, its scale invariance prevents the model from entering the \emph{linear regime}. Namely, if $\sigma(\cdot)$ has a well-defined non-vanishing derivative at zero, by picking an encoding matrix $\B$ s.t.\ $\|\B\|_{op}\ll 1$, one can linearize the model, i.e., $\hat{\x}_{\boldsymbol{\Theta}}(\x) \approx \sigma'(0) \cdot \A\B\x$, which results in PCA-like behaviour \cite{pmlr-v162-refinetti22a}. 
Thus, $\mathrm{sign}$ is a natural candidate to tackle the non-linear setting of interest in applications and, in fact, hard-thresholding activations are common in large-scale models \cite{van2017neural}.

Our main contributions can be summarized as follows:
\vspace{-.8em}
\begin{itemize}
    \item Theorem \ref{thm:GD-min-sparse-body} proves that the \emph{linear decoder} in \eqref{eq:linear_decoding} may be \emph{unable to exploit the sparsity} in the data: when $\x$ has a Bernoulli-Gaussian (or ``sparse Gaussian'') distribution, both the gradient descent solution and the MSE coincide with those obtained for the compression of purely Gaussian data (with no sparsity).

 \vspace{-.25em}
   \item Going beyond Gaussian data, we give evidence of the emergence of a \emph{phase transition} in the structure of the optimal matrices $\A, \B$ in 
    \eqref{eq:linear_decoding}, as the sparsity level $p\in (0, 1)$ varies:
    Proposition \ref{proposition:identity_is_better} locates the critical value of $p$ such that the minimizer stops being a random rotation (as for purely Gaussian data), and it becomes the identity (up a permutation); numerical simulations for gradient descent corroborate this phenomenology and display a ``staircase'' behavior of the loss function. 
.

\vspace{-.25em}

\item While for the compression of sparse Gaussian data the linear decoder in \eqref{eq:linear_decoding} does not capture the sparsity, we show in Section \ref{sec:5} that increasing the expressivity of the decoder 
improves upon Gaussian performance. First, we  
post-process  
the output of \eqref{eq:linear_decoding}, i.e., we consider
\vspace{-.5em}
\begin{equation}\label{eq:linear_decoding_denoising}
    \hat{\x}_{\boldsymbol{\Theta}}(\x) = f(\A \boldsymbol{z}) = f(\A \mathrm{sign}(\B\x)),
\vspace{-2mm}\end{equation}
where $f$ 
is applied component-wise, and we prove that a suitable 
$f$ leads to a smaller MSE. In other words, adding a \emph{nonlinearity} to the linear decoder in \eqref{eq:linear_decoding} provably helps. Finally, we  
further improve the performance by increasing the \emph{depth} and using a multi-layer decoder. Our analysis leverages a connection between multi-layer autoencoders and the iterates of the RI-GAMP algorithm proposed by \citet{venkataramanan2022estimation}, which may be of independent interest.

\end{itemize}
\vspace{-.5em}
Experiments on syntethic data confirm our findings, and similar phenomena are displayed when running gradient descent to compress CIFAR-10/MNIST images.
Taken together, our results show that, for the compression of structured data, a more expressive decoding architecture provably improves performance. This is in sharp contrast with the compression of unstructured, Gaussian data where, as discussed in Section 6 of \cite{shevchenko2023fundamental}, multiple decoding layers do not help.

\section{Related work}\label{section:related_works}

\paragraph{Theoretical results for autoencoders.} The practical success of autoencoders has spurred a flurry of theoretical research, started with the analysis of linear autoencoders: \citet{kunin2019loss} indicate a PCA-like behaviour of the minimizers of the $L_2$-regularized loss; \citet{bao2020regularized} provide evidence that the convergence to the minimizer is slow due to ill-conditioning, which worsens as the dimension of the latent space increases; \citet{oftadeh2020eliminating} study the geometry of the loss landscape; \citet{gidel2019implicit} quantify the time-steps of the training dynamics at which deep linear networks recover features of increasing complexity. More recently, the focus has shifted towards \emph{non-linear} autoencoders. 
\citet{pmlr-v162-refinetti22a} characterize the training dynamics 
via a
system of ODEs when the compression rate $r$ is vanishing. 
\citet{nguyen2021analysis} takes a mean-field view that requires a polynomial growth of the number of neurons $n$ in the input dimension $d$, which results in a diverging compression rate.
\citet{cui2023high} use tools from statistical physics 
to predict the MSE of denoising a Gaussian mixture via a two-layer autoencoder with a skip connection. \citet{shevchenko2023fundamental} consider the compression of Gaussian data with a two-layer autoencoder when the compression rate $r$ is fixed and show that gradient descent methods achieve a minimizer of the MSE.

\vspace{-.5em}

\paragraph{Incremental learning and staircases in the training dynamics.} Phenomena similar to the staircase behavior of the loss function that we exhibit in Figure \ref{fig:sgd_rademacher} have drawn significant 
attention. For  
parity learning,   
the line of works \cite{abbe2021staircase, abbe2022merged, abbe2023sgd} shows that parities are recovered in a 
sequential fashion with increasing complexity. A similar behaviour is observed in transformers with diagonal weight matrices at small initialization \cite{abbe2023transformers}: gradient descent progressively learns a solution of increasing rank.
For a single index model, \citet{berthier2023learning} 
show a separation of time-scales at which the training dynamics follows an alternating pattern of plateaus and rapid decreases in the loss. Evidence of incremental learning in deep linear networks is 
provided by \citet{berthier2023incremental, pesme2023saddle, simon2023stepwise, jacot2021saddle, milanesi2021implicit}. The recent work by \citet{szekely2023learning} shows that the cumulants of the data distribution are learnt sequentially, revealing a sample complexity gap between neural networks and random features.

\vspace{-.5em}

\paragraph{Approximate Message Passing (AMP).} AMP refers to a family of iterative algorithms developed for a variety of statistical inference problems \cite{feng2022unifying}. Such problems include the recovery of a signal $\x$ from observations $\bz$ of the form in \eqref{eq:shallow_encoding}, namely,  a Generalized Linear Model \cite{McCullagh_GLM}, when the encoder matrix $\B$ is Gaussian \cite{rangan_GAMP,mondelli-2021-amp-spec-glm} or rotationally-invariant \cite{rangan2019vector,schniter2016vector,ma2017orthogonal,takeuchi2019rigorous}. Of particular interest for our work is the RI-GAMP algorithm by \citet{venkataramanan2022estimation}. In fact, RI-GAMP enjoys a  
computational graph structure that can be mapped to a suitable neural network, 
and it approaches the information-theoretically optimal MSE. The optimal MSE was computed via the replica method by \citet{takeda2006analysis,tulino2013support}, and these predictions were rigorously confirmed for the high-temperature regime by \citet{li2023random}.

\section{Preliminaries}

\vspace{-.25em}

\paragraph{Notation.} We use plain symbols $a,b$ for scalars, bold symbols $\boldsymbol{a}, \boldsymbol{b}$ for vectors, and capitalized bold symbols $\A,\B$ for matrices. Given a vector $\a$, its $\ell_2$-norm is $\|\a\|_2$. Given a matrix $\A$, its operator norm is $\opn{\A}$. 
We denote a unidimensional Gaussian distribution with mean $\mu$ and variance $\sigma^2$ by $\mathcal{N}(\mu,\sigma^2)$. We use the shorthand $\hat{\bx}$ for 
$\hat{\bx}_{\boldsymbol{\Theta}}$. 
Unless specified otherwise, function are applied \emph{component-wise} to vector/matrix-valued inputs. We denote by $C, c>0$ universal constants, which are independent of $n, d$. 

\vspace{-.5em}

\paragraph{Data distribution and MSE.} For $p\in(0,1]$, a sparse Gaussian distribution ${\rm SG}_1(p)$ is equal to $\mathcal{N}(0,1/p)$ with probability $p$ and is $0$ otherwise. The scaling of the variance of the Gaussian component ensures a unit second moment for all $p$. We use the notation $\x\sim {\rm SG}_d(p)$ to denote a vector with i.i.d.\ components distributed according to ${\rm SG}_1(p)$. Decreasing $p$ makes $\x\sim {\rm SG}_d(p)$ more sparse: for $p=1$ one recovers the isotropic Gaussian, i.e., ${\rm SG}_d(1) \equiv \mathcal{N}(\0,\I)$, while $p=0$ implies that $\x = \0$.

\citet{shevchenko2023fundamental} consider Gaussian data $\x\sim {\rm SG}_d(1)$ and the two-layer autoencoder with linear decoder in \eqref{eq:linear_decoding}. Their analysis shows that, for a compression rate $r\leq 1$, the MSE obtained by minimizing \eqref{eq:mse} over $\boldsymbol{\Theta} = \{\A,\B\}$ is given by 
\vspace{-.75em}
\begin{equation}\label{eq:gaussian_val}
  \mathcal R_{\rm Gauss}:= 1 - \frac{2}{\pi} \cdot r. 
\end{equation}
The set of minimizers $(\A,\B)$ has a weight-tied orthogonal structure, 
i.e., $\B\B^\top = \I$ and $\A \propto \B^\top$, and gradient-based optimization schemes reach a global minimum.

\section{Limitations of a linear decoding layer}\label{sec:4}
Our main technical result is that a two-layer autoencoder with a single linear decoding layer does not capture the sparse structure of the data. Specifically, we consider the autoencoder in \eqref{eq:linear_decoding} with Gaussian data $\x\sim {\rm SG}_d(p)$ trained via gradient descent. We show that, when $n, d$ are both large (holding the compression rate $r=n/d$ fixed), the trajectory of the algorithm is the same as that obtained from the compression of non-sparse data, i.e., $\x\sim \rm{SG}_d(1) \equiv \mathcal{N}(\boldsymbol{0}, \I)$. As a consequence, the minimizer has a weight-tied orthogonal structure ($\B\B^\top = \I$, $\A \propto \B^\top$), and the MSE at convergence is given by $\mathcal R_{\rm Gauss}$ as defined in \eqref{eq:gaussian_val}. 

We now go into the details.
Since the optimization objective is convex in $\A$, we consider the following alternating minimization version of Riemannian gradient descent:
\vspace{-.4em}
\begin{equation}\label{eq:body-GDmin-formulas}
    \begin{split}
\A(t+1) &= \argmin_\A \mathcal{R}(\A, \B(t)),\\
\B(t+1) \hspace{-.15em}&:=\mathrm{proj} \left( \hspace{-.1em}\B(t) - \eta \left(\nabla_{\B(t)}+ \G(t) \right) \right).
    \end{split}
\end{equation}
In fact, due to the convexity in $\A$ of the MSE $\mathcal{R}(\cdot,\cdot)$ in \eqref{eq:mse}, we can compute in closed form $\argmin_\A \mathcal{R}(\A, \B(t))$. Here, Riemannian refers to  the space of matrices with unit-norm rows, $\nabla_{\B(t)}$ is a shorthand for the gradient $\nabla_{\B(t)}\mathcal{R}(\A(t),\B(t))$,  and $\mathrm{proj}$ normalizes the rows of a matrix to have unit norm. The projection step (and, hence, the Riemannian nature of the optimization) is due to the scale-invariance of $\mathrm{sign}$, and it ensures numerical stability. The term $\G(t)$ corresponds to
Gaussian noise of arbitrarily small variance, which acts as a (probabilistic) smoothing for the discontinuity of $\mathrm{sign}$ 
at $0$ and, therefore, implies that the gradient is well-defined along the trajectory of the algorithm. (Note that $\G(t)$ is not needed in experiments, as we use a straight-through estimator, see Appendix \ref{app:num-setup}).

\begin{theorem}[Gradient descent does not capture the sparsity]\label{thm:GD-min-sparse-body}
    Consider the gradient descent algorithm in \eqref{eq:body-GDmin-formulas} with $\x\sim {\rm SG}_d(p)$ and $(\G(t))_{i,j} \sim \mathcal N(0, \sigma^2)$, where $d^{-\gamma_g}\leq \sigma \leq C/d$ for some fixed $1<\gamma_g<\infty$. Initialize the algorithm with $\B(0)$ equal to a row-normalized Gaussian, i.e., $\B'_{i, j}(0)\sim \mathcal N(0, 1/d)$, $\B(0)=\mathrm{proj}(\B'(0))$, and let $\B(0)=\U\S(0)\V^\top$ be its SVD. Let the step size $\eta$ be $\Theta(1/\sqrt{d})$. Then, for any fixed $r<1$ and $T_{\rm max} \in (0, \infty)$, with probability at least $1-Cd^{-3/2}$, the following holds for all $t\le T_{\rm max}/\eta$
 \vspace{-.4em}
   \begin{equation}\label{eq:thmstat}
    \begin{split}
        &\B(t) = \U \S(t) \V^\top + \bR(t), \\ 
        &\opn{\S(t)\S(t)^\top - \I} \leq C \exp\left(-c \eta t\right), \\
        &\lim_{d \to \infty} \sup_{t \in [0, T_{\rm max}/\eta]} \opn{\bR(t)} = 0,
    \end{split}
    \end{equation}
    where $C, c$ are universal constants depending only on $p, r$ and $T_{\rm max}$.
    Moreover, we have that, almost surely,
\vspace{-.4em}
    \begin{align}
        \lim_{t\to \infty}&\lim_{d \to \infty}\mathcal{R}(\A(t), \B(t)) = \mathcal{R}_{\rm Gauss},\label{eq:body-appG2}\\
    \lim_{d \to \infty}& \sup_{t \in [0, T_{\rm max}/\eta]} \opn{\B(t)-\B_{\rm Gauss}(t)}=0,\label{eq:body-appG1}
    \end{align}
where  $\mathcal{R}_{\rm Gauss}$ is defined in \eqref{eq:gaussian_val} and $\B_{\rm Gauss}(t)$ is obtained by running \eqref{eq:body-GDmin-formulas} with $\x\sim \mathcal N(\0, \I)$.
\end{theorem}
In words, \eqref{eq:thmstat} gives a precise characterization of the gradient descent trajectory: throughout the dynamics, 
the eigenbasis of $\B(t)$ does not change significantly (i.e., it remains close to that of $\B(0)$) and, as $t$ grows, all the singular values of $\B(t)$ approach $1$. As a consequence, \eqref{eq:body-appG2} gives that, at convergence, the MSE achieved by \eqref{eq:body-GDmin-formulas} with $\x\sim {\rm SG}_d(p)$ approaches $\mathcal{R}_{\rm Gauss}$, which corresponds to the compression of standard Gaussian data $\x\sim \mathcal N(\0, \I)$.
In fact, a stronger result holds: \eqref{eq:body-appG1} gives that the whole trajectory of \eqref{eq:body-GDmin-formulas} for $\x\sim {\rm SG}_d(p)$ is the same as that obtained for 
$\x\sim \mathcal N(\0, \I)$.

\paragraph{Proof sketch.} 
The sparse Gaussian distribution can be seen as the component-wise product between a standard Gaussian vector and a mask $\m\in \{0, 1\}^d$ with i.i.d.\ Bernoulli($p$) entries.
The key idea is to approximate the randomness in the mask $\m$ with its average, which heuristically corresponds to having again Gaussian data. This is done formally by pushing the mask into the network parameters, and then using high-dimensional concentration tools to bound the deviation from the average. 

We now go into the 
details. As a starting point, 
Lemma \ref{lem:mmse-to-matrix-obj} shows 
that, up to an error 
exponentially small in $d$, instead of the MSE in \eqref{eq:mse} we can consider the objective 
\vspace{-.5em}
\begin{equation}\label{eq:objreg}
\E_{\m}\left[\tr{\A^\top \A \cdot \arcsin(\hat{\B}_\m\hat{\B}_\m^\top)} - \frac{2}{\sqrt{p}} \cdot \tr{\A \hat{\B}_\m}\right].    
\end{equation}
Here, $\B_\m$ denotes a masked version of $\B$, i.e., the columns of $\B$ are set to zero according to the Bernoulli mask $\m$, and $\hat{\B}_\m$ is obtained by normalizing the rows of $\B_\m$, i.e., $(\hat{\B}_\m)_{i,:} = (\B_\m)_{i,:}/\|(\B_\m)_{i,:}\|_2$. 

Next, we provide a number of concentration bounds for quantities to which the Bernoulli mask $\m$ is applied. We start with random vectors (Lemma \ref{lem:D-concentration}), random matrices (Lemmas \ref{lem:lip-conc} and \ref{lem:lip-conc-conditional}), and quantities that appear when optimizing the objective \eqref{eq:objreg} via gradient descent (Lemma \ref{lem:lip-conc-application}). We note that both the largest entry and the operator norm of the error matrix have to be controlled. Then, we take care of the row normalization in the definition of $\hat{\B}_\m$. To do so, Lemma \ref{lem:master-Bh} is a general result showing that 
\vspace{-.3em}
\begin{equation}\label{eq:Fap}
  \mathbb E_\m  F\left(\hat{\B}_\m\right) \approx \mathbb E_\m F\left(\frac{1}{\sqrt{p}}\B_\m\right),
\vspace{-1mm}\end{equation}
for a class of sufficiently regular matrix-valued functions $F$. In words, 
\eqref{eq:Fap} gives that, on average over $\m$, the row normalization can be replaced by the multiplication with 
$1/\sqrt{p}$. This result is instantiated in Lemma \ref{lem:master-Bh-explicit} for three choices of $F$ useful for the analysis of gradient descent. 

Armed with these technical tools, we are able to 
remove the effect of the masking from the gradient descent dynamics. First, Lemma \ref{lem:A-concentration} focuses on the optimization of the matrix $\A$, which has a closed form due to the convexity of the objective \eqref{eq:objreg} in $\A$. Next, Lemmas \ref{lem:grad-concentration-part1} and \ref{lem:grad-concentration-part2} estimate the gradient $\nabla_{\B(t)}$ as 
\vspace{-.5em}
$$
 \opn{\nabla_{\B(t)}-  \U\tilde{\S}(t)\V^\top} \le C(T_{\rm max}) \cdot \frac{\log^{10}(d)}{\sqrt{d}},
$$
where $\U, \V$ come from the SVD of $\B(0)$, $\S(t)$ is a diagonal matrix containing the singular values of $\B(t)$, and $\tilde{\S}(t) = G(\S(t))$ for a deterministic function $G$. This shows that, up to the leading order in the approximation, the singular vectors of $\B(t)$ are fixed along the gradient trajectory.
Crucially, the function $G$ does not depend on the sparsity $p$ of the data. Thus, for any $p \in (0, 1)$, the gradient update for the masked objective \eqref{eq:objreg} is close to the update for the same objective without the masking (i.e., corresponding to the compression of Gaussian data with $p=1$). 

Finally, Lemma \ref{lem:B-specevo} derives an a-priori Gr\"onwall-type estimate, which bootstraps the bounds to the whole gradient descent trajectory \eqref{eq:body-GDmin-formulas} and concludes the proof. The complete argument is deferred to Appendix \ref{app:GD-min-Proof}. \qed

\begin{figure*}[t!]
    \centering
   \subfloat{\includegraphics[width=0.815\columnwidth]{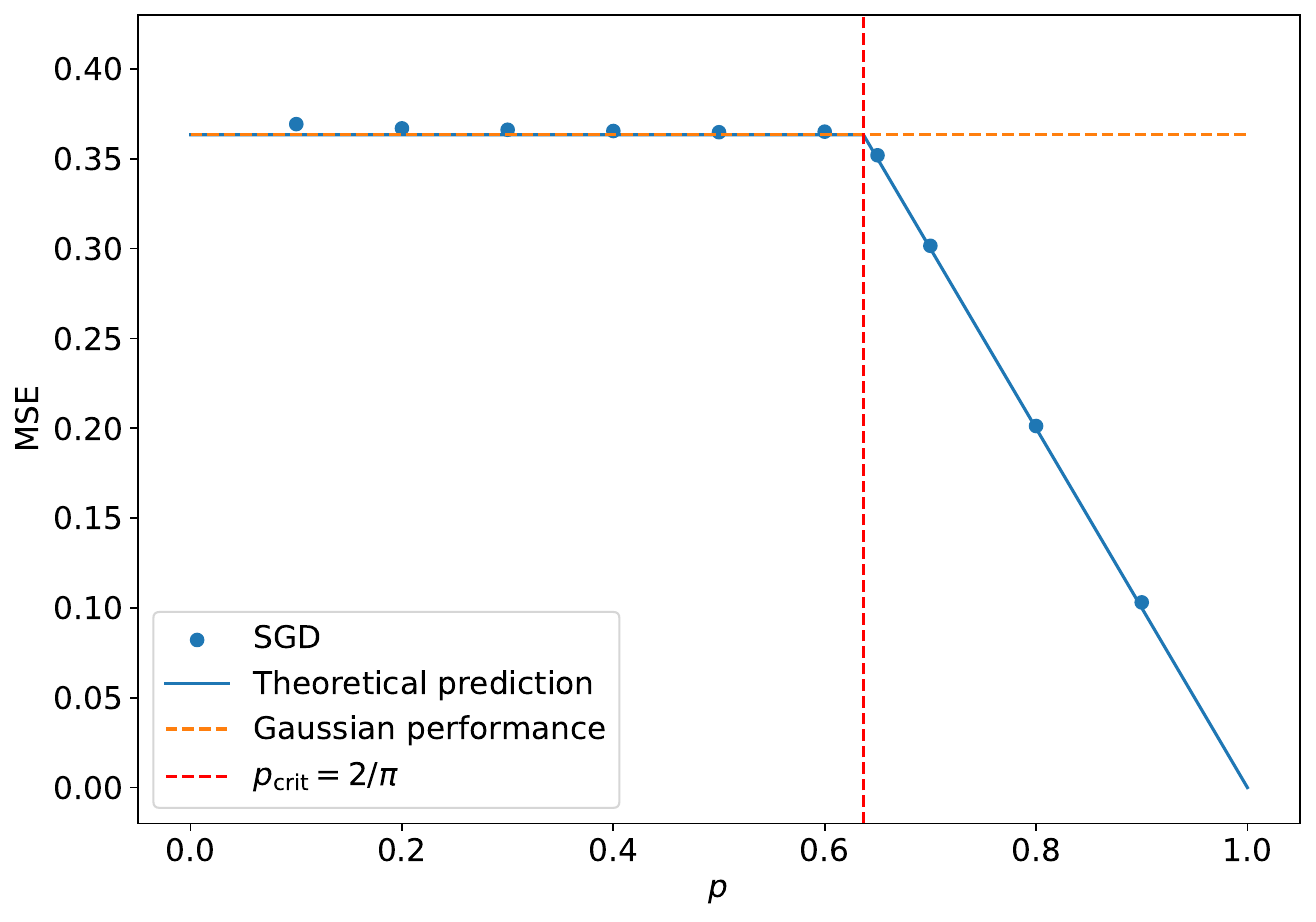}}
    \hspace{-0.em}
    \subfloat{\includegraphics[width=0.58\columnwidth]{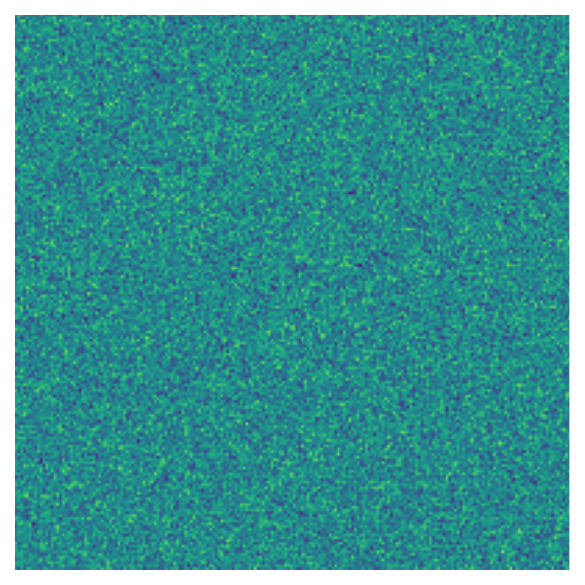}}
    \subfloat{\includegraphics[width=0.58\columnwidth]{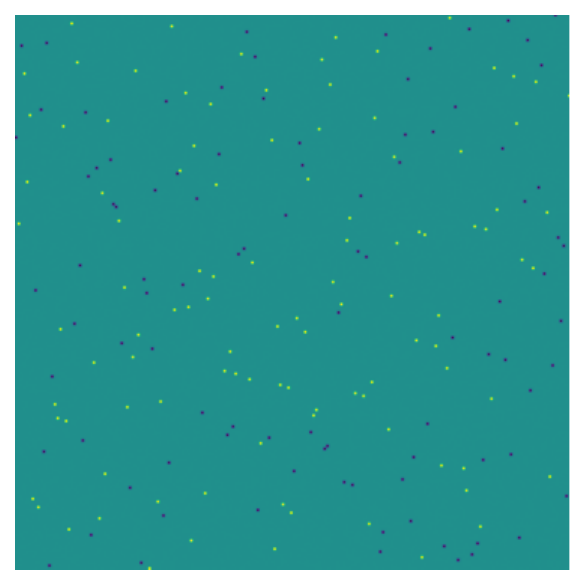}}\hspace{0.2em}
\vspace{-3mm}\caption{Compression of sparse Rademacher data via the two-layer autoencoder in \eqref{eq:linear_decoding}. We set $d=200$ and $r=1$. \emph{Left.} MSE achieved by SGD at convergence, as a function of the sparsity level $p$. The empirical values (dots) match our theoretical prediction (blue line): for $p<p_{\mathrm{crit}}$, the loss is equal to the value obtained for Gaussian data, i.e., $\mathcal R_{\rm Gauss}=1-2r/\pi$; for $p\ge p_{\mathrm{crit}}$, the loss is smaller, and it is equal to $1 - r \cdot \left(\E |x_1|\right)^2=1-r\cdot p$. \emph{Center.} Encoder matrix $\B$ at convergence of SGD when $p=0.3<p_{\mathrm{crit}}$: the matrix is a random rotation. \emph{Right.} Encoder matrix $\B$ at convergence of SGD when $p=0.7\ge p_{\mathrm{crit}}$: the negative sign in part of the entries of $\B$ is cancelled by the corresponding sign in the entries of $\A$; hence, $\B$ is equivalent to a permutation of the identity.  \vspace{-3mm}}\label{fig:rademacher_phase_transition}
\end{figure*}

\vspace{-.5em}

\paragraph{Beyond Gaussian data: Phase transitions, staircases in the learning dynamics, and image data.}

For general distributions of the data $\x$, we empirically observe that the minimizers of the model in \eqref{eq:linear_decoding} found by stochastic gradient descent (SGD) either \emph{(i)} coincide with those obtained for standard Gaussian data, or  \emph{(ii)} are equivalent to (suitably sub-sampled) permutations of the identity. Up to a permutation of the neurons, these two candidates can be expressed as:
\vspace{-.5em}
\begin{align}\label{eq:two_nets}
        &\hat{\x}_{\mathrm{Haar}}(\x) = \alpha_{\mathrm{Haar}} \cdot \B^\top \mathrm{sign}(\B\x), \\ &\hat{\x}_{\mathrm{Id}}(\x) = \alpha_{\mathrm{Id}} \cdot \begin{bmatrix}
        \I_n\\
        \boldsymbol{0}_{(d-n)\times n} 
    \end{bmatrix}\mathrm{sign}([\I_{n}, \boldsymbol{0}_{n \times (d-n)}]\x)\nonumber,
\end{align}
where $\B$ is obtained by subsampling a Haar matrix (i.e., a matrix taken uniformly from the group of rotations), $\boldsymbol{0}_{(d-n)\times n}$ is a $(d-n)\times n$ matrix of zeros, and $(\alpha_{\mathrm{Haar}},\alpha_{\mathrm{Id}})$ are scalar coefficients. The losses of these two candidates can be expressed in a closed form as derived below.

\begin{proposition}[Candidate comparison]\label{proposition:identity_is_better} Let $r\leq1$ and let $\x$ have i.i.d.\ components with zero mean and unit variance. Then, 
we have that, almost surely, the MSE of $\hat{\x}_{\mathrm{Haar}}(\cdot)$ coincides with the Gaussian performance $\mathcal R_{\rm Gauss}$ in \eqref{eq:gaussian_val}, 
i.e., 
\vspace{-.5em}
\begin{equation}\label{eq:Haarloss}    
\min_{\alpha_{\mathrm{Haar}\in\mathbb{R}}}\lim_{d\rightarrow\infty} \frac{1}{d}\cdot\E_{\x 
} \left[\|\hat{\x}_{\mathrm{Haar}}(\x)-\x\|_2^2\right] = 1 - \frac{2}{\pi} \cdot r\ .
\end{equation}
Furthermore, we have that, for any $d$, 
\vspace{-.5em}
\begin{equation}\label{eq:idloss}
\min_{\alpha_{\mathrm{Id}\in\mathbb{R}}} \frac{1}{d}\cdot\E_{\x 
} \left[\|\hat{\x}_{\mathrm{Id}}(\x)-\x\|_2^2\right] =1 - r \cdot \left(\E |x_1|\right)^2,
\end{equation}
where $x_1$ is the first component of $\x$. This implies that $\hat{\x}_{\mathrm{Id}}(\cdot)$ is superior to $\hat{\x}_{\mathrm{Haar}}(\cdot)$ in terms of MSE whenever
\vspace{-.5em}
\begin{equation}\label{eq:phase_transition}
\E |x_1| > \sqrt{2/\pi} =\E_{g\sim\mathcal N(0, 1)}|g|.
\end{equation}
\end{proposition}

The MSE of $\hat{\x}_{\mathrm{Id}}(\cdot)$ in \eqref{eq:idloss} is obtained via a direct calculation. To evaluate the MSE of $\hat{\x}_{\mathrm{Haar}}(\cdot)$ in \eqref{eq:Haarloss}, we relate this estimator to the first iterate of the RI-GAMP algorithm proposed by \citet{venkataramanan2022estimation}. Then, the high dimensional limit of $\| \B^\top\mathrm{sign}(\B\x)-\x\|_2^2$ follows from the state evolution analysis of RI-GAMP. A similar strategy will be used also in Section 
\ref{sec:5} to analyze different decoding architectures. The complete proof is in Appendix \ref{subsec:idbetter}.

\begin{figure}[t]
    \centering
    \subfloat{\includegraphics[width=0.84\columnwidth]{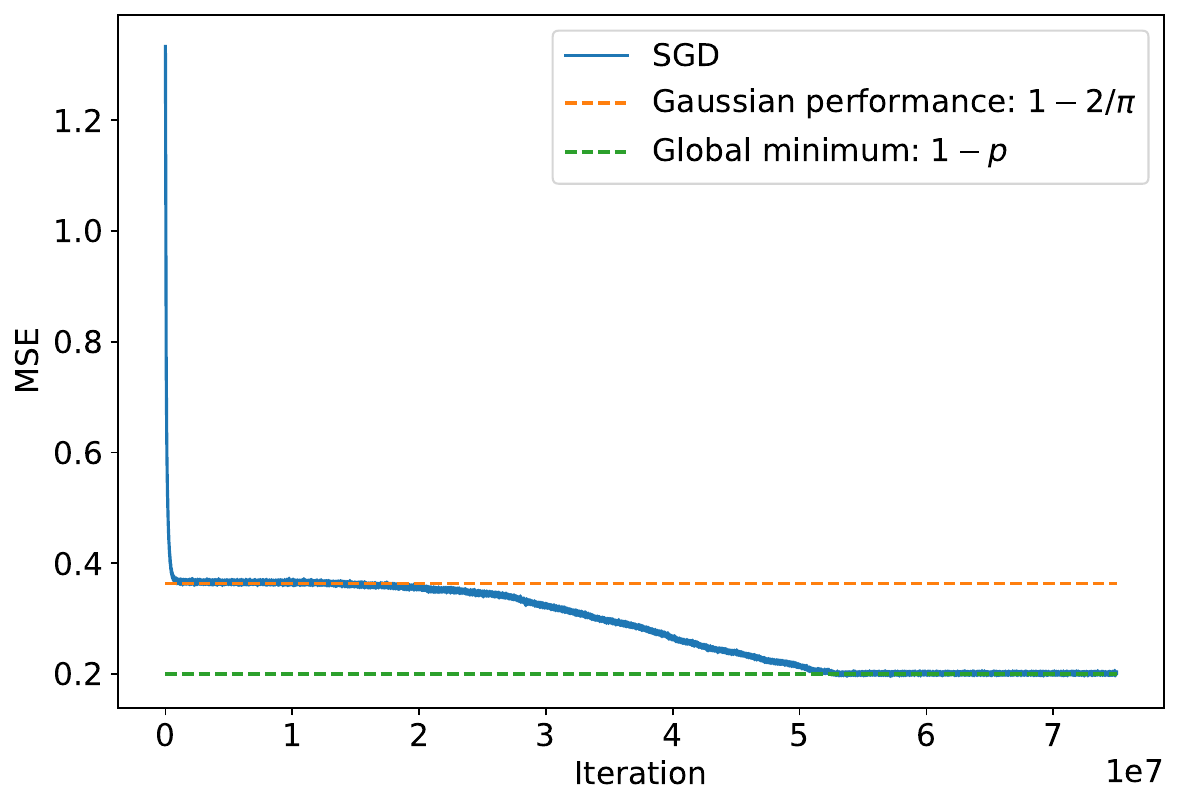}}

\vspace{-4mm}\caption{Compression of sparse Rademacher data via the two-layer autoencoder in \eqref{eq:linear_decoding}. We set $d=200$, $r=1$ and $p=0.8$. The MSE is plotted as a function of the number of iterations and, as $p>p_{\mathrm{crit}}$, it displays a staircase behavior. 
}\vspace{-4.5mm}\label{fig:sgd_rademacher}
\end{figure}

As mentioned above, our numerical results lead us to conjecture that SGD recovers either of the candidates in \eqref{eq:two_nets}, depending on which achieves a smaller loss. Specifically, if condition \eqref{eq:phase_transition} is met, the SGD predictor converges to $\hat{\x}_{\mathrm{Id}}(\cdot)$ and improves upon the Gaussian loss $\mathcal R_{\rm Gauss}$; otherwise, it converges to $\hat{\x}_{\mathrm{Haar}}(\cdot)$ and its MSE is equal to $\mathcal R_{\rm Gauss}$.

For sparse Gaussian data, condition \eqref{eq:phase_transition} is never satisfied, as $\E_{x_1\sim {\rm SG}_1(p)}|x_1| = \sqrt{2p/\pi} \leq \sqrt{2/\pi}$. In fact, as proved in Theorem \ref{thm:GD-min-sparse-body}, the SGD solution approaches $\hat{\x}_{\mathrm{Haar}}(\cdot)$ and its MSE matches $\mathcal R_{\rm Gauss}$.

For sparse Rademacher data\footnote{Each i.i.d.\ component is equal to $0$ w.p.\ $1-p$ and 
to $\pm 1/\sqrt{p}$ w.p.\ $p/2$, which 
ensures a unit second moment for all $p\in [0, 1]$.}, condition \eqref{eq:phase_transition}
reduces to

$
p > p_{\mathrm{crit}} := 2/\pi \approx 0.64,
$

and Figure \ref{fig:rademacher_phase_transition} shows a \emph{phase transition} in the structure of the minimizers found by SGD:
\vspace{-1.25em}

\begin{itemize}
    \item For $p<p_{\mathrm{crit}}$, SGD converges to a solution s.t.\ $\B$ is a uniform rotation (central heatmap) and the MSE is close to $\mathcal R_{\rm Gauss}=1-2r/\pi$, see \eqref{eq:Haarloss}.
    \item For $p > p_{\mathrm{crit}}$, SGD converges to a solution s.t.\ $\B$ is equivalent to a permutation of the identity (right heatmap) and the MSE is close to $1 - r \cdot \left(\E |x_1|\right)^2=1-r\cdot p$, see \eqref{eq:idloss}.  In both cases, $\A \propto \B^\top$. 
\end{itemize}

\begin{figure}[t]
\vspace{-.75em}
\begin{tabular}{@{}cc@{}}
    \raisebox{-\height}{\includegraphics[width=0.35\textwidth]{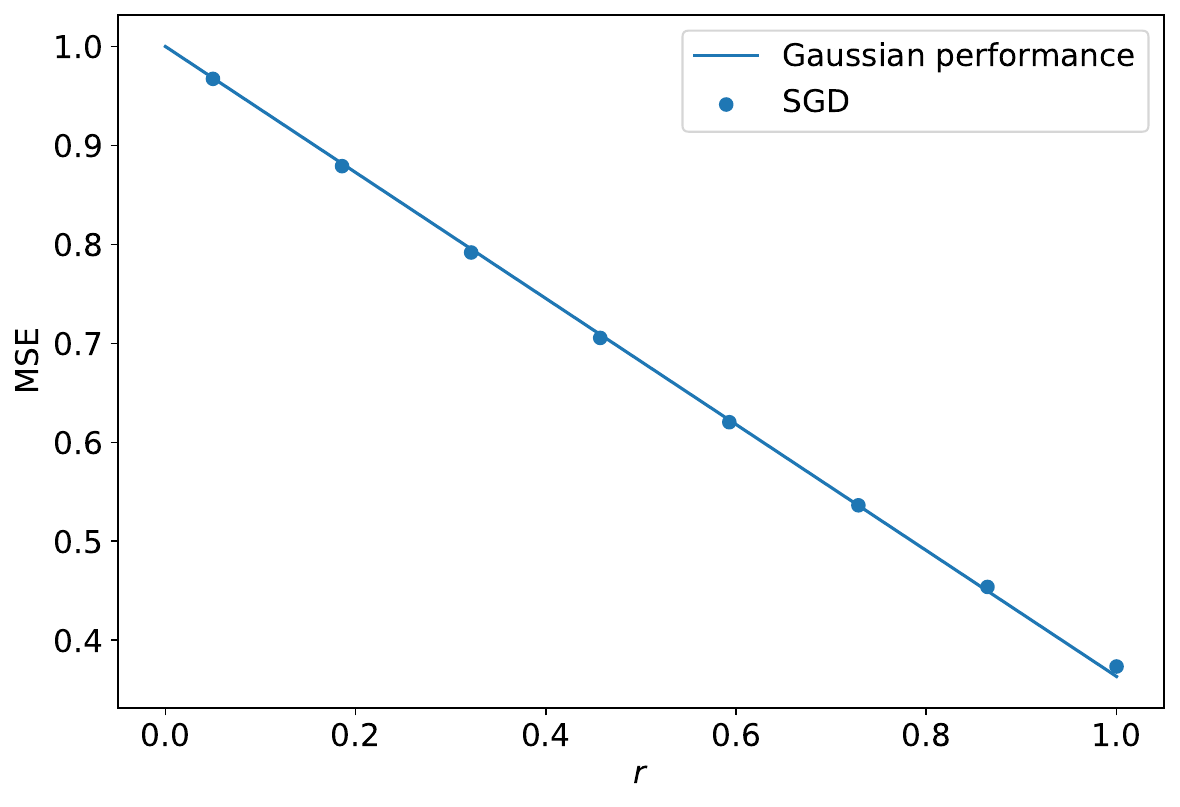}} & 
    \begin{tabular}[t]{@{}cc@{}}
        \raisebox{-\height}{\includegraphics[width=0.11\textwidth]{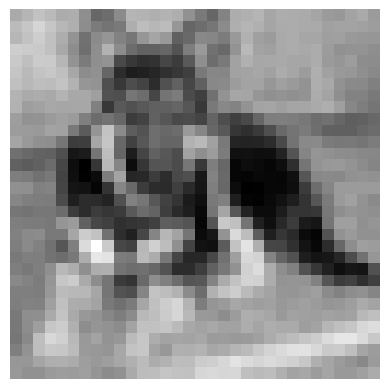}} &  \\[1.8cm]
        \raisebox{-\height}{\includegraphics[width=0.11\textwidth]{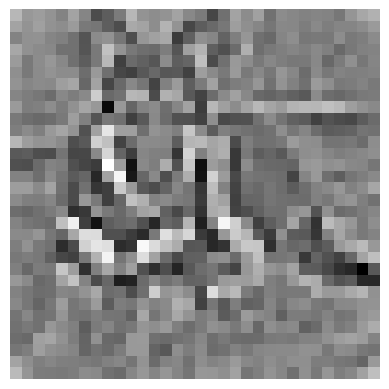}} & 
    \end{tabular}
\end{tabular}
\vspace{-1em}\caption{Compression of masked and whitened CIFAR-10 images of the class ``dog'' via the two-layer autoencoder in \eqref{eq:linear_decoding}. First, the data is whitened so that it has identity covariance (as in the setting of Theorem \ref{thm:GD-min-sparse-body}). Then, the data is masked by setting each pixel independently to $0$ with probability $p=0.7$. An example of an original image is on the top right, and the corresponding masked and whitened image is on the bottom right. The SGD loss at convergence (dots) matches the solid line, which corresponds to the prediction in \eqref{eq:gaussian_val} for the compression of standard Gaussian data (with no sparsity). 
}\vspace{-4mm}\label{fig:noniso_exps} 
\end{figure}

\begin{figure*}[t]
  \begin{center}
    \subfloat{\includegraphics[width=0.8\columnwidth]{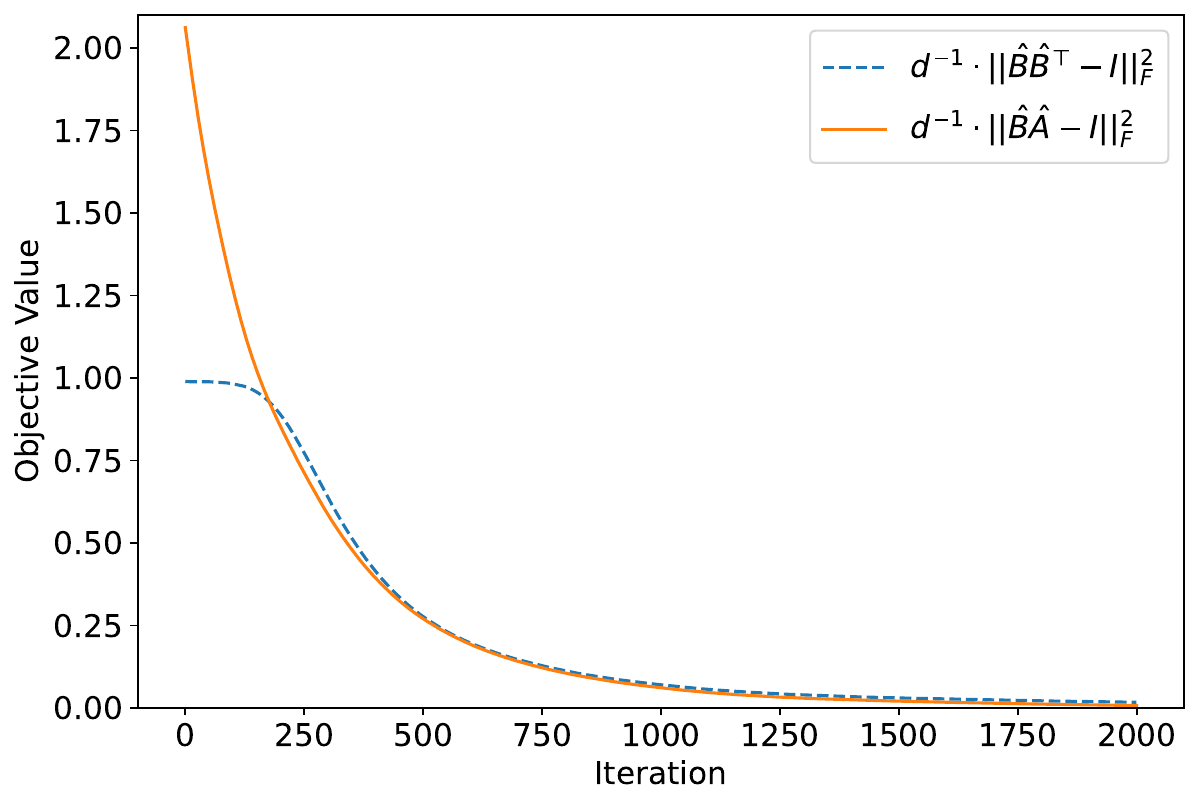}}\hspace{2.5em}
    \subfloat{\includegraphics[width=0.8\columnwidth]{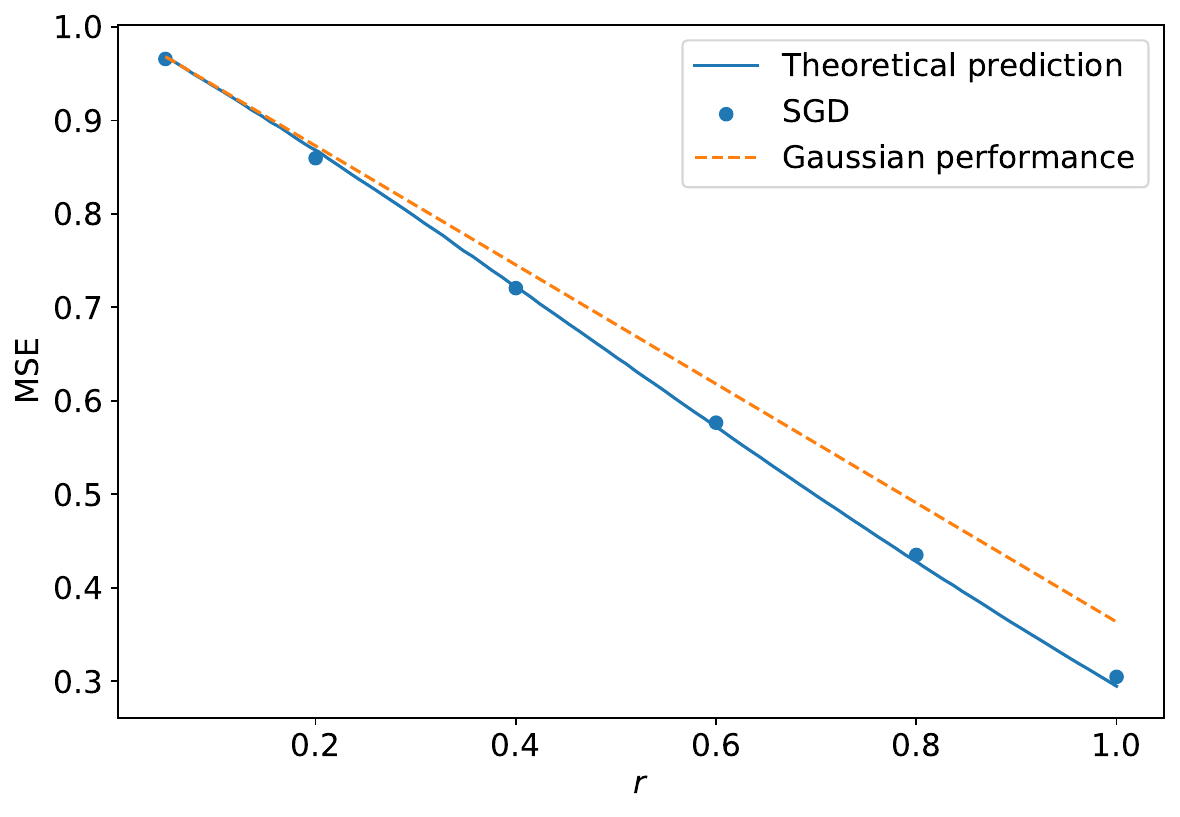}}
    
  \end{center}
  \vspace{-6mm}\caption{Compression of sparse Gaussian data via the autoencoder in \eqref{eq:linear_decoding_denoising}, where $f$ has the form in \eqref{eq:parametric_denoiser} and its parameters $(\alpha_1, \alpha_2, \alpha_3)$ are optimized via SGD. We set $d=100$ and $p=0.4$. \emph{Left.} Distance between $\hat\B\hat\B^\top$, $\hat\B\skew{5}\hat\A$ and the identity, as a function of the number of iterations, where $\hat\B$, $\skew{5}\hat\A$ denote the row-normalized versions of $\B$, $\A$. $\|\hat\B\hat\B^\top-\I\|_F$ and $\|\hat\B\skew{5}\hat\A-\I\|_F$ decrease and tend to $0$, meaning that (up to a rescaling of the rows) $\B\A$ and $\B\B^\top$ approach the identity. Here, we take $r=1$. \emph{Right.} MSE achieved by SGD at convergence, as a function of the compression rate $r$. The empirical values (dots) match the characterization of Proposition \ref{proposition:1} for $f=f^*$ in \eqref{eq:pmean} (blue line), and they outperform the MSE \eqref{eq:gaussian_val} obtained by compressing standard Gaussian data (orange dashed line). }
  \label{fig:BBT_BA}
  \vspace{-1em}
\end{figure*}

\vspace{-1em}

If there is an improvement upon $\mathcal R_{\rm Gauss}$ (i.e., $p > p_{\mathrm{crit}}$), the SGD 
dynamics exhibits a \emph{staircase} behavior. This phenomenon is 
displayed in Figure \ref{fig:sgd_rademacher}, which plots the error as a function of the number of SGD iterations for $p=0.8>p_{\mathrm{crit}}$: first, the MSE rapidly converges to $\mathcal R_{\rm Gauss}$; then, there is a plateau; finally,  the global minimum $1-r\cdot p$ is reached. 

We also remark that, as $p$ approaches $p_{\mathrm{crit}}$, the time needed by SGD to escape the plateau increases. A possible explanation is that, as $p$ decreases, the 
noise due to masking 
increases, which
increases the variance of the gradient. This makes it harder for $\B$ to find a direction 
towards a permutation of the identity (i.e., the global minimum).
Additional evidence of both the phase transition and the staircase behavior of SGD is 
in Appendix \ref{appendix:pt_sprase_gaussian_mixture}, where Figure \ref{fig:sgd_rademacher_2} considers Rademacher data and Figures \ref{fig:spare_mixture_of_gaussians_phase_transition}-\ref{fig:sgd_pt_sparse_gauss_2} data coming from a sparse mixture of Gaussians.

The proof strategy of Theorem \ref{thm:GD-min-sparse-body} could be useful to track SGD until it reaches the plateau. However, characterizing the time-scale needed to escape the 
plateau
likely requires new tools, and it provides an exciting research direction.

Finally, Figure \ref{fig:noniso_exps} shows that our theory predicts well the behavior of the compression of \emph{CIFAR-10 images} via the two-layer autoencoder in \eqref{eq:linear_decoding}. We let $x_1$ be the empirical distribution of the image pixels after whitening and masking\footnote{The whitening makes the data have isotropic covariance, as required by our theory; the masking makes the data sparse.}, and we verify that condition \eqref{eq:phase_transition} does not hold. Then, as expected, the autoencoder 
is unable to capture the structure coming from masking part of the pixels, and the loss at the end of SGD training equals  $\mathcal R_{\rm Gauss}$. Similar results hold for MNIST, see Figure \ref{fig:noniso_exps_2} in Appendix \ref{appendix:masked_mnist}.

\section{Provable benefit of nonlinearities and depth}\label{sec:5}

In this section, we prove that more expressive decoders than the linear one in \eqref{eq:linear_decoding} capture the sparsity of the data and, therefore, improve upon the Gaussian loss $\mathcal R_{\rm Gauss}$.

\subsection{Provable benefit of nonlinearities}

First, we apply a nonlinearity at the output of the linear decoding layer, as in \eqref{eq:linear_decoding_denoising}. Specifically, we take 
\vspace{-.4em}
\begin{equation}\label{eq:parametric_denoiser}
\vspace{-2mm}f(x) = \alpha_1 x + \alpha_2 \mathrm{tanh}(\alpha_3 x),
\end{equation}
and run SGD on the weight matrices $(\A, \B)$ and on the trainable parameters $(\alpha_1, \alpha_2, \alpha_3)$ in $f$. Figure \ref{fig:BBT_BA} shows that, at convergence, the minimizers have the same weight-tied orthogonal structure as obtained for Gaussian data ($\B\B^\top = \I$, $\A \propto \B^\top$), see the left plot. However, in sharp contrast with Gaussian data, the loss is \emph{smaller} than $\mathcal R_{\rm Gauss}$, see the blue dots on the right plot and compare them with the orange dashed curve. This empirical evidence motivates us to analyze the performance of autoencoders of the form \eqref{eq:linear_decoding_denoising}, where $\B$ is obtained by subsampling a Haar matrix of appropriate dimensions and $\A=\B^\top$.

\begin{proposition}[MSE characterization]\label{proposition:1} 
Let $r\leq 1$ and $\x$ have i.i.d.\ components with zero mean and unit variance. Consider the autoencoder $\hat\x(\x)$ in \eqref{eq:linear_decoding_denoising}, where $\B$ is obtained by subsampling a Haar matrix, $\A=\B^\top$, and $f$ is a Lipschitz function. Then, we have that, almost surely,
\vspace{-.4em}
\begin{equation} \label{eq:1RIGAMP}   
\lim_{d\to\infty}\frac{1}{d} \cdot \E_{\x}\|\x - \hat{\x}(\x)\|_2^2 = \E_{x_1,g} |x_1 - f (\mu x_1 + \sigma g)|_2^2,
\end{equation}
where $x_1$ is the first entry of $\bx$, $g\sim\mathcal N(0, 1)$ and independent of $x_1$, and the parameters $(\mu,\sigma)$ are given by
\vspace{-.4em}
\begin{equation}\label{eq:mu_sigma}
    \mu = r\sqrt{\frac{2}{\pi}}, \quad \sigma^2 = r \left(1 -r\cdot \frac{2}{\pi}\right) >0.
\end{equation}
\end{proposition}
\vspace{-2mm}Proposition \ref{proposition:1} is a generalization of Proposition \ref{proposition:identity_is_better}, which corresponds to taking a linear $f$. The idea is to relate $f(\B^\top\mathrm{sign}(\B\x))$ to the first iterate of a suitable RI-GAMP algorithm, so that the characterization in \eqref{eq:1RIGAMP} follows from state evolution. The details are in Appendix \ref{subsec:MSEden}. 

Armed with Proposition \ref{proposition:1}, one can readily establish the function $f$ that minimizes the MSE for large $d$. This in fact corresponds to the $f$ that minimizes the RHS of \eqref{eq:1RIGAMP}, i.e.,
\begin{equation}\label{eq:pmean}
f^{*}(y) = \E [x_1|\mu x_1 + \sigma g = y],    
\end{equation}
as long as the latter is Lipschitz (so that Proposition \ref{proposition:1} can be applied). Sufficient conditions for $f^{*}$ to be Lipschitz are that either \emph{(i)} $x_1$ has a log-concave density, or \emph{(ii)} there exist independent random variables $u_0, v_0$ s.t.\ $u_0$ is Gaussian, $v_0$ is compactly supported and $x_1$ is equal in distribution to $u_0+v_0$, see Lemma 3.8 of \cite{feng2022unifying}. The expression of $f^*$ for distributions of $x_1$ considered in the experiments (sparse Gaussian, Laplace, and Rademacher) is derived in Appendix \ref{appendix:denoiser_computations}.

The blue curve in the right plot of Figure \ref{fig:BBT_BA} evaluates the RHS of \eqref{eq:1RIGAMP} for the optimal $f=f^*$, when $x_1\sim {\rm SG}_1(p)$. Two observations are in order:

\vspace{-.75em}

\begin{enumerate}
    \item The blue curve matches the blue dots, obtained by optimizing via SGD the matrices $\A, \B$ and $f$ in the parametric form \eqref{eq:parametric_denoiser}. This means that the SGD performance is accurately tracked by plugging the optimal function \eqref{eq:pmean} into the prediction of Proposition \ref{proposition:1}.

    \item The blue curve improves upon the Gaussian loss $\mathcal R_{\rm Gauss}$ (orange dashed line). 
This means that, while the two-layer autoencoder in \eqref{eq:linear_decoding} is stuck at the MSE in orange (as proved by Theorem \ref{thm:GD-min-sparse-body}), by incorporating a nonlinearity, the autoencoder in 
\eqref{eq:linear_decoding_denoising} does better. In fact, as shown in Figure \ref{fig:MSEcomp} in Appendix \ref{app:prov}, the MSE achieved by the autoencoder in \eqref{eq:linear_decoding_denoising} with the optimal choice of $f$ (namely, the RHS of \eqref{eq:1RIGAMP} with $f=f^*$) is strictly lower than $\mathcal R_{\rm Gauss}$ for any $p\in (0, 1)$.
\end{enumerate}

\paragraph{Beyond Gaussian data: Phase transitions, staircases in the learning dynamics, and image data.} For general data $\x$ with i.i.d.\ zero-mean unit-variance components, the autoencoder in \eqref{eq:linear_decoding_denoising} displays a behavior similar to that described in Section \ref{sec:4} for the autoencoder in \eqref{eq:linear_decoding}: the SGD minimizers of the weight matrices $\A, \B$ either exhibit a weight-tied orthogonal structure ($\B\B^\top = \I$, $\A \propto \B^\top$), or 
come from permutations of the identity. This leads to a \emph{phase transition}  
in the structure of the minimizer (and in the  
MSE expression), as the sparsity 
$p$ varies. To quantify the critical value of $p$ at which the minimizer changes, one can compare the MSE when $\B$ is subsampled \emph{(i)} from a Haar matrix, and \emph{(ii)} from the identity. The former is readily obtained from Proposition \ref{proposition:1} where $f$ is given by \eqref{eq:pmean}, and the latter is given by the result below, which is proved in Appendix \ref{subsec:MSEid}.  

\begin{proposition}\label{proposition:sparse_rademacher_id_denoising} 
Let $\x$ have i.i.d.\ components with zero mean, unit variance and a symmetric distribution (i.e., the law of $x_1$ is the same as that of $-x_1$). Define  $\hat{\x}_{\mathrm{Id}}(\x)$ as in \eqref{eq:two_nets}, and fix $r\le 1$. Then, we have that, for any $d$, 
\vspace{-.5em}
\begin{equation}\label{eq:idden}    
\min_{f}\frac{1}{d}\cdot\E_{\x} \left[\|f(\hat{\x}_{\mathrm{Id}}(\x))-\x\|_2^2\right] = 1 - r \cdot (\E |x_1|)^2.
\end{equation}
\end{proposition}
\vspace{-1mm}

Figure \ref{fig:rademacher_denoising_phase_transition-body} displays the phase transition for the compression of sparse Rademacher data:

\vspace{-.75em}

\begin{itemize}
    \item For $p<\tilde p_{\mathrm{crit}}\approx 0.67$, SGD converges to a solution with MSE given by the RHS of \eqref{eq:1RIGAMP} with $f=f^*$. Furthermore, 
    $\B$ is a uniform rotation (see the central heatmap in Figure \ref{fig:rademacher_denoising_phase_transition} of Appendix \ref{appendix:denoiser_pt}).

\vspace{-.5em}
    
    \item For $p > \tilde p_{\mathrm{crit}}$,
    SGD converges to a solution with MSE given by the RHS of  
    \eqref{eq:idden}. Furthermore, 
    $\B$ is equivalent to a permutation of the identity (see the right heatmap in Figure \ref{fig:rademacher_denoising_phase_transition} of Appendix \ref{appendix:denoiser_pt}). 
\end{itemize}

\vspace{-.75em}

By comparing the blue dots/curve with the orange dashed line in Figure \ref{fig:rademacher_denoising_phase_transition-body}, we also conclude that, for all $p$, the MSE of the autoencoder in \eqref{eq:linear_decoding_denoising} improves upon the Gaussian performance $\mathcal R_{\rm Gauss}$. This is in contrast with the behavior of the autoencoder in \eqref{eq:linear_decoding} which remains stuck at $\mathcal R_{\rm Gauss}$ for $p<2/\pi$ (see Figure \ref{fig:rademacher_phase_transition}), and it demonstrates the benefit of adding the nonlinearity $f$.

\begin{figure}[t]
\centering
\includegraphics[width=0.44\textwidth]{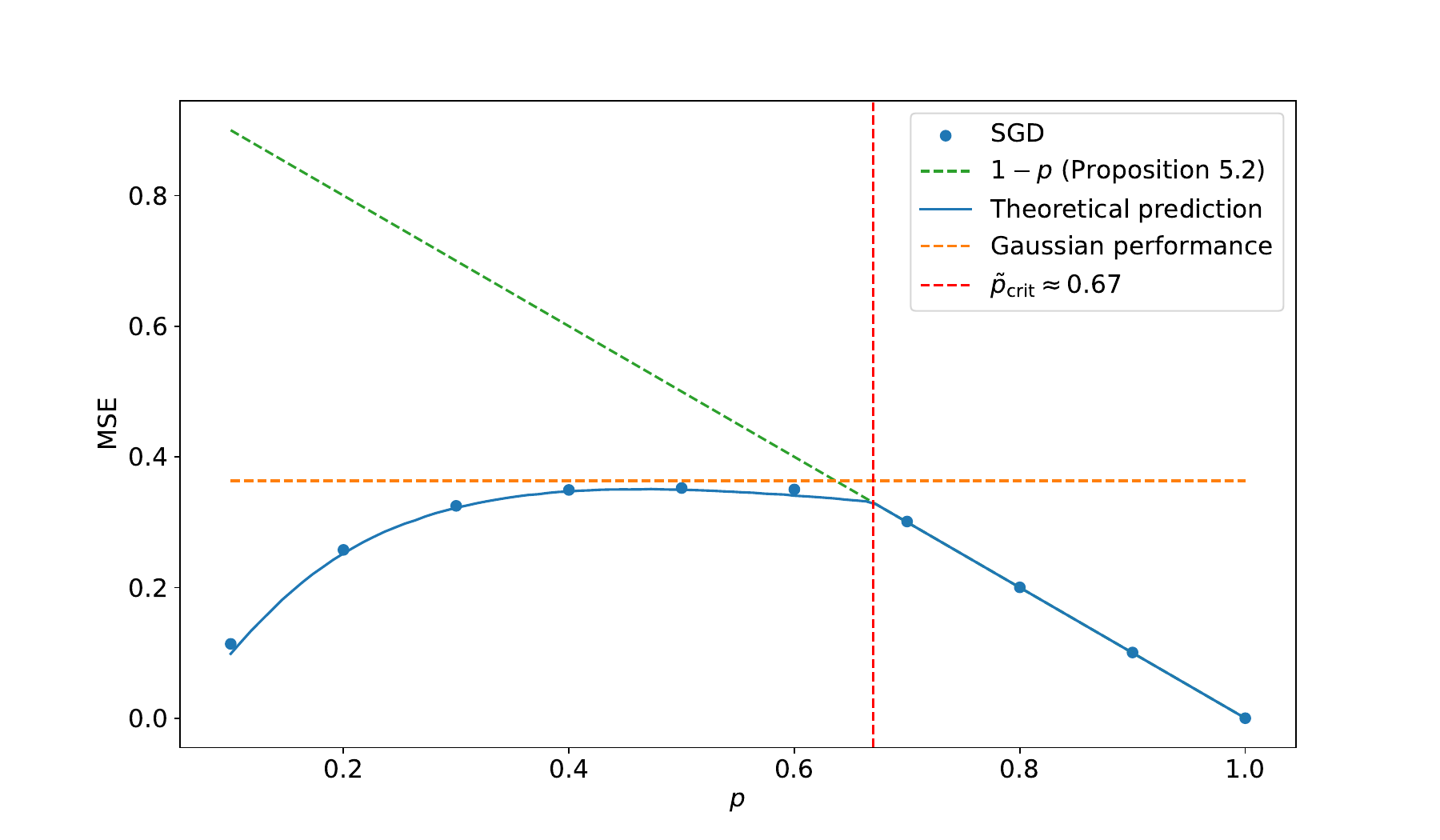}
\vspace{-1.25em}\caption{Compression of sparse Rademacher data via the autoencoder in \eqref{eq:linear_decoding_denoising}. We set $d=200$ and $r=1$. The MSE achieved by SGD at convergence is plotted as a function of the sparsity level $p$. The empirical values (blue dots) match our theoretical prediction (blue line). For $p<\tilde p_{\mathrm{crit}}$, the MSE is given by Proposition \ref{proposition:1} for $\B$ sampled from the Haar distribution; for $p\ge \tilde p_{\mathrm{crit}}$, the MSE is given by Proposition \ref{proposition:sparse_rademacher_id_denoising} for $\B$ equal to the identity.}\label{fig:rademacher_denoising_phase_transition-body}
\vspace{-4mm}
\end{figure}

For $p> \tilde p_{\mathrm{crit}}$, the learning dynamics exhibits again a \emph{staircase} behavior in which the MSE first gets stuck at the value given by the RHS of \eqref{eq:1RIGAMP} with $f=f^*$, and then reaches the optimal value of $1 - r \cdot (\E |x_1|)^2$. This is reported for $p=0.9> \tilde p_{\mathrm{crit}}\approx 0.67$ in Figure \ref{fig:sgd_pt_sparse_rademacher_transition} of Appendix \ref{appendix:denoiser_pt}.

\begin{figure}[t!]
    \centering
    \subfloat{\includegraphics[width=.9\columnwidth]{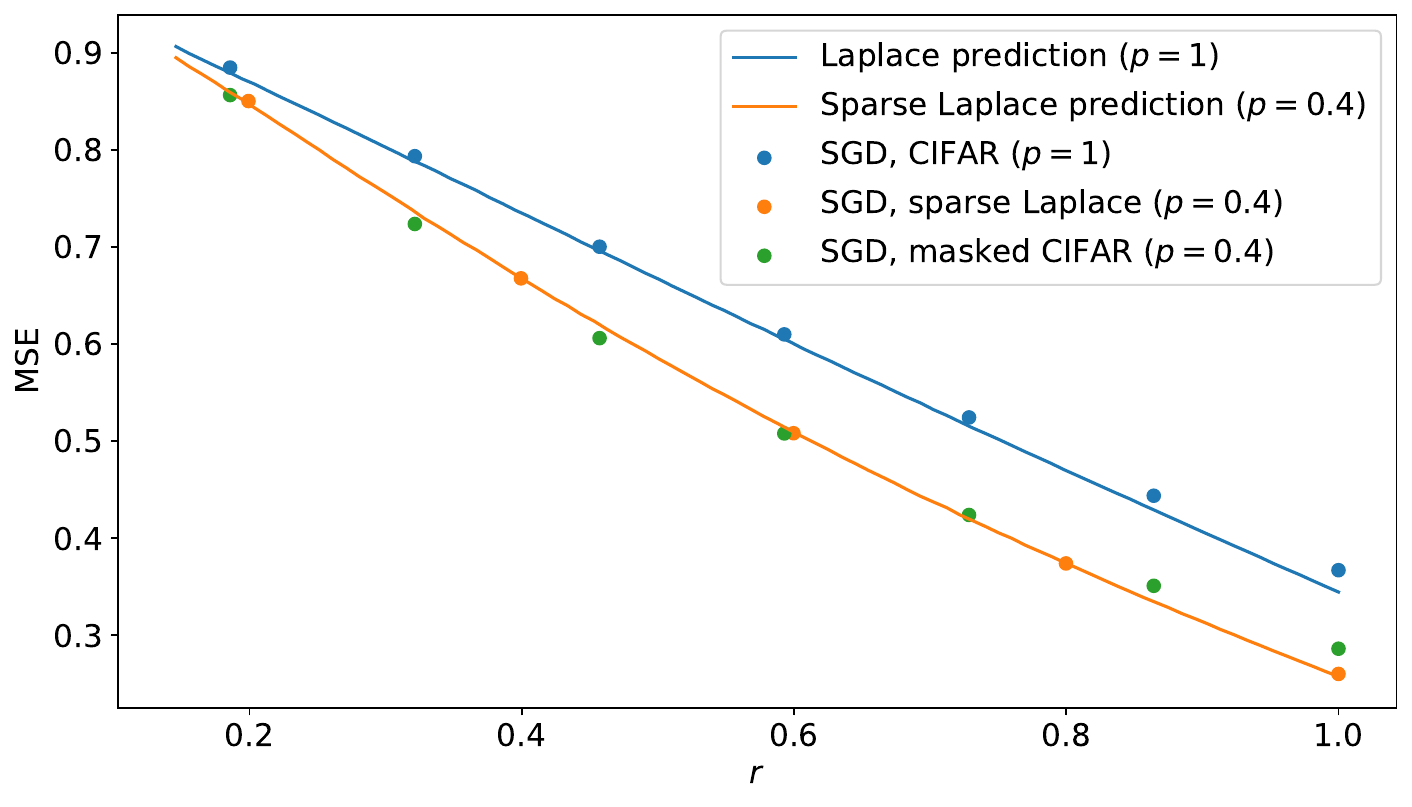}}\hspace{2.4em}
\vspace{-8mm}\caption{Compression of masked and whitened CIFAR-10 images of the class ``dog'' via the autoencoder in \eqref{eq:linear_decoding_denoising}. We plot the MSE as a function of the compression rate $r$. Dots are obtained by training the decoder matrix $\A$ and the parameters $(\alpha_1, \alpha_2, \alpha_3)$ via SGD on masked ($p=0.4$, green) or original ($p=1$, blue) CIFAR-10 images. Continuous lines refer to the predictions of Proposition \ref{proposition:1} for the optimal $f=f^*$ in \eqref{eq:pmean}, where $x_1$ has a Laplace distribution ($p=1$, blue) or a sparse Laplace distribution ($p=0.4$, orange). These curves match well the corresponding values obtained via SGD. Orange dots are obtained by training the matrices $\A, \B$ and the parameters $(\alpha_1, \alpha_2, \alpha_3)$ via SGD when $\x$ has i.i.d.\ sparse Laplace entries with $p=0.4$.\vspace{-3.5mm}}\label{fig:noniso_exps_l}
\end{figure}

Finally, Figure \ref{fig:noniso_exps_l} shows that the key features we unveiled for the autoencoder in \eqref{eq:linear_decoding_denoising} are still present when compressing \emph{sparse CIFAR-10 data}. The empirical distribution of the image pixels after whitening is well approximated by a Laplace random variable (see Figure \ref{fig:nlaplace_approx_appendix} in Appendix \ref{appendix:laplace_approx}), thus we denote by $x_1$ the corresponding sparse Laplace distribution (see \eqref{eq:sLap} in Appendix \ref{appendix:denoiser_computations} for a formal definition). The encoder matrix $\B$ is obtained by subsampling a Haar matrix, and it is fixed; the decoder matrix $\A$ and the parameters $(\alpha_1, \alpha_2, \alpha_3)$ in the definition \eqref{eq:parametric_denoiser} of $f$ are obtained via SGD training.
Two observations are in order:

\vspace{-.75em}

\begin{enumerate}
    \item The autoencoder in \eqref{eq:linear_decoding_denoising} captures the sparsity: 
    the MSE achieved on sparse data ($p=0.4$, green dots) is lower than the MSE on non-sparse data ($p=1$, blue dots).

\vspace{-.25em}

    \item For both values of $p$, the SGD performance matches the RHS of \eqref{eq:1RIGAMP} (continuous lines) with $f=f^*$. As expected, this MSE is smaller than $1 - r \cdot (\E |x_1|)^2$, and it coincides with that obtained for compressing synthetic data with i.i.d.\ Laplace entries (orange dots).
\end{enumerate}

\subsection{Provable benefit of depth}\label{sec:multilayer}

We conclude by showing that the MSE can be further reduced by considering a multi-layer decoder. Our design of the decoding architecture is inspired by the RI-GAMP algorithm \cite{venkataramanan2022estimation}, which iteratively estimates $\x$ from an observation of the form $\sigma(\B\x)$ via 
\vspace{-2.5mm}
\begin{align}\label{eq:RIGAMPt}
        &\x^t = \B^\top \hat{\bz}^{t} - \sum_{i=1}^{t-1} \beta_{t,i} \hat{\bx}^i, \quad \hat{\x}^t = f_t(\x^1,\cdots,\x^t), \\
        &\bz^{t} = \B \hat{\x}^{t} - \sum_{i=1}^{t} \alpha_{t,i} \hat{\bz}^i, \quad \hat{\bz}^{t+1} = g_{t}(\bz^1,\cdots,\bz^{t}, \hat\bz^1).\nonumber
\vspace{-2mm}\end{align}
Here, $f_t, g_t$ are Lipschitz and applied component-wise, and the initialization is $\hat{\bz}^{1}=\mathrm{sign}(\B\x)$. The coefficients $\{\beta_{t,i}\}$ and $\{\alpha_{t,i}\}$ are chosen so that, under suitable assumptions on $\B$,\footnote{$\B$ has to be bi-rotationally invariant in law, namely, the matrices appearing in its SVD are sampled from the Haar distribution.} the empirical distribution of the iterates is tracked via a low-dimensional recursion, known as \emph{state evolution}. This in turn allows to evaluate the MSE $\lim_{d\to\infty}\frac{1}{d}\|\bx-\hat\bx^t\|_2^2$.

\begin{figure}[t]
    \centering
    \subfloat{\includegraphics[width=0.85\columnwidth]{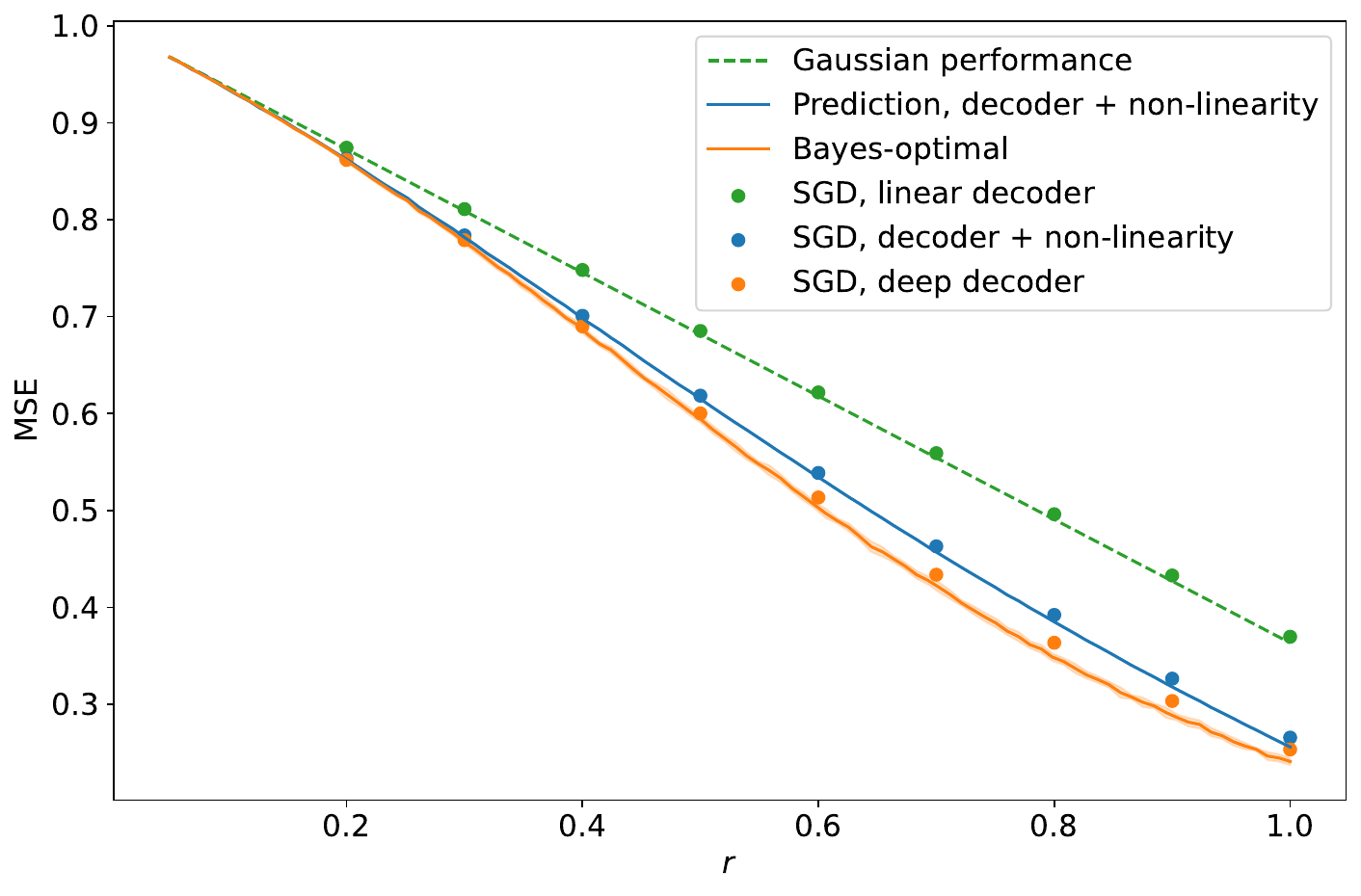}}\hspace{0.2em}
\vspace{-4mm}\caption{Compression of sparse Gaussian data $\x\sim {\rm SG}_d(p)$ 
for $p=0.3$ and $d=500$. We plot the MSE as a function of the compression rate $r$ for various autoencoder architectures. The architecture in \eqref{eq:rigamp_like_decoder} (orange dots) outperforms the autoencoders in \eqref{eq:linear_decoding} (green dots) and in \eqref{eq:linear_decoding_denoising} (blue dots), and it approaches the Bayes-optimal MSE (orange line).}\label{fig:rigamp_plots}\vspace{-4mm}
\end{figure}

The results of Proposition \ref{proposition:identity_is_better} and \ref{proposition:1} follow from relating the autoencoders in \eqref{eq:linear_decoding}-\eqref{eq:linear_decoding_denoising} to RI-GAMP iterates in \eqref{eq:RIGAMPt}. More generally, $\hat\x^t$ is obtained by multiplications with $\B, \B^\top$, linear combinations of previous iterates, and component-wise applications of Lipschitz functions. As such, it can be expressed via a multi-layer decoder with residual connections. The numerical results in \cite{venkataramanan2022estimation} show that taking $f_t, g_t$ as posterior means (as in \eqref{eq:pmean}) leads to Bayes-optimal performance, having fixed the encoder matrix $\B$. Thus, this provides a proof-of-concept of the optimality of multi-layer decoders.

In fact, Figure \ref{fig:rigamp_plots} shows that an architecture with three decoding layers is already near-optimal when $\x\sim {\rm SG}_d(p)$. The decoder output is 
$\hat{\bx}^2$ computed as (see also the block diagram in Figure \ref{fig:rigamp_decoder_block_diagram} in Appendix \ref{appendix:rigamp_like_nn})
\vspace{-.2em}
\begin{equation}\label{eq:rigamp_like_decoder}
    \begin{split}
    &\hat{\bz}^1 = \mathrm{sign}(\B\bx), \quad \x^1 = \W_1\hat{\bz}^1, \quad \hat{\bx}^1 = f_1(\bx^1),\\
    &\hat{\bz}^2 = g_1(\boldsymbol{V}_1 \hat{\bx}^1 \oplus_1 \hat{\bz}^1),\\
    &\x^2 = \hat{\bx}^1 \oplus_2 \boldsymbol{W}_2 \hat{\bz}^2, \quad \hat{\bx}^2 = f_2(\bx^1 \oplus_3 \bx^2).
\end{split}
\vspace{-.5em}\end{equation}

Here, $f_1(\cdot), f_2(\cdot), g_1(\cdot)$ are trainable parametric functions of the form in \eqref{eq:parametric_denoiser} and, for $i\in\{1, 2, 3\}$,  $\a\oplus_{i}\b=\beta_i\a+\gamma_i\b$, where $\{\beta_i$, $\gamma_i\}$ are also trained. The plot demonstrates the benefit of employing more expressive decoders:

\vspace{-.75em}

\begin{enumerate}
    \item The green dots are obtained via SGD training of the autoencoder in \eqref{eq:linear_decoding} and, as proved in Theorem \ref{thm:GD-min-sparse-body}, they match the Gaussian performance $\mathcal R_{\rm Gauss}$.
\vspace{-.25em}

    \item The blue dots are obtained via SGD training of the autoencoder in \eqref{eq:linear_decoding_denoising} and they match the prediction of Proposition \ref{proposition:1} with $f=f^*$ in \eqref{eq:pmean}.
\vspace{-.25em}

    \item The orange dots are obtained by using the decoder in \eqref{eq:rigamp_like_decoder} where $\W_1=\W_2=\B^\top$, $\V_1=\B$ are subsampled Haar matrices and the parameters in the functions $f_1, f_2, g_1, \{\oplus_i\}_{i=1}^3$ are trained via SGD. Similar results are obtained by training also $\W_1,\W_2,\V_1$, although at the cost of a slower convergence.
\end{enumerate}

\vspace{-.75em}

In summary, the architecture in \eqref{eq:rigamp_like_decoder} improves upon those in \eqref{eq:linear_decoding}-\eqref{eq:linear_decoding_denoising}, and it approaches the orange curve which gives the Bayes-optimal MSE achievable by fixing a rotationally invariant encoder matrix $\B$ \cite{ma2021analysis}. Additional details are deferred to Appendix \ref{appendix:rigamp_like_nn}.

We also note that considering a deep fully-connected decoder in place of the architecture in \eqref{eq:rigamp_like_decoder} does not improve upon the autoencoder in \eqref{eq:linear_decoding_denoising}. In fact, while sufficiently wide and deep models have high expressivity, their SGD training is notoriously difficult, due to e.g.\ vanishing/exploding gradients \cite{glorot2010understanding, He_2016_CVPR}.

\section*{Acknowledgements}

Kevin Kögler, Alexander Shevchenko and Marco Mondelli are supported by the 2019 Lopez-Loreta Prize. Hamed Hassani acknowledges the support by the NSF CIF award (1910056) and the NSF Institute for CORE Emerging Methods in Data Science (EnCORE).

\bibliography{ref}
\bibliographystyle{arxiv}

\newpage
\appendix
\onecolumn

\section{Proof of Theorem \ref{thm:GD-min-sparse-body}}\label{app:GD-min-Proof}

\subsection{Additional notation}

Given two matrices $\M_1$ and $\M_2$ of the same shape, their element-wise Schur product is $\M_1 \circ \M_2$ and the $\ell$-th element-wise power is $\M_1^{\circ \ell}$. The same notation is adopted for the element-wise product of vectors, i.e., $\v \circ \u$.
By convention, if $\B_{i,:}$ is a vector of zeroes, its normalization $\hat{\B}_{i,:}$ is also a vector of zeroes. We fix the evaluation of $\mathrm{sign}(\cdot)$ at the origin to be a Rademacher random variable, i.e., $\mathrm{sign}(0)$ takes values in the set $\{-1,1\}$ with equal probability. Note that this is a technical assumption with no bearing on the proof of the result.

For a matrix $\B$, we denote its $i$-th row by $\b_i = \B_{i, :}$, the exception being  that by convention $\a_k$ denotes the $k$-th column of $\A$.
$\B_\m$ is the masked version of a matrix $\B$, where masking is defined as $\bm_i = \b_i \circ \m$ and $\m$ has i.i.d.\ Bernoulli($p$) components. For convenience of notation, we define $\Bm = \B_\m$ \emph{only} for the matrix $\B$. By convention, masking has priority over transposing, i.e., $\B_\m^\top = (\B_\m)^\top$.
For $\B$ (and only $\B$), we define
$\Bh = \Bh_\m = \Dh \Bm_\m,$ where $\Dh$ is a diagonal matrix with entries $\Dh_{i,i} =1/\norm{\bm_i}$, as the masked and re-normalized version of $\B$.
We define $\norm{\B}_{\max} := \max_{i, j} \abs{\B_{i,j}}$.

We use the following convention for the constants.
All constants are independent of $d$ including those that are dependent on the quantities $p,r, f(x)=\arcsin(x), \alpha = f(1)-1$ and the dependence on these quantities will be suppressed most of the time. Uppercase constants like $C, C_X, C_R$ should be thought of as being much larger than $1$, whereas lowercase constants should be thought of as being smaller than $1$.

For a vector, the norm $\norm{\cdot}$ without subscript always refers to the 2-norm $\norm{\cdot}_2$.
Unless stated otherwise, we consider the space of matrices $\mathcal{M}_{n\times d}$ to be endowed with $\opn{\cdot}$. 
For a matrix $\bR$, we denote by $\Om{\bR}$ a matrix of the same dimensions as $\bR$ with $\opn{\Om{\bR}}\leq C \opn{\bR}$. This is a way to extend the big $O$ notation to matrices. Similarly, we will use the notation $O_{max}$ which functions as $O$ except that $\opn{\cdot} $ is replaced by $\maxn{\cdot}$.
We will often use that $n = \Om{d}$, since $r = \frac{n}{d}$ is fixed.

\subsection{Outline}\label{appendix:gd_min_section}
The start of our analysis is the following lemma.
\begin{lemma}\label{lem:mmse-to-matrix-obj}
  Let $\mathcal{R}(\cdot,\cdot)$ be the MSE defined in \eqref{eq:mse}, with $\x\sim {\rm SG}_d(p)$. Assume that all entries of $\B$ are not zero. Then, up to a multiplicative scaling and an additive constant, $\mathcal{R}(\A, \B)$ is given by
    \begin{equation}\label{eq:mmse=matrix-obj}
      \E_{\m\neq\0}\left[\tr{\A^\top \A f(\hat{\B}\hat{\B}^\top)} - \frac{2}{\sqrt{p}} \tr{\A \hat{\B}}\right] + O\left((1-p)^d\opn{\A}^2\right),
    \end{equation}
    where $f = \arcsin$ is applied \emph{component-wise} and the second term on the RHS is independent of $\B$.
\end{lemma}
\begin{proof}
    For any $\m \neq \0$ we can fix $\m$ and apply Lemma 4.1 in \cite{shevchenko2023fundamental}. The second term on the RHS corresponds to $\m = \0$.
\end{proof}
We now briefly elaborate on some technical details. First, by convention, all expectations over $\m$ are understood to be over $\m \neq \0$. Second, as the last term on the RHS in \eqref{eq:mmse=matrix-obj} does not depend on $\B$, it suffices to take the gradient of the objective without it. Lastly, the term $O\left((1-p)^d\opn{\A}^2\right)$ has a negligible effect when running the gradient descent algorithm in \eqref{eq:body-GDmin-formulas}. In fact, a by-product of our analysis is that $\A$ has bounded norm throughout the training trajectory, see Lemma \ref{lem:A-concentration}. Hence, the quantity $O\left((1-p)^d\opn{\A}^2\right)$ is exponentially small in $d$ and, therefore, it can be incorporated in the error of order $C\frac{\pl{}}{\sqrt{d}}$ being tracked during the recursion.

As a result, we can consider
the objective
\begin{equation}\label{eq:GDmin-obj}
    \E_{\m}\left[\tr{\A^\top \A f(\hat{\B}\hat{\B}^\top)} - \frac{2}{\sqrt{p}} \tr{\A \hat{\B}}\right],
\end{equation}
where $\m \in \{0, 1\}^d$ denotes a mask with i.i.d.\ Bernoulli($p$) entries, $\hat{\b}_i = \m \circ \b_i / \norm{\m \circ \b_i}_2$ and $(1-p)$ is the sparsity. Thus, the Riemannian gradient descent algorithm \eqref{eq:body-GDmin-formulas} applied to the objective \eqref{eq:GDmin-obj} can be rewritten as
\begin{equation}\label{eq:GDmin-formulas}
    \begin{split}
\A(t) &= \frac{1}{\sqrt{p}}\Em{\Bh(t)^\top} \left(\Em{f(\Bh(t)\Bh(t)^{\top})}\right)^{-1},\\
\B'(t) \hspace{-.15em}&:= \hspace{-.1em}\B(t) - \eta \left(\nabla_{\B(t)}+ \G(t) \right), \B(t+1) \hspace{-.15em}:= \hspace{-.1em}\mathrm{proj}(\B'(t)),        
    \end{split}
\end{equation}
where $\A(t)$ is the optimal matrix for a fixed $\B(t)$, $\nabla_{\B(t)}$ is defined below in \eqref{eq:grad-formula} and $(\G(t))_{i,j} \sim \mathcal N(0, \sigma^2)$ with $d^{-\gamma_g}\leq \sigma \leq C/d$ for some fixed $1<\gamma_g<\infty$.

The goal of this Appendix is to show the following theorem.
\begin{theorem}\label{thm:GD-min-sparse}
    Consider the gradient descent \eqref{eq:GDmin-formulas} with $\x\sim {\rm SG}_d(p)$. Initialize the algorithm with $\B(0)$ equal to a row-normalized Gaussian, i.e., $\B'_{i, j}(0)\sim \mathcal N(0, 1/d)$, $\B(0)=\mathrm{proj}(\B'(0))$, and let $\B(0)=\U\S(0)\V^\top$ be its SVD. Let the step size $\eta$ be $\Theta(1/\sqrt{d})$. Then, for any fixed $r<1$ and $T_{\rm max} \in (0, \infty)$, with probability at least $1-Cd^{-3/2}$, the following holds for all $t\le T_{\rm max}/\eta$
    \begin{equation}\label{eq:app-thm1}
    \begin{split}
        &\B(t) = \U \S(t) \V^\top + \bR(t), \\ 
        &\opn{\S(t)\S(t)^\top - \I} \leq C \exp\left(-c \eta t\right), \\
        &\lim_{d \to \infty} \sup_{t \in [0, T_{\rm max}/\eta]} \opn{\bR(t)} = 0,
    \end{split}
    \end{equation}
    where $C, c$ are universal constants depending only on $p, r$ and $T_{\rm max}$.
       Moreover, we have that, almost surely,
    \begin{equation}
    \begin{split}
        \lim_{t\to \infty}&\lim_{d \to \infty}\mathcal{R}(\A(t), \B(t)) = \mathcal{R}_{\rm Gauss},\\
    \lim_{d \to \infty}& \sup_{t \in [0, T_{\rm max}/\eta]} \opn{\B(t)-\B_{\rm Gauss}(t)}=0,\label{eq:appG1-app}
    \end{split}
    \end{equation}
where  $\mathcal{R}_{\rm Gauss}$ is defined in \eqref{eq:gaussian_val} and $\B_{\rm Gauss}(t)$ is obtained by running \eqref{eq:GDmin-formulas} with $\x\sim \mathcal N(\0, \I)$.

\end{theorem}

Let us provide a high-level overview of the proof strategy.
Using high-dimensional probability tools, we will show that with high probability
\begin{align}
\begin{split}\label{general-assumptions}
    &\B(t) = \X(t) + \bR(t), \\
    & \opn{\X(t)} \leq C_X, \quad \maxn{\X(t)}\leq C_X \frac{\log(d)}{\sqrt{d}}, \quad \X(t) = \U \S(t) \V, \mbox{ with } \U, \V \mbox{ Haar},  \\
    & \opn{\bR(t)} \leq C_R \frac{\plr}{\sqrt{d}},\\
    &\A(t) =  \X(t)^\top \left( \X(t)\X(t)^\top + \alpha\I \right)^{-1} + \Om{ C_X^{10} \frac{\log^{10}(d)}{\sqrt{d}}}  + \Om{\bR},
\end{split}
\end{align}
 where $\alpha=f(1)-1=\arcsin(1)-1$. This implies that the gradient in \eqref{eq:grad-formula} concentrates to the Gaussian one, namely, to the gradient obtained for $\x\sim\mathcal N(\0, \I)$. Then, an a-priori Gr\"onwall-type inequality will extend these bounds for all times $t \in [0, T_{\rm max}]$.
It is essential that the constants $C_X, C_R$  in \eqref{general-assumptions} can be chosen to only depend on $T_{\rm max}$, as otherwise the gradient dynamics could diverge in finite time. Thus, it is crucial that in all our lemmas we keep track of these constants explicitly and that in the error estimates they do not depend on each other.
While the analysis is quite technical, the high-level idea is simple: if each term that depends on $\m$ were replaced by its mean, then we would immediately recover the Gaussian case $p=1$ which was studied in \cite{shevchenko2023fundamental}.
By showing that each of the terms concentrates well enough, we can make this intuition rigorous. The main technical difficulty lies in controlling the additional error terms, which requires a more nuanced approach compared to the Gaussian case considered in \cite{shevchenko2023fundamental}.

The rest of this appendix is structured as follows. Section
\ref{app:aux-results} contains a collection of auxiliary results that are simple applications of standard results. In Section \ref{app:conc-tools}, we develop our high-dimensional concentration tools. In Section \ref{app:conc-grad}, we use these tools to show that $\A, \Bh, \nabla_\B$ all concentrate. Finally in Section \ref{app:gauss-reduction}, we combine these approximations with an a-priori Gr\"onwall bound in Lemma \ref{lem:B-specevo}, which allows us to bound the difference between the gradient trajectory and that obtained with Gaussian data.
\subsection{Auxiliary results}\label{app:aux-results}

A straightforward computation gives:
\begin{lemma}[Gradient formulas]\label{lem:grad-formula}
The derivative of \eqref{eq:GDmin-obj} w.r.t. $\B$ is given by
    \begin{equation}\label{eq:grad-formula}
\begin{split}
    (\nabla_{\B})_{k, :} &= 
\Em  \left[- 2 \frac{1}{\sqrt{p}} \m \circ\Jh_k \a_k + 2 \sum_{\ell=1}^\infty \ell c_\ell^2 \sum_{j \neq k} \langle \a_k, \a_j \rangle \langle \bh_k, \bh_j \rangle^{\ell-1} \Jh_k \bh_j \right],
\end{split}
\end{equation}
where $\Jh_k = \frac{1}{\norm{\bm_k}}\left(\I - \bh_k \bh_k^\top\right)$, $\a_k = \A_{:,k}$ and $c_\ell^2$ are the Taylor coefficients of $\arcsin(x)$.
\end{lemma}

We will make extensive use of the following linear algebra results.
\begin{lemma}[Linear algebra results]\label{lem:linalg-aux}

The following results hold:

\begin{enumerate}
    \item $\opn{\Bm} \leq \opn{\B}$.
    \item  For any $\M \in \R^{n \times d}$, we have $\opn{\M} \leq \sqrt{n} \max_k \norm{\M_{k,:}}$. In particular, $\opn{\Bh} \leq \sqrt{n}$.
    \item   For square matrices $\M_1, \M_2 \in \R^{n\times n}$,  \begin{equation}
        \opn{\M_1 \circ \M_2} \leq \sqrt{n} \opn{\M_1}\maxn{\M_2}.
    \end{equation}

    \item For any square matrix $\M$, we have $\opn{(\1\1^\top - \I) \circ \M} \leq 2 \opn {\M}$.

    \item For any square matrix $\M \in \R^{n\times n}$, we have $\opn{\M} \leq n \maxn{\M}$.
\end{enumerate}
\end{lemma}
\begin{proof}
    \begin{enumerate}
    
        \item This follows directly from the variational characterization of the operator norm, i.e.,
        $$\opn{\Bm} = \sup_{\norm{\u} \leq 1} \norm{\Bm \u} = \sup_{\norm{\u} \leq 1} \norm{\B  (\m \circ \u)} \leq \sup_{\norm{\u} \leq 1} \norm{\B \u},$$
        where the last step follows from $\norm{\m \circ \u} \leq \norm{\u}$.

        \item For $\norm{\v} = 1$, we have
        $$\norm{\M \v} = \sqrt{\sum_{k=1} ^n\langle \M_{k,:}, \v \rangle^2 } \leq \sqrt{\sum_{k=1}^n \norm{\M_{k,:}}^2} \leq \sqrt{n \max_k \norm{\M_{k,:}}^2} = \sqrt{n}\max_k \norm{\M_{k,:}}.$$
          
        \item For a unit vector $\e_i$, we have
        $$\norm{\M_1 \circ \M_2 \e_i} \leq \maxn{\M_2} \norm{\M_1 \e_i} \leq \maxn{\M_2}\opn{\M_1} .$$
        For a general vector $\v$, we can use the triangle inequality to obtain
        $$\norm{\M_1 \circ \M_2 \v} \leq \sum_i \abs{v_i} \maxn{\M_2} \norm{\M_1 \e_i} \leq \maxn{\M_2}\opn{\M_1} \sum_{i} \abs{v_i}.$$
        By using $$ \sum_{i} \abs{v_i} \leq \sqrt{n} \norm{\v},$$
        we obtain the desired bound.
        
        \item Note that $(\1\1^\top - \I) \circ \M = \M - \Diag{\M}$, so
        $$\opn{(\1\1^\top - \I) \circ \M} \leq \opn{\M} + \opn{\Diag{\M}} \leq 2 \opn{\M},$$
        where we have used that 
        $$\opn{\Diag{\M}} = \max_i |\M_{i,i}| \leq \opn{\M}.$$

        \item Note that $\norm{\M_{k, :}}\le \sqrt{n}\maxn{\M}$. Thus, the result follows from the point 2. above.    
    \end{enumerate}
\end{proof}

\begin{lemma}\label{lem:cl-geometric}
    Denote by $c_\ell^2$ the $\ell$-th Taylor coefficient of the function $\arcsin(x)$.
    Then, for $x \in [0, 1)$, $\ell_0 \in \mathbb{N}_{+}$, we have
    \begin{equation}\label{eq:lcl-geometric}
        \sum_{\ell=\ell_0}^\infty \ell c_\ell^2 x^{\ell-1}  \leq C(\ell_0)x^{\ell_0-1}\frac{1}{\sqrt{1-x}}.
    \end{equation}
\end{lemma}
\begin{proof}
    Recall that  
    $$ \frac{d}{dx}\arcsin(x) =\frac{1}{\sqrt{1-x}} =  \sum_{\ell=1}^\infty \ell c_\ell^2 x^{\ell-1},$$
with
    $$c_{2k}=0,\qquad c_{2k+1}^2 =\frac{(2k)!}{4^k(k!)^2(2k+1)}.$$
    By Stirling's approximation we have
    $$c_{2k+1}^2 = \Theta\left(\frac{1}{k^\frac{3}{2}}\right),$$
    which implies that, for odd $\ell, \ell_0 >0$ 
    $$\frac{\ell c_{\ell}^2}{(\ell+\ell_0-1)c_{\ell + \ell_0-1}^2} \leq C \frac{\ell(\ell+\ell_0-1)^\frac{3}{2}}{(\ell+\ell_0-1)\ell^\frac{3}{2}} \leq C(\ell_0).$$  
    Thus we have
    \begin{equation*}
             \sum_{\ell=\ell_0}^\infty \ell c_\ell^2 x^{\ell-1} 
             = x^{\ell_0-1}\sum_{\ell=\ell_0}^\infty  \ell c_\ell^2 x^{\ell-\ell_0}
            = x^{\ell_0-1}\sum_{\ell=1}^\infty (\ell+\ell_0 -1)c_{\ell + \ell_0 -1}^2 x^{\ell-1} 
             \leq x^{\ell_0-1}\sum_{\ell=1}^\infty C(\ell_0)\ell c_{\ell}^2 x^{\ell-1} 
             = C(\ell_0)x^{\ell_0-1} \frac{1}{\sqrt{1-x}},
    \end{equation*}
    which finishes the proof.
\end{proof}

\begin{lemma}\label{lem:l4-bound}
   Assume that $\B = \X + \bR$ with $\opn{\X} \leq C_X$, $\norm{\X}_{max} \leq C_X \frac{\log(d)}{\sqrt{d}}$, $\opn{\bR} \leq C_R\frac{\plr}{\sqrt(d)}$. Then, for large enough $d$, we have
\begin{equation}\label{eq:l4-bound}
    \norm{\b_i}_4^2 \leq CC_X^2\frac{\log(d)}{\sqrt{d}}. 
\end{equation}
\end{lemma}
\begin{proof}
    By H\"older, we have 
    $$\norm{\r_i}_4 \leq\norm{\r_i}_2^{\frac{1}{2}}\norm{\r_i}_\infty^{\frac{1}{2}} \leq \opn{\bR}^{\frac{1}{2}}\opn{\bR}^{\frac{1}{2}} = \opn{\bR}, $$
    and
    $$\norm{\x_i}_4 \leq\norm{\x_i}_2^\frac{1}{2} \norm{\x_i}_\infty^\frac{1}{2} \leq \opn{\X}^\frac{1}{2} \norm{\X}_{max}^\frac{1}{2},$$
    so
$$\norm{\b_i}_4^2 \leq (\norm{\x_i}_4 + \norm{\r_i}_4)^2 \leq C (\norm{\x_i}_4 ^2+\norm{\r_i}_4^2) \leq C C_X^2\frac{\log(d)}{\sqrt{d}} + C C_R^2 \frac{(\plr)^2}{d},
$$
which implies \eqref{eq:l4-bound}.
\end{proof}

\
\begin{lemma}[Concentration of $\Dh$]\label{lem:D-concentration}
    For $\b \in \R^d, \norm{\b}=1$ and $\m \sim \ber$ i.i.d.,
    we have 
    \begin{equation}\label{eq:D-concentration-aux}
        \P {\abs{\norm{\m \circ \b}_2^2 - p \norm{\b}_2^2} > \lambda} \leq C \expbr{-c \frac{\lambda^2}{\norm{\b}_4^4}},
    \end{equation}
    which implies 
    \begin{equation}\label{eq:D-concentration}
        \P{\abs{\norm{\m \circ \b}_2-\sqrt{p}\norm{\b}_2} > \lambda }\leq C \expbr{-c \frac{\lambda^2}{\norm{\b}_4^4}}.
    \end{equation}

    \begin{proof}
       Equation \eqref{eq:D-concentration-aux} is an immediate consequence of Hoeffding's inequality applied to the random variables $(\m_i-p)b_i^2$ (cf.\ Theorem 2.6.2 in \cite{vershynin2018high}).
       
       To obtain \eqref{eq:D-concentration} we note that $\abs{\norm{\m \circ \b}_2-\sqrt{p}\norm{\b}_2} > \lambda $ implies 
       $$\abs{\norm{\m \circ \b}_2^2-p\norm{\b}_2^2} = \abs{\norm{\m \circ \b}_2-\sqrt{p}\norm{\b}_2}\abs{\norm{\m \circ \b}_2+\sqrt{p}\norm{\b}_2}  > \sqrt{p}\norm{\b}_2\lambda .$$
       Since by assumption $\norm{\b} = 1$, the above implies \eqref{eq:D-concentration}.
    \end{proof}
\end{lemma}

\begin{lemma}\label{lem:lip-conc} Let $\U\in \R^{n \times n}$ be a Haar matrix and $\M \in \R^{n \times n} $ an independent random matrix with $\opn{\M} \leq C_M$.
Then, for $\Y := \U\M\U^\top$, we have 
\begin{equation}\label{eq:lip-UMU}
            \P{\norm{\Y-\frac{1}{n}\tr{\Y}\I}_{max} > \lambda} \leq Cd^2 \expbr{-\frac{c}{C_M^2}d\lambda^2}.
\end{equation}
If instead we have $\M \in \R^{n \times d}$ with $\frac{n}{d} \leq C$ and $\Y:= \U \M$, then
    \begin{equation}\label{eq:lip-UM}
            \P{\norm{\Y}_{max} > \lambda} \leq Cd^2 \expbr{-\frac{c}{C_M^2}d\lambda^2}.
\end{equation}
\end{lemma}

\begin{proof}
    We first fix $\M$, and note that since $\M$ and $\U$ are independent, the distribution of $\U$ does not change if we condition on $\M$.
For both inequalities, for fixed $\M$, the map $(SO_n, \norm{\cdot}_{F}) \to (\mathcal{M}_{n\times d}, \opn{\cdot})$, $\U \to \Y$ is Lipschitz as it is a bounded (bi-)linear form on a bounded set. The composition with the projection on the $(i,j)$-th  component of a matrix is also Lipschitz, so we can apply Theorem 5.2.7 in \cite{vershynin2018high} to obtain that $\Y_{i, j} $ is subgaussian in $\U$ with subgaussian norm $\frac{C C_M}{\sqrt{d}}$. Formally this means that 
$$\Psub{\U}{\abs{\Y_{i, j} - \E \Y_{i, j}} > \lambda | \M } \leq  C \expbr{-\frac{c}{C_M^2}d\lambda^2}.$$
Since the RHS is independent of $\M$ (i.e., it only depends on $C_M$) we have 
\begin{align*}
    \P{{\abs{\Y_{i, j} - \E \Y_{i, j}} > \lambda } } &= \Psub{\M}{\Psub{\U, \V}{\abs{\Y_{i, j} - \E \Y_{i, j}} > \lambda | \M }} \\
    &\leq \Psub{\M}{C \expbr{-\frac{c}{C_M^2d}\lambda^2}} \\
    &= C \expbr{-\frac{c}{C_M^2d}\lambda^2}.
\end{align*}
Now, \eqref{eq:lip-UMU} follows by noting that $\E \U \M \U^\top = \frac{1}{n}\tr{\U \M \U^\top} \I$ and using a simple union bound over $(i, j)$.
The proof of \eqref{eq:lip-UM} is the same, with the only difference being $\E\Y_{i,j} = 0$.
\end{proof}

\begin{lemma}\label{lem:lip-conc-conditional}
    Let $\U \in \R^{n\times n}, \V \in \R^{d \times d}$ be Haar matrices, and $\S_1, \S_2^\top \in \R^{n \times d}$  be deterministic diagonal matrices. Define $\bar{\M} = \S_1 (\V^\top)_\m (\V^\top)_\m^\top \S_2, \bar{\Y} = \U \bar{\M}\U^\top, \Y = p\U \S_1 \S_2  \U^\top$, $C_M := \opn{\S_1}\opn{\S_2}$. Then, for any $\gamma > 0$ and $d > d_0(\gamma)$, we have with probability at least $1 - C/d^2$ (in $\U, \V)$ \begin{equation}\label{eq:lip-conc-applic}
        \Psub{\m}{\maxn{\bar{\Y}-\frac{1}{n}\tr{\Y}\I} > CC_M\frac{\log^3(d)}{\sqrt{d}} \bigg| \U, \V}\leq  C \frac{1}{d^\gamma}.
    \end{equation}
    
\end{lemma}
\begin{proof}
    In the first step we will show that with probability at least $1- C/d^2$ in $\U, \V$
    \begin{equation}
        \Psub{\m}{\maxn{\bar{\Y}-\frac{1}{n}\tr{\bar{\Y}}\I} > C_M\frac{\log^3(d)}{\sqrt{d}} \bigg| \U, \V}\leq  C \frac{1}{d^\gamma}.
    \end{equation}
    The key observation is that
    \begin{align*}
        \Psub{\U, \V}{\Psub{\m}{\maxn{\bar{\Y}-\frac{1}{n}\tr{\bar{\Y}}\I} > C_M\lambda \bigg| \U, \V}> \alpha } &\leq\frac{1}{\alpha} \E_{\U, \V} \Psub{\m}{\maxn{\bar{\Y}-\frac{1}{n}\tr{\bar{\Y}}\I} > C_M\lambda \bigg| \U, \V}\\
        &=  \frac{1}{\alpha} \P{\norm{\bar{\Y}-\frac{1}{n}\tr{\bar{\Y}}\I}_{max} > C_M\lambda} \\
        & \leq C \frac{1}{\alpha}d^2\expbr{-\frac{c}{C_M^2}dC_M^2\lambda^2}\\
        & =  C \frac{1}{\alpha}d^2\expbr{-cd\lambda^2},
    \end{align*}
where the first passage follows from Markov's inequality and the last inequality is due to Lemma \ref{lem:lip-conc} and $C_M = \opn{\S_1}\opn{\S_2}\geq \opn{\bar{\M}}$.

    By choosing $\alpha = \frac{1}{d^\gamma}$ and $\lambda = \frac{\log^3(d)}{\sqrt{d}}$, we obtain that, with probability at least $1 - C/d^2$ in $\U, \V$,

    \begin{equation}\label{eq:lip-conc-cond-aux1}    
    \Psub{\m}{\maxn{\bar{\Y}-\frac{1}{n}\tr{\bar{\Y}}\I} > C_M\frac{\log^3(d)}{\sqrt{d}} \bigg| \U, \V} \leq \frac{1}{d^\gamma}.
    \end{equation}

    Now, in the second step, we will show that, with probability at least $1 - C/d^2$ over $\V$,
    \begin{equation}\label{eq:lip-conc-cond-aux2}
         \Psub{\m}{\maxn{\frac{1}{n}\tr{\bar{\Y}}\I-\frac{1}{n}\tr{\Y}\I} > C_M\frac{\log^3(d)}{\sqrt{d}} \bigg| \U, \V} \leq \frac{1}{d^\gamma}.
    \end{equation}
   
    First, note that by Lemma \ref{lem:lip-conc} with probability at least $1 - C/d^2$ (in $\V$) we have that $\maxn{\V} \leq C \frac{\log(d)}{\sqrt{d}}$, so $\norm{\V_{:,i}}_4^4 \leq C\frac{\log^4(d)}{d}$.     By Lemma \ref{lem:D-concentration}, we have that 
    $$\P{\abs{\left((\V^\top)_\m (\V^\top)_\m^\top\right)_{i,i} - p} > \lambda } \leq C \expbr{-c \frac{d\lambda^2}{\log^4(d)}}.$$
    Choosing $\lambda = \frac{\log^3(d)}{\sqrt{d}}$ we obtain for large $d$
    $$\P{\abs{\left((\V^\top)_\m (\V^\top)_\m^\top\right)_{i,i} - p} > \frac{\log^3(d)}{\sqrt{d}} } \leq C\frac{1}{d^{\gamma +1}}.$$
    By a simple union bound we obtain
    $$\P{\opn{\Diag{(\V^\top)_\m (\V^\top)_\m^\top} - p\I} 
    > \frac{\log^3(d)}{\sqrt{d}} } \leq C\frac{1}{d^\gamma}.$$

    Note that since $\S_1, \S_2$ are diagonal we have 
    $$\tr{\S_1(\V^\top)_\m (\V^\top)_\m^\top \S_2 } = \tr{\S_1\Diag{(\V^\top)_\m (\V^\top)_\m^\top} \S_2 },$$
    so $\opn{\Diag{(\V^\top)_\m (\V^\top)_\m^\top} - p\I} 
    \leq \frac{\log^3(d)}{\sqrt{d}}$ implies
    \begin{equation*}
        \begin{split}
            \abs{\tr{\S_1(\V^\top)_\m (\V^\top)_\m^\top \S_2 }- p\tr{\S_1 \S_2}} &=\abs{\tr{\S_1\left(\Diag{(\V^\top)_\m (\V^\top)_\m^\top} -p\I \right) \S_2 }}\\
            &\leq n\opn{\S_1\left(\Diag{(\V^\top)_\m (\V^\top)_\m^\top} -p\I \right) \S_2 } \\
            & \leq C_M n\frac{\log^3(d)}{\sqrt{d}}.
        \end{split}
    \end{equation*}
    Thus, 
    $$\P{\abs{\frac{1}{n}\tr{\S_1(\V^\top)_\m (\V^\top)_\m^\top \S_2 } - p\frac{1}{n}\tr{\S_1 \S_2} }> C_M\frac{\log^3(d)}{\sqrt{d}}  } \leq C\frac{1}{d^\gamma}.$$
    Noting that, by definition, $\tr{\bar{\Y}} = \tr{\S_1(\V^\top)_\m (\V^\top)_\m^\top \S_2 }$ and $\tr{\Y} = p\tr{\S_1 \S_2} $, we obtain \eqref{eq:lip-conc-cond-aux2}.

    Finally, combining \eqref{eq:lip-conc-cond-aux1} and \eqref{eq:lip-conc-cond-aux2} finishes the proof.
\end{proof}

\begin{lemma}\label{lem:G-bound}
Let $\B \in \R^{n\times d}$ be an arbitrary matrix, $\G\in \R^{n\times d}$ with $\G_{ij} \sim \mathcal{N}(0,\sigma^2)$, and assume $n = O(d)$. Then, for any $\delta >0$, we have
$$\Psub{\G}{\min_{i,j} \abs{\B_{ij} + \G_{ij} }\leq \delta}\leq Cd^2 \frac{\delta}{\sigma}.$$
\end{lemma}
\begin{proof}
    By the scale invariance of the problem, we may assume that $\sigma = 1$. 
    Let $g$ be a standard normal variable and $b \in \R$.
    Then, 
    $$\P{\abs{b+g} \leq \delta} = \P{g \in [b-\delta, b+\delta]} \leq C\delta,$$
    where the second step holds since the pdf of $g$ is bounded by a universal constant.
    The result of the lemma now follows by a simple union bound over all $(i,j)$ (and using $n = O(d))$.
\end{proof}

\subsection{Concentration tools}\label{app:conc-tools}
In this section, we provide the matrix concentration results needed for the proof. We recall that we use the shorthand notation  $\Bm = \B_\m$ and $\Bh = \Bh_\m = \Dh \Bm_\m, \Dh_{i,i} = 1/\norm{\bm_i}$ \emph{only} for the matrix $\B$. Here, the masking $\B_\m$ was defined 
as $(\B_\m)_{i,j} = \B_{i,j}\m_j$.

\begin{lemma}\label{lem:lip-conc-application}
    Let $\B = \X + \bR = \U \S \V^\top + \bR$, $\A = \X^\top (\X\X^\top+\alpha \I)^{-1} + \Om{\bR} +\Om{C_X^7\frac{\log^{10} (d)}{\sqrt{d}}}$, where $\U,\V$ are Haar matrices, $\opn{\bR} = o(1)$, $\S$ is a diagonal matrix s.t.\ $\opn{\S} \leq C_X, \frac{1}{n}\tr{\S\S^\top} = 1$, and $\alpha >0$ fixed. 
    Then, for any $\gamma >0, d > d_0(\gamma)$, with probability at least $1 - 
 C/d^2$ in $\U, \V$ and at least  $1 - 
 C/d^\gamma$ in $\m$, 
    \begin{enumerate}
        \item $\maxn{\B}  \leq C_X\frac{\log(d)}{\sqrt{d}}+ \Om{\opn{\bR}}$.
        
        \item 
        $\Diag {\left(\B\B^\top + \alpha \I\right)^{-2}\B\B^\top} = \frac{1}{n} \tr{\left(\B\B^\top + \alpha \I\right)^{-2}\B\B^\top}\I + \Om{\frac{\log(d)}{\sqrt{d}}}$.

        \item $\maxn{\A^\top \A - \frac{1}{n}\tr{\left(\X\X^\top + \alpha \I\right)^{-2}\X\X^\top} \I } \leq C_X^7\frac{\log^{10}(d)}{\sqrt{d}}+ \Om{\opn{\bR}}$.  
        
        \item $\frac{1}{p}\Bm \Bm^\top - \frac{1}{n}\tr{\B\B^\top}\I  = O_{max}(CC_X^2\frac{\log^3(d)}{\sqrt{d}})+ \Om{C_X\opn{\bR}}$. 

        \item $\Diag{\frac{1}{p}\Bm\A}=\frac{1}{n}\tr{\B\A} \I + \Om{C_X^8\frac{\log^{10}(d)}{\sqrt{d}}}+ \Om{C_X\bR}$.
        
        \item $\Diag{\frac{1}{p}\A^\top \A \Bm \Bm^\top}  = \frac{1}{n}\tr{\A^\top \A \B\B^\top} \I + \Om{ C_X^9\frac{\log^{10}(d)}{\sqrt{d}}}+ \Om{C_X^2\bR} $.
    \end{enumerate}
\end{lemma}

\begin{proof}
    The $\Om{\bR},\Om{\opn{\bR}}$ terms are always extracted by using that the LHS is a continuous function in $\bR$ (w.r.t. $\opn{\cdot}$). We carry this out explicitly for the first item and skip the details for the other items.
\begin{enumerate}
    \item By Lemma \ref{lem:lip-conc}, with probability at least $1-C/d^\gamma$, $$\maxn{\B} = \maxn{\U\S\V^\top + \bR} \leq \maxn{\U\S\V^\top} + \maxn{\bR} \leq C_X\frac{\log(d)}{\sqrt{d}} + \opn{\bR} = C_X\frac{\log(d)}{\sqrt{d}} +\Om{\opn{\bR}}, $$
    where we have used \eqref{eq:lip-UM} with $\lambda = C_X\frac{\log(d)}{\sqrt{d}}$ in the third step. 
    
    \item  This follows from \eqref{eq:lip-UMU} with $\lambda = \frac{\log(d)}{\sqrt{d}}$ by noting that  for any matrix $\B$ we have $\opn{\left(\B\B^\top + \alpha \I\right)^{-2}\B\B^\top} \leq \frac{1}{\alpha}$.

    \item As in the previous item, we have $ \opn{\left(\X\X^\top + \alpha \I\right)^{-2}\X\X^\top} \leq \frac{1}{\alpha}$
    so the result follows again from \eqref{eq:lip-UMU} with $\lambda = \frac{\log(d)}{\sqrt{d}}$.

    \item 
    First note that by 1. in Lemma \ref{lem:linalg-aux} we have $\opn{\Bm} \leq \opn{\B} \leq CC_X$, where we have used the assumption $\Om{\opn{\bR}}=o(1)$.
    Thus, we have 
    $$\Bm \Bm^\top = \U \S \Vtm  \Vtmt \S \U^\top + \Om{C_X\bR},$$
    so the result follows from Lemma \ref{lem:lip-conc-conditional}.

    \item  Note that $\opn{\X^\top\left(\X\X^\top + \alpha \I\right)^{-2} }\leq \frac{1}{2\sqrt{\alpha}}$ and by Lemma \ref{lem:linalg-aux} we have $\opn{\Bm} \leq CC_X$.
    Thus, we have
    $$\Bm\A = \U \S \Vtm \Vtmt \tilde{\S} \U^\top+ \Om{C_X\bR},$$
    where $\tilde{\S}$ is a diagonal matrix,
    so the result follows from Lemma \ref{lem:lip-conc-conditional}.

    \item By using again that $\opn{\X^\top\left(\X\X^\top + \alpha \I\right)^{-2} }\leq \frac{1}{2\sqrt{\alpha}}$ and $\opn{\Bm} \leq CC_X$, we have
    $$\A^\top \A \Bm \Bm^\top = \U \tilde{\S}^2\S \Vtm \Vtmt \S \U^\top+ \Om{C_X^2\bR},$$  
    where $\tilde{\S}$ is a diagonal matrix,
    so the result follows from Lemma \ref{lem:lip-conc-conditional}.
\end{enumerate} 
\end{proof}

\begin{lemma}[Master concentration for $\Bh$]\label{lem:master-Bh}
     Consider a fixed $\B = \X + \bR$ with unit norm rows and $\opn{\X} \leq C_X$, $\norm{\X}_{max} \leq C_X \frac{\log(d)}{\sqrt{d}}$, $\opn{\bR} \leq C_R \frac{\plr}{\sqrt{d}}$.  Let $2C_X\frac{1}{\sqrt{p}}>C_b > (1+c)C_X\frac{1}{\sqrt{p}}>0$, for some small constant $c>0$. Let $F: \ball_{\sqrt{d}}(\0) \cup (\Omega \cap \ball_{C_b}(\0))\subset \mathcal{M}_{n \times d} \to \mathcal{M}_{\tilde{n} \times \tilde{d}} $, for arbitrary $\tilde n, \tilde d$. Assume that $\Omega$ is s.t. with probability at least $1-Cd^{-k_F/2 -1/2}$ in $\m$ we have for $\Dh^b_{i, i} = \min\{\frac{1}{\norm{\bm_i}}, \frac{C_b}{C_X}\} $  that $\Dh^b\Bm,\frac{1}{\sqrt{p}}\Bm \in \Omega$. Further assume that $F$ satisfies the following properties: 
    \begin{enumerate}
      \item $\opn{F(\M)} \leq C_F d^\frac{k_F}{2}$  for every $\M = \Bh$;\
      
      \item $\opn{F(\M)} \leq C_F $ for every $\M \in \Omega \cap \ball_{C_b}(\0)$;
      
    \item $F$ is Lipschitz with constant $ C_{F'}$ on $\Omega \cap \ball_{C_b}(\0)$ (w.r.t. $\opn{\cdot}$ on both spaces).
 
    \end{enumerate}
    Then, for large enough $d>d_0(C_F, C_{F'}, C_b, C_R)$,
        \begin{equation}\label{eq:master-Bh-main}
        \opn{\Em F(\Bh) - \Em \mathbbm{1}_{\{\Dh^b\Bm,\frac{1}{\sqrt{p}}\Bm \in \Omega \}}F(\frac{1}{\sqrt{p}}\Bm) }\leq CC_{F'} C_X^3 \frac{\log^3(d)}{\sqrt{d}},
    \end{equation}

    where crucially the RHS is independent of $C_R$.

\begin{proof}
Define $\mathbbm{1}_{\bar{\Omega}} := \mathbbm{1}_{\{\Dh^b\Bm, \frac{1}{\sqrt{p}} \Bm \in \Omega \}}$ and $\mathbbm{1}_{\bar{\Omega}^c}:=1 - \mathbbm{1}_{\bar{\Omega}}$.
We will actually show the slightly stronger statement 
\begin{equation}\label{eq:master-Bh-proof}
    \Em\opn{ F(\Bh) - \mathbbm{1}_{\bar{\Omega}}F\left(\frac{1}{\sqrt{p}}\Bm\right)} \leq CC_{F'} C_X^3 \frac{\log^3(d)}{\sqrt{d}},
\end{equation}
which immediately implies \eqref{eq:master-Bh-main}.
First, by a simple triangle estimate we have
$$ \Em\opn{ F(\Bh) - \mathbbm{1}_{\bar{\Omega}}F\left(\frac{1}{\sqrt{p}}\Bm\right)} \leq \Em\opn{\mathbbm{1}_{\bar{\Omega}}\left( F(\Bh) - F\left(\frac{1}{\sqrt{p}}\Bm\right)\right)} +  \E_{\m}{\opn{\mathbbm{1}_{\bar{\Omega}^c}F(\Bh)}}.$$
  By our assumptions on $\Omega$ and assumption 1. we have 
   $$ \E_{\m}{\opn{\mathbbm{1}_{\bar{\Omega}^c}F(\Bh)}} \leq CC_Fd^{k_F/2} d^{-k_F/2 -1/2} = 2C_F\frac{1}{\sqrt{d}},$$
so w.l.o.g. we may assume that $\mathbbm{1}_{\bar{\Omega}} \equiv 1$. 

We now show that we can truncate $\Dh$ to $\Dh^b$  by applying the truncation function $\min\{x, \frac{C_b}{C_X}\}$ to each entry. Note that by definition of $C_b$ we have $\frac{1+c}{\sqrt{p}} \leq \frac{C_b}{C_X} \leq \frac{2}{\sqrt{p}}$. 
    By 2. in Lemma \ref{lem:linalg-aux}, we have that   $\opn{\Bh},\opn{\Dh^b\Bm} \leq \sqrt{n}$. Thus, by assumption, we obtain
    \begin{equation}\label{eq:bh-master-aux1}
        \Em{\opn{F(\Bh)-{F(\Dh^b \Bm)}}} \leq C_F d^{k_F/2} \P{\opn{\Dh} > \frac{C_b}{C_X}}.
    \end{equation}
    We also have the trivial bound
    
    $$\norm{\b_i}_4^4 \leq d \left(C_X \frac{\log(d)}{\sqrt{d}} +  C_R\frac{\plr}{\sqrt{d}}\right)^4 \leq \frac{1}{\sqrt{d}}.$$
    By a simple union bound, it follows from Lemma \ref{lem:D-concentration} that, for large enough $d$,
    \begin{align*}
         \P{\opn{\Dh} > \frac{C_b}{C_X}} &\leq Cd \cdot \expbr{-c\left(\frac{C_X}{C_b}-\sqrt{p}\right)^2\sqrt{d}}\\
         &\leq C \cdot \frac{1}{d^{1+k_F/2}},
    \end{align*}
    where we have used that $\frac{1}{\norm{\bm}} \geq \frac{C_b}{C_X}$ implies $\norm{\bm} \leq \frac{C_X}{C_b} \leq (1-c)\sqrt{p}$ (here, $c$ can indeed be treated as a universal constant by assumption).
    Together with \eqref{eq:bh-master-aux1}, we have
$$ \Em{\opn{F(\Bh)-{F(\Dh^b \Bm)}}}  \leq C_{F} \frac{1}{d}. $$
        
    We now need to bound
    $$\Em{\opn{F(\Dh^b\Bm ) - F\left(\frac{1}{\sqrt{p}}\Bm\right)}}.$$
    Since by 3. in the assumptions
    \begin{align*}
        \Em{\opn{F(\Dh^b\Bm ) - F\left(\frac{1}{\sqrt{p}}\Bm\right)}}&\leq C_{F'}\Em{\opn{\Dh^b\Bm  - \frac{1}{\sqrt{p}}\Bm}},
    \end{align*}
    we will only need to show concentration for $\Dh^b$.

    Recall that $\Dh^b_{i, i} =  \min\{\norm{\bm_i}_2^{-1}, C_b/C_X\}$. We now show that $\abs{\Dh^b_{i, i} - \frac{1}{\sqrt{p}}} > \lambda$ implies $\abs{\norm{\bm_i}_2-\sqrt{p}} >c\lambda$. To do so, we distinguish two cases. If $\norm{\bm_i} \geq \frac{C_X}{C_b}$, we have
     $$\abs{\Dh^b_{i, i} - \frac{1}{\sqrt{p}}} = \abs{\frac{\norm{\bm_i}-\sqrt{p}}{\norm{\bm_i}\sqrt{p}}}.$$
     Thus,  $\abs{\Dh^b_{i, i} - \frac{1}{\sqrt{p}}} > \lambda$ implies  $\abs{\norm{\bm_i}_2-\sqrt{p}} > \norm{\bm_i}\sqrt{p} \lambda \ge \sqrt{p}\frac{C_X}{C_b} \lambda = c \lambda $. 
     
     Next, if $\norm{\bm_i} \leq \frac{C_X}{C_b} < \sqrt{p}$, then $\abs{\Dh^b_{i, i} - \frac{1}{\sqrt{p}}} = \frac{C_b}{C_X}-\frac{1}{\sqrt{p}} $
     so necessarily $\lambda \leq \frac{C_b}{C_X}-\frac{1}{\sqrt{p}}$. But then $$\abs{\norm{\bm}_2-\sqrt{p}} \geq \abs{\frac{C_X}{C_b}-\sqrt{p}} \geq c \lambda,$$ where the last step is just the previous case for $\norm{\bm_i} = \frac{C_X}{C_b}$.
     
     This completes the proof that $\abs{\Dh^b_{i, i} - \frac{1}{\sqrt{p}}} > \lambda$ implies $\abs{\norm{\bm}_2-\sqrt{p}} >c\lambda$.
   Now, by Lemma \ref{lem:l4-bound} we have $\norm{\b_i}_4^2 \leq CC_X^2\frac{\log(d)}{\sqrt{d}}$. 
   So, we can use Lemma \ref{lem:D-concentration} and a union bound to obtain
    \begin{equation*}
        \P{\opn{\Dh^b - \frac{1}{\sqrt{p}}\I}>\lambda} \leq  Cd \cdot \expbr{-c \frac{d\lambda^2}{C_X^4\log(d)^2}}.
    \end{equation*}
    For $\lambda=C_X^2\frac{\log^3(d)}{\sqrt{d}}$ and large enough $d$, we can bound the RHS by $d^{-2}$.
    By 1. in Lemma \ref{lem:linalg-aux}, we have that, for large $d$, $\opn{\Bm} \leq \opn{\B} \leq C_X + C_R\frac{\log^3(d)}{\sqrt{d}} \leq 2C_X$. Note also that $\Dh^b\Bm  - \frac{1}{\sqrt{p}}\Bm = \left(\Dh^b - \frac{1}{\sqrt{p}}\I\right)\Bm$. Hence,
$$\P{\opn{\Dh^b\Bm  - \frac{1}{\sqrt{p}}\Bm}>2C_X^3\frac{\log^3(d)}{\sqrt{d}}} \leq \P{\opn{\Dh^b - \frac{1}{\sqrt{p}}\I}>C_X^2\frac{\log^3(d)}{\sqrt{d}}} \leq \frac{1}{d^2}.$$
    Thus, by fixing $\lambda = 2C_X^3\frac{\log^3(d)}{\sqrt{d}}$ and using that $\Dh^b$ is bounded, we conclude
    \begin{align*}
       \Em{\opn{\Dh^b\Bm  - \frac{1}{\sqrt{p}}\Bm}} &\leq \lambda\P{\opn{\Dh^b\Bm  - \frac{1}{\sqrt{p}}\Bm}\leq\lambda}  +\max_\m\opn{\Dh^b\Bm  - \frac{1}{\sqrt{p}}\Bm}\P{\opn{\Dh^b\Bm  - \frac{1}{\sqrt{p}}\Bm}>\lambda}\\
       &\leq \lambda + CC_X\P{\opn{\Dh^b\Bm  - \frac{1}{\sqrt{p}}\Bm}>\lambda}\\
       &\leq CC_X^3 \frac{\log^3(d)}{\sqrt{d}} + CC_X \frac{1}{d^2} \\
       &\leq CC_X^3 \frac{\log^3(d)}{\sqrt{d}}.
    \end{align*}  
\end{proof}
\end{lemma}

\begin{lemma}[Explicit approximations]\label{lem:master-Bh-explicit}
Assume that $\Psub{\m}{\Bm \in \Omega}\geq 1 - C\frac{1}{d^2}$,  where 
\begin{align*}
\Omega &= \{\M| (\1\1^\top-\I)\circ(\M\M^\top) = \Y + \Z, \maxn{\Y} \leq CC_X^2 \frac{\log^3(d)}{\sqrt{d}},\opn{\Z}\leq CC_XC_R\frac{\plr}{\sqrt{d}} \}\\ &\subset \{\M |  \maxn{(\1\1^\top -\I)\circ\M \M^\top} \leq C (C_X^2 + C_X C_R)  \frac{\plr)}{\sqrt{d}}\}.  
\end{align*}
Then, with probability at least $1 - C\frac{1}{d^2}$ in $\U, \V$, the following functions satisfy the assumption of Lemma \ref{lem:master-Bh} (with $C_F, C_{F'}$ independent of $d$)

    \begin{enumerate}
        \item $F(\B) =  -\A^\top + \frac{1}{\sqrt{p}}\A^\top \A\B +  \frac{1}{p}\Diag{\B \A} \B   - \frac{1}{\sqrt{p}}\mathrm{Diag}(\A^\top \A\B\B^\top) \B$,
        \item $F(\B) = \sum_{\ell \geq 3} c_{\ell}^2(\1\1^\top - \I) \circ(\B\B^\top )^{\circ \ell}, \sum_{\ell \geq 3} c_{\ell}^2 < \infty $,
        \item $F(\B) =  \sum_{\ell \geq 3}
        \ell c_{\ell}^2 \left( (\A^\top \A)\circ (\1\1^\top - \I) \circ(\B\B^\top)^{\circ (\ell-1)} - \mathrm{Diag} \left(\A^\top \A(\1\1^\top - \I) \circ(\B\B^\top)^{\circ \ell}\right) \right)\B, \\\sum_{\ell \geq 3} c_{\ell}^2 < \infty$,
    \end{enumerate} 
    where for 3. we need to additionally assume that the conclusion of Lemma \ref{lem:A-concentration} holds. (Note that this is not an issue as the proof of Lemma \ref{lem:A-concentration} uses only 2. in the current lemma).
    
\end{lemma}
    \begin{proof}
        First note that, for fixed $C_b,C_X, p$, by scaling the constant in the definition of $\Omega$ by $\frac{1}{p}\frac{C_b^2}{C_X^2}$, we may w.l.o.g.\ assume that $\P{\D^b\Bm, \frac{1}{\sqrt{p}}\Bm \in \Omega}\geq 1 - C\frac{1}{d^2}$. 
        
        We now show the claim for each of the functions separately.

        \begin{enumerate}
        
        \item 
        The first function is the sum of of multi-linear functions and thus is polynomially bounded and Lipschitz on bounded sets. Since we have a dimension independent bound on $\A$, the bounds are dimension-independent as well.

        \item 
    
        We will show condition 1., 2. and 3. in Lemma \ref{lem:master-Bh} separately.
        For 1., note that since $\abs{(\Bh\Bh^\top)_{i,j}}\leq 1$ we have $\abs{F(\Bh)_{i, j}} \leq \sum_{\ell} c_{\ell}^2 = C<\infty$. 
        Thus from 5. in Lemma \ref{lem:linalg-aux}, we obtain
        $\opn{F(\Bh)} \leq C d $ as desired.
        
        Next we show condition 2. We have
         the following estimate for $\M \in \Omega$ and $\ell \geq 3$
        \begin{equation}\label{eq:11-I-bound}
        \begin{split}\ell\opn{(\1\1^\top - \I) \circ\left(\M \M^\top\right)^{\circ \ell}} &\leq n\ell \maxn{(\1\1^\top - \I) \circ\left(\M\M^\top\right)^{\circ \ell}}\\ 
        &\leq n\ell \left(C(C_M^2+ C_MC_R)\frac{\plr}{\sqrt{d}}\right)^\ell\\
        &\leq C(C_M^2+C_MC_R)^3\frac{(\plr)^3}{\sqrt{d}}\\
        &\leq C\frac{1}{d^\frac{1}{4}},
        \end{split}
        \end{equation}
        where the first step follows from 5.\ in Lemma \ref{lem:linalg-aux}.
        Since the above bound is independent of $\ell$ it also holds for $F$ since $\sum_{\ell\geq 3} c_{\ell}^2 < \infty$. This gives us the desired bound for 2.
    
        To show 3. we write $F(\M) = \sum_\ell c_\ell^2 F_2^\ell(F_1(\M))$ where $F_1(\M):= (\1\1^\top - \I) \circ(\M\M^\top )$ and $F_2^\ell(\bQ) = \bQ^{\circ \ell}$.  We will show that $F_1, F_2^\ell$ are Lipschitz.
        By 4. in Lemma \ref{lem:linalg-aux}, we have that $\opn{(\1\1^\top - \I) \circ \bQ} \leq 2 \opn {\bQ}$. Thus, for $\B_1, \B_2 \in \ball_{C_b}(\0)$,
        \begin{align*}
            \opn{F_1(\B_1) - F_1(\B_2)}&= \opn{(\1\1^\top - \I) \circ (\B_1 \B_1^\top -\B_2\B_2^\top)} \\
            &\leq 2 \opn{\B_1 \B_1^\top -\B_2\B_2^\top} \\
            &= 2 \opn{\B_1 ( \B_1^\top - \B_2^\top) + ( \B_1 -\B_2) \B_2^\top}\\
            &\leq 4 C_b \opn{\B_1- \B_2}.
        \end{align*}
        Hence,  $F_1$ is Lipschitz with constant $4C_b$.       
        
        For $F_2^\ell$, we will show that $\norm{DF_2^\ell(\bQ)} \leq C d^{-\frac{1}{4}}$ if $\bQ = (\1\1^\top - \I) \circ(\M\M^\top ), \M \in \Omega \cap \ball_{C_b}(\0)$. Note that $\maxn{\bQ} \leq C (C_M^2+C_MC_R)\frac{\plr}{\sqrt{d}}$ and $\opn{\bQ} \leq 2 C_b^2$.
    Furthermore,  $F_2^\ell$ is a symmetric $\ell$-linear function, so the derivative in the direction $\Z$ is given by $DF_2^\ell(\bQ)\Z = \ell \Z \circ \bQ^{\circ \ell-1}$.
        From 3. in Lemma \ref{lem:linalg-aux} we have
        \begin{equation}
        \begin{split}\label{eq:YlcircZ-bound}\opn{\ell \Z \circ \bQ^{\circ \ell-1}} &\leq \ell\sqrt{n} \opn{\Z}\left(C(C_M^2+C_MC_R)\frac{\plr}{\sqrt{d}}\right)^{\ell-1}\\ 
        &\leq  C (C_M^2+C_MC_R)^2\opn{\Z}\frac{(\plr)^2}{\sqrt{d}}\\
        & \leq C\opn{\Z} \frac{1}{d^\frac{1}{4}},
        \end{split}
        \end{equation}
        so $\norm{DF_2^\ell(F_1(\B))} \leq C d^{-\frac{1}{4}}$.
        Now, note that since 
        $$\{ \bQ | \maxn{(\1\1^\top -\I) \circ\bQ} \leq C(C_M^2+C_MC_R) \frac{\plr}{\sqrt{d}}, \Diag{\bQ} = 0  \}$$
        is convex, the line segment between any two points lies in the set, so a bound on the derivative implies that the $\bQ^{\circ \ell}$ is Lipschitz with the same constant.
        Multiplying the two Lipschitz constants of $F_1, F_2^\ell$ we obtain that their composition is Lipschitz with constant $4CC_b d^{-\frac{1}{4}} $.
        
        Since none of the bounds depends on $\ell$, this immediately implies that $F(\M) = \sum_\ell c_\ell^2 F_2^\ell(F_1(\M))$ is Lipschitz as well, up to an additional constant $\sum_{\ell \geq 3} c_{\ell}^2$.

        \item  
        Again we will first show that condition 1. holds for $\B = \Bh$.
        First note that we can write (as in Lemma \ref{lem:grad-concentration-part1} below)
        \begin{equation*}
             (F(\Bh))_{k, :} =  \sum_{\ell=3}^\infty \ell c_\ell^2 \sum_{j \neq k} \langle \a_k, \a_j \rangle \langle \bh_k, \bh_j \rangle^{\ell-1} \Jh_k' \bh_j,
        \end{equation*}
        where $\Jh_k' =\I - \bh_k \bh_k^\top$. 
        Observe that
        \begin{equation}\label{eq:Jk'-copy}
    \norm{\Jh_k'\bh_j}^2 =  \norm{\bh_j - \bh_k \langle \bh_k, \bh_j \rangle }^2 = 1 -\langle \bh_k, \bh_j \rangle^2  \leq 2 \left(1 -\abs{\langle \bh_k, \bh_j \rangle}\right).
\end{equation}
Thus, using Lemma \ref{lem:cl-geometric}, we have
\begin{equation*}
\begin{split}
    \norm{ \sum_{\ell=3}^\infty \ell c_\ell^2 \sum_{j \neq k} \langle \a_k, \a_j \rangle \langle \bh_k, \bh_j \rangle^{\ell-1} \Jh_k' \bh_j }
   &\leq C \maxn{(\1\1^\top - \I)\circ\A^\top \A}\sum_{\ell=3}^\infty  \ell c_\ell^2\sum_{j\neq k }\abs{\langle \bh_k, \bh_j \rangle^{\ell-1}}\norm{ \Jh_k' \bh_j }  \\
   &\leq C \maxn{(\1\1^\top - \I)\circ\A^\top \A}  \maxn{\Bh\Bh^\top}^2\sum_{j\neq k} \frac{1}{\sqrt{1 - \abs{ \langle \bh_k, \bh_j \rangle}}} \norm{ \Jh_k' \bh_j } \\
   &\leq Cn \maxn{(\1\1^\top - \I)\circ\A^\top \A}\maxn{\Bh\Bh^\top}^2\\
   &\leq Cd,
   \end{split}
\end{equation*}
where we have used \eqref{eq:Jk'-copy} in the third step .
Using 2. in Lemma \ref{lem:linalg-aux}, the above implies
$$\opn{F(\Bh)} \leq  CC_R d^\frac{3}{2}.$$ 
This shows that condition 1. in Lemma \ref{lem:master-Bh} holds for $\Bh$.

To show the rest of the conditions, we may now assume that $\B \in \Omega \cap \ball_{C_b}(\0)$.
Note that, if we show that $F$ is Lipschitz on this set, condition 2. holds since
\begin{equation}\label{eq:bh-app-aux1}
    \opn{F(\B)} = \opn{F(\B) - F(\0)} \leq C_F' \opn{\B},
\end{equation}
where we have used $F(\0)=\0$.
Thus we only need to show the third condition.

Similarly to the previous case in 3., we define 
$F(\M) = F_2(F_1(\M))\M$ where $F_1(\M):= (\1\1^\top - \I) \circ(\M\M^\top )$ and 
\begin{equation*}
    F_2(\bQ) = \sum_{\ell \geq 3}
    \ell c_{\ell}^2 \left( (\A^\top \A)\circ \bQ^{\circ (\ell-1)} - \mathrm{Diag} \left(\A^\top \A\bQ^{\circ \ell}\right)\right).
\end{equation*}
Note that it is enough to show that $F_2(F_1(\M))$ is Lipschitz, as then by \eqref{eq:bh-app-aux1} and $\M \in \ball_{C_b}(\0)$, $F$ is the product of two bounded Lipschitz functions, and thus Lipschitz.
As for the previous function, we have that $F_1$ is Lipschitz with constant $4C_b$.
We will now derive a uniform bound for all $\ell \geq 3$.
Define
$$F_2^\ell(\bQ):=\ell c_{\ell}^2 \left((\A^\top \A)\circ \bQ^{\circ (\ell-1)} - \mathrm{Diag} \left(\A^\top \A \bQ^{\circ \ell}\right) \right).$$
As in the previous case, since $\bQ^{\circ \ell } $ is a symmetric $\ell$-linear function, we have
$$DF_2^\ell(\bQ)\Z = \ell c_{\ell}^2 \left( (\ell -1)(\A^\top \A)\circ \bQ^{\circ (\ell-2) }\circ \Z - \ell\mathrm{Diag} \left(\A^\top \A\bQ^{\circ (\ell-1)}\circ \Z\right) \right).$$
Recall we assume the conclusion of Lemma \ref{lem:A-concentration} to hold, so
\begin{equation}\label{eq:bh-app-aux2}
    \A = \X^\top \left( \X\X^\top + \alpha\I \right)^{-1} + \Om{\bR} + \Om{C_X^7  \frac{\log^{10}(d)}{\sqrt{d}}} = O_{max}\left(\frac{\log(d)}{\sqrt{d}}\right) + \Om{C_X^7C_R  \frac{\plr}{\sqrt{d}}} ,
\end{equation}
where we have used Lemma \ref{lem:lip-conc} in the second step.
Similarly, we have
\begin{equation}\label{eq:bh-app-aux3}
    \A^\top \A = O_{max}\left(\frac{\log(d)}{\sqrt{d}}\right) + \Om{C_X^7C_R  \frac{\plr}{\sqrt{d}}} .
\end{equation}
Thus, by using that $\opn{\boldsymbol{R}\circ\S}\le\opn{\boldsymbol{R}}\opn{\S}$ for any square matrices $\bR, \S$ (see Theorem 1 in \cite{visick2000quantitative}), we obtain
$$\opn{(\A^\top \A)\circ \bQ^{\circ (\ell-2)} \circ \Z} \leq\opn{O_{max}\left(\frac{\log(d)}{\sqrt{d}}\right) \circ \bQ^{\circ (\ell-2)} \circ \Z} + \Om{C_X^7C_R  \frac{\plr}{\sqrt{d}}} \opn{\bQ^{\circ (\ell-2)} \circ \Z},$$
Using the same estimate as in \eqref{eq:YlcircZ-bound},  3. in Lemma \ref{lem:linalg-aux} and recalling that $\ell\ge 3$,
\begin{equation}
\begin{split}\label{eq:DF2-lip}
    &\opn{DF_2^\ell(\bQ)\Z} 
    \leq C \ell^2c_{\ell}^2 \left(\sqrt{n}\opn{\Z} \frac{\log(d)}{\sqrt{d}} \maxn{\bQ}^{\ell-2} +C_X^7C_R \sqrt{n}\opn{\Z} \frac{\plr}{\sqrt{d}}  \maxn{\bQ}^{\ell-2}+\sqrt{n}\opn{\Z} \maxn{\bQ}^{\ell-1} \right)\\
    &\leq C\ell^2  \sqrt{d}\left( \left(C_X^7C_R\frac{\plr}{\sqrt{d}}\right)\left(C(C_X^2+C_XC_R)\frac{\log^{\alpha_R}(d)}{\sqrt{d}}\right)^{\ell-2} + \left(C(C_X^2+C_XC_R)\frac{\log^{\alpha_R}(d)}{\sqrt{d}}\right)^{\ell-1}\right)\opn{\Z}\\
    &\leq C d^{- \frac{\ell-2}{4}}\opn{\Z}.
\end{split}
\end{equation}
Since $\{ \bQ | \maxn{\bQ} \leq C(C_M^2+C_MC_R) \frac{\log^3(d)}{\sqrt{d}}, \Diag{\bQ} = 0 \}$ is convex, the line segment between any two points lies in the set, so a bound on the derivative implies that the $F_2^\ell$ is Lipschitz with the same constant.
As $F_2 = \sum_{\ell \geq 3} F_2^\ell$, we have that $F_2$ is Lipschitz with constant $\sum_{\ell \geq 3}C d^{- \frac{\ell-2}{4}} \leq C d^{- \frac{1}{4}}$. Finally the composition $F(\B) = F_2(F_1(\B))$ is Lipschitz with constant $C C_b d^{- \frac{1}{4}}$, so condition 3. holds, which concludes the proof.
\end{enumerate}
\end{proof}

\subsection{Concentration of the gradient}\label{app:conc-grad}
\begin{lemma}[Error analysis of $\A$]\label{lem:A-concentration}
Assume that $\B = \X + \bR$ with unit norm rows, $\X=\U \S \V^\top, \U, \V $ Haar, $\opn{\X} \leq C_X$, $\norm{\X}_{max} \leq C_X \frac{\log(d)}{\sqrt{d}}$, $\opn{\bR} \leq C_R \frac{\plr}{\sqrt{d}}$. Then, for $d >d_0(C_R)$ with probability at least $1 - 
 C\frac{1}{d^2 }$ (in $\U, \V$)  we have
        \begin{align}\label{eq:A-approx}
        \begin{split}
            \A = \frac{1}{\sqrt{p}}\Em{\Bh}^\top \left(\Em{f\left(\Bh \Bh^\top\right)}\right)^{-1} &= \B^\top \left( \B\B^\top + \alpha\I \right)^{-1} +\Om{ C_X^7 \frac{\log^{10}(d)}{\sqrt{d}}}\\
            &= \X^\top \left( \X\X^\top + \alpha\I \right)^{-1} + \Om{\bR} + \Om{C_X^7  \frac{\log^{10}(d)}{\sqrt{d}}}.
            \end{split}
    \end{align}
\end{lemma}
\begin{proof}
By a straightforward application of Lemma \ref{lem:master-Bh}, we have
    \begin{align}\label{eq:fBh}
    \frac{1}{\sqrt{p}}\Em{\Bh} &= \Em{\frac{1}{p}\Bm} + \Om{C_X^3  \frac{\log^3(d)}{\sqrt{d}}} \nonumber \\
    &= \B  + \Om{C_X^3 \frac{\log^3(d)}{\sqrt{d}}}.
    \end{align}

    Next, we will estimate $\Em{f\left(\Bh \Bh^\top\right)}$. Recall that $f(x) = \sum_{\ell} c_{\ell}^2 x^\ell$. As $f(1) < \infty$, we can define $\alpha = \sum_{\ell \geq 3} c_{\ell}^2$. As $c_1=1$, we have
    $$f(\Bh\Bh^\top) := \Bh \Bh^\top +\sum_{\ell \geq 3} c_{\ell}^2 \left(\Bh\Bh^\top\right)^{\circ \ell}  = \Bh\Bh^\top +\alpha  \I + \sum_{\ell \geq 3} c_{\ell}^2 \left(\Bh\Bh^\top - \I\right)^{\circ \ell} .$$
    Let $\Omega = \{\M| (\1\1^\top-\I)\circ(\M\M^\top) = \Y + \Z, \maxn{\Y} \leq CC_X^2 \frac{\log^3(d)}{\sqrt{d}},\opn{\Z}\leq CC_XC_R\frac{\plr}{\sqrt{d}} \} \subset \{\M |  \maxn{(\1\1^\top -\I)\circ\M \M^\top} \leq C (C_X^2 + C_X C_R)  \frac{\plr)}     {\sqrt{d}}\}$, then by using 4. in lemma \ref{lem:lip-conc-application}  with $\gamma=2$ we have $\Psub{\m}{\Bm \in \Omega}\geq 1 - C\frac{1}{d^2}$ (with probability at least $1-C/d^2$ in $\U, \V$).
    By noting that $\Bh\Bh^\top$ satisfies the assumptions of Lemma \ref{lem:master-Bh} and using 2. in Lemma \ref{lem:master-Bh-explicit} we have
    \begin{equation}\label{eq:Aapprox-Bm}
        \Bh\Bh^\top+ \sum_{\ell \geq 3} c_{\ell}^2 \left(\Bh\Bh^\top\right)^{\circ \ell} = \frac{1}{p}\Bm\Bm^\top + \alpha \I + \mathbbm{1}_{\Bm \in \Omega} \sum_{\ell \geq 3} c_{\ell}^2 (\1\1^\top - \I )\circ\left(\frac{1}{p}\Bm\Bm^\top \right)^{\circ \ell} + \Om{C C_X^3 \frac{\log^3(d)}{\sqrt{d}}}.
    \end{equation}
    By linearity, we have $\Em{\frac{1}{p}\Bm\Bm^\top} = \B\B^\top $.
     We will now show that
    \begin{equation}\label{eq:BBT3-bound-lsum}
        \Em\sum_{\ell\geq 3}\mathbbm{1}_{\Bm \in \Omega}  c_{\ell}^2  (\1\1^\top - \I )\circ\left(\frac{1}{p}\Bm\Bm^\top \right)^{\circ \ell} = \Om{C_X^6\frac{\log^{10}(d)}{\sqrt{d}}}.
    \end{equation}
    For now, let $(\1\1^\top -\I)\circ\Bm\Bm^\top =\Y +\Z$, $\maxn{\Y} \leq CC_X^2 \frac{\log^3(d)}{\sqrt{d}},\opn{\Z}\leq CC_XC_R\frac{\plr}{\sqrt{d}}$, as in the definition of $\Omega$ above. By the definition of $\Omega$, we have that, for ${\Bm \in \Omega} $ and $\ell\geq 3$, 
    \begin{equation}\label{eq:BBT3-bound}
    \left(\frac{1}{p}(\1\1^\top -\I)\circ\Bm\Bm^\top \right)^{\circ \ell}  = \frac{1}{p^\ell}\Y^{\circ \ell} + \frac{1}{p^\ell}\ell\Y^{\circ (\ell-1)}\circ \left(\Z+ e^\ell\Om{\Z^2}\right).
    \end{equation}

   Thus, by 3. in Lemma \ref{lem:linalg-aux}, we have that for $\ell \geq 3$
   $$\frac{1}{p^\ell}\opn{\Y^{\circ \ell}} \leq C\frac{1}{p^\ell} \sqrt{d}C_X^2 \left(C_X^2 \frac{\log^3(d)}{\sqrt{d}}\right)^{\ell-1}  \leq CC_X^6\frac{\log^{10}(d)}{\sqrt{d}}$$
   and 
   $$\frac{1}{p^\ell}e^\ell\ell\opn{\Y^{\circ (\ell-1)}\circ \left(\Z + e^\ell\Om{\Z^2}\right)} \leq C\frac{1}{p^\ell}e^\ell\ell\sqrt{d}C_XC_R\frac{\plr}{\sqrt{d}}\left(C_X^2 \frac{\log^3(d)}{\sqrt{d}}\right)^{\ell-1}\leq d^{-3/4}.$$
    Thus, we can further estimate \eqref{eq:BBT3-bound} by
    \begin{equation}\label{eq:BmBmT-entrywise-bound}
        \E_{\m}\opn{\mathbbm{1}_{\Bm \in \Omega} (\1\1^\top -\I) \circ \left(\frac{1}{p}\Bm\Bm^\top \right)^{\circ \ell}}\leq CC_X^6\frac{\log^{10}(d)}{\sqrt{d}}.
    \end{equation}
    Since the bound is independent of $\ell$, this shows \eqref{eq:BBT3-bound-lsum}.

    Combining \eqref{eq:Aapprox-Bm} and \eqref{eq:BBT3-bound-lsum} we now have
    \begin{equation}\label{eq:fBmBm}
    \Em f(\Bh\Bh^\top)  = \B \B^\top + \alpha \I + \Om{ C_X^6 \frac{\log^{10}(d)}{\sqrt{d}}}.
    \end{equation}
    
    From \eqref{eq:fBmBm} it also immediately follows that 
    \begin{equation}\label{eq:fBmBm-inv}\left(\Em {f\left(\Bh\Bh^\top\right)}\right)^{-1}  = \left(\B \B^\top + \alpha \I\right)^{-1} + \Om{C_X^6 \frac{\log^{10}(d)}{\sqrt{d}}},
    \end{equation}
    since for any psd matrix $\X > \alpha \I$, the map from $(\mathcal{M}_{n, n},\opn{\cdot}) \to (\mathcal{M}_{n, n},\opn{\cdot}) $,
    $\bR \to (\X + \bR)^{-1}$ is locally continuously differentiable at $0$.
    Combining \eqref{eq:fBh} and \eqref{eq:fBmBm-inv} yields the first equality in \eqref{eq:A-approx}. To see the second equality in \eqref{eq:A-approx}, it suffices to 
use the fact that the function $\B \to \B^\top \left( \B\B^\top + \alpha\I \right)$ is Lipschitz on bounded sets w.r.t $\opn{\cdot}$.
\end{proof}

\begin{lemma}[Gradient concentration, Part 1]\label{lem:grad-concentration-part1}

Assume that $\B = \X + \bR$ with unit norm rows, $\X=\U \S \V^\top, \U, \V $ Haar, $\opn{\X} \leq C_X$, $\norm{\X}_{max} \leq  CC_X\frac{\log(d)}{\sqrt{d}}$, $\opn{\bR} \leq C_R \frac{\plr}{\sqrt{d}}, \opn{\A} \leq C$. Further assume that $\min_{i, j} \abs{\B_{i,j}} \geq \delta>d^{-\gamma_\delta}$ for some $\gamma_\delta>0$. Then, for $d >d_0(C_X, C_R, \gamma_\delta)$ with probability at least $1 - C\frac{1}{d^2}$ in $\U, \V$,
the gradient of \eqref{eq:GDmin-obj} w.r.t.\ $\B$ can be written as
    \begin{equation}\label{eq:gradconc-formula}
        \nabla_{\B} = \Em\nabla_{\Bh}^1 + \sum_{\ell=3}^\infty \ell c_{\ell}^2\E_\m\nabla_{\Bh}^{\ell} + \Om{C_X^7\frac{\log^2(d)}{\sqrt{d}}},
    \end{equation}
    where
    \begin{equation}
    \begin{split}\label{eq:grad1conc-formula}
        &\frac{1}{2}(\nabla^1_{\Bh})_{k,:}
        =  -\a_k + \frac{1}{p}\langle \a_k, \bh_k \rangle \bh_k  + \frac{1}{\sqrt{p}}\sum_{j } \langle \a_k, \a_j \rangle \Jh_k' \bh_j, \\
        &\frac{1}{2}\nabla^1_{\Bh} =-\A^\top + \frac{1}{\sqrt{p}}\A^\top \A\Bh +  \frac{1}{p}\mathrm{Diag}(\Bh \A) \Bh   - \frac{1}{\sqrt{p}}\mathrm{Diag}(\A^\top \A\Bh\Bh^\top) \Bh,
        \end{split}
    \end{equation}
    
    \begin{equation}
    \begin{split}\label{eq:grad2conc-formula}
        &\frac{1}{2}(\nabla^\ell_{\Bh})_{k,:} =   \frac{1}{\sqrt{p}}\sum_{j } \langle \a_k, \a_j \rangle \langle \bh_k, \bh_j \rangle^{\ell-1} \Jh_k' \bh_j, \\
        &\frac{1}{2}\nabla^\ell_{\Bh}=\frac{1}{\sqrt{p}}(\A^\top \A)\circ (\Bh\Bh^\top - \I)^{\circ (\ell-1)}\Bh - \frac{1}{\sqrt{p}}\mathrm{Diag}(\A^\top \A(\Bh\Bh^\top-\I)^{\circ \ell})\Bh,
        \end{split}
   \end{equation}
   and
   $$\Jh_k' :=\frac{1}{\sqrt{p}}\left( \I - \bh_k \bh_k^\top\right).$$
\end{lemma}
\begin{proof}
Recall from \eqref{eq:grad-formula} that the gradient is given by
\begin{equation}\label{eq:grad-formula-prelim}
\begin{split}
    (\nabla_{\B})_{k, :} &= 
\Em  - 2 \frac{1}{\sqrt{p}} \m \circ\Jh_k \a_k + 2 \sum_{\ell=1}^\infty \ell c_\ell^2 \sum_{j \neq k} \langle \a_k, \a_j \rangle \langle \bh_k, \bh_j \rangle^{\ell-1} \Jh_k \bh_j, \\
\end{split}
\end{equation}
where $\Jh_k = \frac{1}{\norm{\bm_k}}\left(\I - \bh_k \bh_k^\top\right)$.

We will approximate $\Jh_k$ by $\Jh_k'$. This will make the gradient have the same functional form (for fixed $\m$) as in the Gaussian case. This follows from the fact that the gradient inside the expectation is the same as the gradient of the Gaussian objective  (86) in \cite{shevchenko2023fundamental} evaluated at $\B = \Bm$. 
We denote the new gradient with $\Jh_k$ replaced by $\Jh_k'$ as $\nabla_\B'$. 
We proceed by decomposing the error
 $\norm{(\nabla_\B)_{k, :} - (\nabla_\B')_{k, :}}$
into multiple parts and analysing them individually.

First, we need to decompose the error. Combining Lemma \ref{lem:D-concentration} and Lemma \ref{lem:l4-bound},  we have
\begin{equation}
\begin{split}\label{eq:norm-bm-bound}
       \Psub{\m}{\abs{\norm{\bm_k} - \sqrt{p}} > C_X^2\frac{\log^2(d)}{\sqrt{d}}}
       & =\Psub{\m}{\abs{\norm{\bm_k}^2 - p} > C_X^2(\norm{\bm_k}
       + \sqrt{p})^{-1}\frac{\log^2(d)}{\sqrt{d}}} \\ 
       &\leq \Psub{\m}{\abs{\norm{\bm_k}^2 - p} > C_X^2\frac{1}{1 + \sqrt{p}}\frac{\log^2(d)}{\sqrt{d}}}\\
       &\leq C \frac{1}{d^\gamma}.
\end{split}
\end{equation}

Denoting by $A$ the event that $\abs{\norm{\bm_k} - \sqrt{p}} > C_X^2\frac{\log^2(d)}{\sqrt{d}}$ jointly for all $k$, we have
\begin{equation}
\begin{split}\label{eq:grad-initial-approx}
    \Em(\nabla_\B)_{k, :} - (\nabla_\B')_{k, :} &=
    \E_\m \1_{A}((\nabla_\B)_{k, :} - (\nabla_\B')_{k, :}) +  \Em\1_{A^c} (\nabla_\B)_{k, :} - (\nabla_\B')_{k, :}  \\
    &=  \E_\m \1_{A}((\nabla_\B)_{k, :} - (\nabla_\B')_{k, :}) + \Em  - 2 \frac{1}{\sqrt{p}} \m \circ \epsilon_\m^k\Jh_k' \a_k + 2 \sum_{\ell=1}^\infty \ell c_\ell^2 \sum_{j \neq k} \langle \a_k, \a_j \rangle \langle \bh_k, \bh_j \rangle^{\ell-1} \epsilon_\m^k\Jh_k' \bh_j \\
    &=: (\nabla^1_{err})_{k, :} + (\nabla^2_{err})_{k, :},
    \end{split}
\end{equation}
where $\abs{\epsilon_\m^k} \leq C_X^2\frac{\log^2(d)}{\sqrt{d}}$ and $\nabla^1_{err}, \nabla^2_{err}$ are the matrices corresponding to the first and second expectation, respectively.
Using this notation, proving the lemma is equivalent to showing that 
\begin{equation}\label{eq:grad-conc-part1-err}
    \nabla^1_{err} + \nabla^2_{err} = \Om{C_X^7\frac{\log^2(d)}{\sqrt{d}}}.
\end{equation}

We will start with bounding $\opn{(\nabla^1_{err})_{k, :}}$.
By the definition of $(\nabla_\B)_{k, :} - (\nabla_\B')_{k, :})$, we have the following simple bound:
\begin{equation}\label{eq:grad-part1-rhs1}
     \E_\m\norm{ \1_{A}((\nabla_\B)_{k, :} - (\nabla_\B')_{k, :})} 
     \leq  C \frac{1}{d^\gamma} \max_\m\norm{ - 2 \frac{1}{\sqrt{p}} \m \circ(\Jh_k - \Jh_k') \a_k + 2 \sum_{l=1}^\infty \ell c_\ell^2 \sum_{j \neq k} \langle \a_k, \a_j \rangle \langle \bh_k, \bh_j \rangle^{\ell-1} (\Jh_k - \Jh_k') \bh_j}.
\end{equation}
Note that
\begin{equation}\label{eq:jhk-jhk'}
    \opn{\Jh_k - \Jh_k'} = \opn{\left(\frac{1}{\norm{\bm_k}} - \frac{1}{\sqrt{p}}\right)\left(\I - \bh_k \bh_k^\top\right)} \leq \left(\frac{\sqrt{p}}{\delta} + 1\right) \opn{\Jh_k'}.
\end{equation}
Furthermore, since by definition $\norm{\bh_j}= \norm{\bh_k} = 1$,
\begin{equation}\label{eq:Jk'}
    p\norm{\Jh_k'\bh_j}^2 =  \norm{\bh_j - \bh_k \langle \bh_k, \bh_j \rangle }^2 = 1 -\langle \bh_k, \bh_j \rangle^2  \leq 2 \left(1 -\abs{\langle \bh_k, \bh_j \rangle}\right).
\end{equation}
We clearly have
\begin{equation}\label{eq:aux1}
    \norm{\m \circ(\Jh_k - \Jh_k') \a_k} \leq  \left(\frac{\sqrt{p}}{\delta}+1\right) \opn{\Jh_k'}\norm{\a_k} \leq C(1 + \frac{1}{\delta}) ,
\end{equation}
where we have used \eqref{eq:jhk-jhk'} and the fact that masking reduces the norm.

By Lemma \ref{lem:cl-geometric}, we have
   \begin{equation}
   \begin{split}\label{eq:BBTl-powerseries}
   \norm{ \sum_{\ell=1}^\infty \ell c_\ell^2 \sum_{j\neq k } \langle \a_k, \a_j \rangle \langle \bh_k, \bh_j \rangle^{\ell-1} (\Jh_k -\Jh_k') \bh_j }
   &\leq C(1 + \frac{1}{\delta}) \maxn{(\1\1^\top - \I)\circ\A^\top \A} \sum_{\ell=1}^\infty \ell c_\ell^2  \sum_{j  \neq k}\abs{\langle \bh_k, \bh_j \rangle^{\ell-1}}\norm{ \Jh_k' \bh_j }  \\
   &\leq C(1 + \frac{1}{\delta}) \opn{\A}^2 \sum_{j \neq k} \frac{1}{\sqrt{1 - \abs{ \langle \bh_k, \bh_j \rangle}}} \norm{ \Jh_k' \bh_j } \\
   &\leq C(1 + \frac{1}{\delta}) d,
   \end{split}
   \end{equation}
   where the last step follows from \eqref{eq:Jk'}.
Now combining \eqref{eq:aux1} and \eqref{eq:BBTl-powerseries} we can bound the RHS of \eqref{eq:grad-part1-rhs1} by
\begin{equation}\label{eq:grad-est1}
    C\frac{1}{d^\gamma} \max_\m\norm{ - 2 \frac{1}{\sqrt{p}} \m \circ(\Jh_k - \Jh_k') \a_k + 2 \sum_{l=1}^\infty \ell c_\ell^2 \sum_{j \neq k} \langle \a_k, \a_j \rangle \langle \bh_k, \bh_j \rangle^{l-1} (\Jh_k - \Jh_k') \bh_j} \leq CC_X^2 (1 + \frac{1}{\delta}) d^{-(\gamma-1)}.
\end{equation}
From this and 2. in Lemma \ref{lem:linalg-aux}, it follows that
\begin{equation}
\begin{split}\label{eq:gradconc0}
  \opn{\nabla^1_{err}} &\leq C (1 + \frac{1}{\delta}) d^{-(\gamma-1)} \sqrt{d},
    \end{split}
\end{equation}
Now by choosing $\gamma = 3+\gamma_\delta$ the RHS is of of order than  $\Om{d^{-3/2}}$, which finishes bounding $\opn{\nabla^1_{err}}$.

For $\opn{\nabla^2_{err}}$ we need a more nuanced approach. We will break this term in three different parts, $\nabla^2_{err} = -\frac{2}{\sqrt{p}}\M_1+ 2\M_2+ 2\M_3$, in \eqref{eq:M1-definition}, \eqref{eq:M2-definition}, \eqref{eq:M3-definition} below. 
First we consider
\begin{equation}\label{eq:M1-definition}
    (\M_1)_{k, :} :=\m \circ \epsilon_\m^k\Jh_k' \a_k = \epsilon_\m^k \m \circ (\a_k - \bh_k \langle \bh_k , \a_k \rangle),
\end{equation}
so defining $(\D_\epsilon)_{k,k} := \epsilon_\m^k$ can write 
$$\M_1 = \D_\epsilon (\A_\m - \Diag{\Bh \A} \Bh_\m).$$
By 1. in Lemma \ref{lem:linalg-aux}, we can bound 
$$\opn{\M_1} \leq \opn{\D_\epsilon} (\opn{\A} + \opn{\Bh}^2 \opn{\A}) \leq C C_X^2\frac{\log^2(d)}{\sqrt{d}} \left(1+\opn{\Bh}^2\right).$$
By Lemma \ref{lem:master-Bh}, we have 
$$\Em \opn{\Bh}^2 \leq \frac{1}{p}\opn{\Bm}^2 + CC_X^3 \frac{\log^3(d)}{\sqrt{d}} \leq C C_X^2 + CC_X^3 \frac{\log^3(d)}{\sqrt{d}},$$
which gives us 
\begin{equation}\label{eq:gradconc1}
    \Em \opn{\M_1} \leq C C_X^7 \frac{\log^2(d)}{\sqrt{d}}.
\end{equation}
Next, we consider the term 
\begin{equation}\label{eq:M2-definition}
(\M_2)_{k, :} := \epsilon_\m^k \sum_j \langle \a_k, \a_j \rangle \Jh_k'\bh_j +  3c_3^2\langle \a_k, \a_j \rangle \langle \bh_k, \bh_j \rangle^2\Jh_k'\bh_j,
\end{equation}
which we can write as 
$$\M_2 = \D_\epsilon \left(\A^\top \A \Bh - \Diag{\A^\top \A \Bh \Bh^\top }\Bh + 3c_3^2(\A^\top \A)\circ (\Bh\Bh^\top - \I)^{\circ 2}\Bh -  3c_3^2\text{Diag}(\A^\top \A(\Bh\Bh^\top-\I)^{\circ 3})\Bh  \right).$$
One can verify that the RHS satisfies the assumption of Lemma \ref{lem:master-Bh}. Hence, the same reasoning as for $\M_1$ gives that
\begin{equation}\label{eq:gradconc2}
    \Em\opn{\M_2} \leq C C_X^7 \frac{\log^2(d)}{\sqrt{d}}.
\end{equation}

Lastly, define 
\begin{equation}\label{eq:M3-definition}
(\M_3)_{k,:} =  \epsilon_\m^k\sum_{\ell=5}^\infty \ell c_\ell^2 \sum_{j } \langle \a_k, \a_j \rangle \langle \bh_k, \bh_j \rangle^{\ell-1} \Jh_k' \bh_j  .
\end{equation}
Using Lemma \ref{lem:cl-geometric}, we have
   \begin{align}
   \begin{split}\label{eq:M3-1}
   \norm{\epsilon_\m^k \sum_{\ell=5}^\infty \ell c_\ell^2 \sum_{j\neq k } \langle \a_k, \a_j \rangle \langle \bh_k, \bh_j \rangle^{\ell-1} \Jh_k' \bh_j } 
   &\leq CC_X^2 \frac{\log^2(d)}{\sqrt{d}}\maxn{(\1\1^\top - \I)\circ\A^\top \A}\sum_{j\neq k}\sum_{\ell=5}^\infty \ell c_\ell^2  \abs{\langle \bh_k, \bh_j \rangle^{\ell-1}}\norm{ \Jh_k' \bh_j }  \\
   &\leq CC_X^2\frac{\log^2(d)}{\sqrt{d}} \maxn{(\1\1^\top - \I)\circ\A^\top \A}\sum_{j\neq k} \frac{\langle \bh_k, \bh_j \rangle^4}{\sqrt{1 - \abs{ \langle \bh_k, \bh_j \rangle}}} \norm{ \Jh_k' \bh_j } \\
   &\leq C C_X^2 \frac{\log^2(d)}{\sqrt{d}}d\maxn{(\1\1^\top - \I)\circ\A^\top \A}\maxn{(\1\1^\top - \I)\circ\Bh\Bh^\top }^4.
   \end{split}
   \end{align}
Note that
$$\maxn{(\1\1^\top - \I)\circ\Bh\Bh^\top }^4 =\maxn{(\1\1^\top - \I)\circ\Dh\Bm\Bm^\top \Dh}^4
\leq\min\{\opn{\Dh}^8\maxn{(\1\1^\top - \I) \circ\Bm\Bm^\top}^4, 1\}.
$$ 
Thus, by  using 4. in Lemma \ref{lem:lip-conc-application}, with probability at least $1- C\frac{1}{d^2}$ in $\U, \V$, we have
\begin{equation}\label{eq:M3-2}
    \begin{split}
        \Em \maxn{(\1\1^\top - \I)\circ\Bh\Bh^\top }^4 &\leq \P{\opn{\Dh}>\frac{2}{\sqrt{p}}} + \left( \frac{2}{\sqrt{p}}\right)^8\left(  C C_X^2\frac{\log^3(d)}{\sqrt{d}}  +\Om{C_X\opn{\bR}} \right)^4 \\ &\leq Cd^{-\gamma} + C \left(C_X^2\frac{\log^3(d)}{\sqrt{d}}  +C_XC_R \frac{\log^3(d)}{\sqrt{d}}\right)^4\\
        &\leq C \frac{1}{d^{\frac{3}{2}}}.
    \end{split}
\end{equation}
Next, note that under the assumptions of the current lemma we can apply both Lemma \ref{lem:A-concentration} and 3. in Lemma \ref{lem:lip-conc-application} to obtain
\begin{equation}\label{eq:M3-3}
    \maxn{(\1\1^\top - \I)\circ\A^\top \A} \leq C_X^7 \frac{\log^{10}(d)}{\sqrt{d}}  +C_R \frac{\plr}{\sqrt{d}}.
\end{equation}

Combining \eqref{eq:M3-1}, \eqref{eq:M3-2}, \eqref{eq:M3-3}, we obtain

$$\Em\norm{(\M_3)_{k,:}}^2 \leq  CC_X^4 \frac{\log(d)^4}{d} d^2 (C_X^7 \log^{10}(d) + C_R \plr)^2 \frac{1}{d} \frac{1}{d^{3}} \leq C \frac{1}{d^{\frac{5}{2}}}.$$
This now gives us
\begin{equation}\label{eq:gradconc3}
    \Em\opn{\M_3} \leq \Em \sqrt{\sum_k\opn{(\M_3)_{k,:}}^2  } \leq\sqrt{\sum_k\Em \opn{(\M_3)_{k,:}}^2  } \leq C\frac{1}{d^\frac{3}{4}},
\end{equation}
where we have used Jensen's inequality in the second step.
Finally combining \eqref{eq:gradconc0}, \eqref{eq:gradconc1}, \eqref{eq:gradconc2}, \eqref{eq:gradconc3} we can conclude.
\end{proof}

\begin{lemma}[Gradient concentration, Part 2]\label{lem:grad-concentration-part2}
Assume we have $\B = \X + \bR$ with unit norm rows, $\X=\U \S \V^\top, \U, \V $ Haar,  $\tr{\S\S^\top} = n,\opn{\X} \leq C_X$, $\norm{\X}_{max} \leq  CC_X\frac{\log(d)}{\sqrt{d}}$, $\opn{\bR} \leq C_R \frac{\plr}{\sqrt{d}}$. 
Further assume that $\min_{i, j} \abs{\B_{i,j}} \geq \delta>d^{-\gamma_\delta}$ for some $\gamma_\delta>0$. Then, for $d >d_0(C_X, C_R, \gamma_\delta)$ with probability at least $1 - C\frac{1}{d^2}$ in $\U, \V$
\begin{align}
    \frac{1}{2}\nabla_{\B} &=       - \alpha \left(\B\B^\top + \alpha \I\right)^{-2}\B + \alpha\frac{1}{n} \tr{\left(\B\B^\top + \alpha \I\right)^{-2}\B\B^\top} \B  + \Om{C_X^{10}\frac{\log^{10}(d)}{\sqrt{d}}} \label{eq:nablaB-approx-1}\\
    & = - \alpha \left(\X\X^\top + \alpha \I\right)^{-2}\X + \alpha\frac{1}{n} \tr{\left(\X\X^\top + \alpha \I\right)^{-2}\X\X^\top} \X  +\Om{\bR} + \Om{C_X^{10}\frac{\log^{10}(d)}{\sqrt{d}}}\label{eq:nablaB-approx-2},
\end{align}
where $\nabla_{\B}$ was defined in \eqref{eq:grad-formula}.
\end{lemma}

\begin{proof}
    By Lemma \ref{lem:grad-concentration-part1}, we may assume that, up to an error of order $\Om{C_X^7\frac{\log^2(d)}{\sqrt{d}}}$, the gradient is given by \eqref{eq:gradconc-formula}, \eqref{eq:grad1conc-formula} and \eqref{eq:grad2conc-formula}.

    We will start by analysing the first part of the gradient in \eqref{eq:grad1conc-formula} which we restate here for convenience:
    \begin{equation}\label{eq:grad1-formula-copy}
        \frac{1}{2}\nabla^1_{\Bh}
        = -\A^\top + \frac{1}{\sqrt{p}}\A^\top \A\Bh +  \frac{1}{p}\textrm{Diag}(\Bh \A) \Bh   - \frac{1}{\sqrt{p}}\mathrm{Diag}(\A^\top \A\Bh\Bh^\top) \Bh.
    \end{equation}
     By Lemma \eqref{lem:A-concentration},  we have with probability at least $1 - C \frac{1}{d^2}$ in $\U, \V$
            \begin{align*}
            \A &=  \B^\top \left( \B\B^\top + \alpha\I \right)^{-1} +\Om{ C_X^7 \frac{\log^{10}(d)}{\sqrt{d}}} \\
            &= \X^\top \left( \X\X^\top + \alpha\I \right)^{-1} + \Om{ C_X^7 \frac{\log^{10}(d)}{\sqrt{d}}} + \Om{\bR},
        \end{align*}
    where the expectation over $\m$ has not been taken yet.
    Using 1. in Lemma \ref{lem:master-Bh-explicit}, we see that the RHS  in \eqref{eq:grad1-formula-copy} satisfies the assumptions of Lemma \ref{lem:master-Bh} (noting that $\Omega$ is the entire space for 1.), so we have 
    $$
   \Em \frac{1}{2}\nabla^1_{\Bh}
        = \Em -\A^\top + \frac{1}{p}\A^\top \A\Bm 
        +  \frac{1}{p}\textrm{Diag}(\frac{1}{p}\Bm \A) \Bm   
        - \frac{1}{p}\mathrm{Diag}(\A^\top \A\frac{1}{p}\Bm\Bm^\top) \Bm  + \Om{C_X^{10}\frac{\log^{10}(d)}{\sqrt{d}}}.
    $$

    We now estimate $\Em\frac{1}{2}\nabla^1_{\Bh}$. 
    We clearly have $$\Em -\A^\top + \frac{1}{p}\A^\top \A\Bm = -\A^\top + \A^\top \A\B.$$

    For the third term we have by 5. in Lemma \ref{lem:lip-conc-application} that, with probability at least $1-C\frac{1}{d^2}$ in $\m $, 
    $$\textrm{Diag}(\frac{1}{p}\Bm \A) = \beta \I + \Om{C_X^8\frac{\log^{10}(d)}{\sqrt{d}}} + \Om{C_X\bR}, $$
    where $\beta = \frac{1}{n}\tr{\B\A}$, which implies that
    $$\Em \frac{1}{p}\textrm{Diag}(\frac{1}{p}\Bm \A) \Bm = \beta \B + \Om{C_X^9\frac{\log^{10}(d)}{\sqrt{d}}}+ \Om{C_X^2\bR}.$$
    By exactly the same argument, we can use 6. in Lemma \ref{lem:lip-conc-application} and obtain
    $$\Em  \mathrm{Diag}(\A^\top \A\frac{1}{p}\Bm\Bm^\top) \Bm = \tilde{\beta} \B + \Om{C_X^{10}\frac{\log^{10}(d)}{\sqrt{d}}}+ \Om{C_X^3\bR},$$
    where $\tilde{\beta} = \frac{1}{n}\tr{\A^\top\A\B\B^\top}$.

    In total we have
    \begin{align*}   \Em \frac{1}{2}\nabla^1_{\Bh}
        &= -\A^\top + \A^\top \A\B
        +  \beta \B 
        -  \tilde{\beta} \B + \Om{C_X^{10}\frac{\log^{10}(d)}{\sqrt{d}}}+\Om{C_X^3\bR} \\
        &= \left( -\A^\top + \A^\top \A\X
        +   \beta \X 
        -  \tilde{\beta} \X \right) + \Om{C_X^{10}\frac{\log^{10}(d)}{\sqrt{d}}} + \Om{C_X^3\bR}.
        \end{align*}
   Plugging in  $\A =  \B^\top \left( \B\B^\top + \alpha\I \right)^{-1} +\Om{ C_X^7 \frac{\log^{10}(d)}{\sqrt{d}}}$ in the second term and $\A  = \X^\top \left( \X\X^\top + \alpha\I \right)^{-1} + \Om{ C_X^7 \frac{\log^{10}(d)}{\sqrt{d}}} + \Om{\bR}$ in the third term, we obtain the leading order terms for \eqref{eq:nablaB-approx-1}. 

  To see that this also implies \eqref{eq:nablaB-approx-2} note that $\beta = \frac{1}{n}\tr{\B\A} = \frac{1}{n}\tr{\X\A} + \Om{\bR}$ and $\tilde{\beta} = \frac{1}{n}\tr{\A^\top\A\B\B^\top} =  \frac{1}{n}\tr{\A^\top\A\X\X^\top} + \Om{C_X \bR}$ .
           
    It remains to show that the higher order terms are small. Here we will not need to distinguish between the two approximations of $\A$.
    The remaining part of the gradient in \eqref{eq:grad2conc-formula} is given by
    $$\sum_{\ell \geq 3} c_{\ell}^2\ell\nabla^\ell_{\Bh},$$
    where
    \begin{equation}\label{lem:grad2-formula-copy}
           \frac{1}{2}\nabla^\ell_{\Bh} = \frac{1}{\sqrt{p}}(\A^\top \A)\circ (\Bh\Bh^\top - \I)^{\circ (\ell-1)}\Bh - \frac{1}{\sqrt{p}}\text{Diag}(\A^\top \A(\Bh\Bh^\top-\I)^{\circ \ell})\Bh.
    \end{equation}
       Let $\Omega = \{\M| (\1\1^\top-\I)\circ(\M\M^\top) = \Y + \Z, \maxn{\Y} \leq CC_X^2 \frac{\log^3(d)}{\sqrt{d}},\opn{\Z}\leq CC_XC_R\frac{\plr}{\sqrt{d}} \} $. Then, by 4. in Lemma \ref{lem:lip-conc-application},
       $\P{\Bm \in \Omega}\geq 1 - C\frac{1}{d^2}$.
    Thus, by Lemma \ref{lem:master-Bh-explicit}, we have
    \begin{equation}
        \begin{split}\label{eq:gradpart2-lgeq3}
                     \frac{1}{2}&\sum_{\ell\ge 3} c_\ell^2 \ell\nabla^\ell_{\Bh} = \frac{1}{\sqrt{p}}\sum_{\ell\ge 3} c_\ell^2\ell\mathbbm{1}_{\{\Bm \in \Omega \}}  \\
         &\cdot \left((\A^\top \A)\circ \left(\left(\1\1^\top - \I \right) \circ \frac{1}{p}\Bm\Bm^\top\right)^{\circ (\ell-1)} 
         - \text{Diag}\left(\A^\top \A \left( \left(\1\1^\top - \I \right) \circ\frac{1}{p}\Bm\Bm^\top\right)^{\circ \ell}\right)\right)\frac{1}{\sqrt{p}} \Bm+ \Om{C_X^3 \frac{\log^3(d)}{\sqrt{d}}}.
        \end{split}
    \end{equation}

    We will now individually bound the different terms. In the following we always assume $\ell \geq 3$.
    We first analyse the term 
    $$(\A^\top \A)\circ \left(\left(\1\1^\top - \I \right) \circ \frac{1}{p}\Bm\Bm^\top\right)^{\circ (\ell-1)}. $$
    We had previously derived the following in \eqref{eq:BBT3-bound} 
    \begin{equation}\label{eq:BBT3-bound-copy}
    \left(\frac{1}{p}(\1\1^\top -\I)\circ\Bm\Bm^\top \right)^{\circ \ell-1}  = \frac{1}{p^{\ell-1}}\Y^{\circ (\ell-1)} + \frac{1}{p^{\ell -1}}\ell\Y^{\circ (\ell-2)}\circ \left(\Z+ e^\ell\Om{\Z^2}\right).
    \end{equation}
     Thus, as in 3. of Lemma \ref{lem:master-Bh-explicit} we obtain from Lemma \ref{lem:A-concentration}
     $$    \A^\top \A = O_{max}\left(\frac{\log(d)}{\sqrt{d}}\right) + \Om{C_X^7C_R  \frac{\plr}{\sqrt{d}}},$$
     so
\begin{equation}
    \begin{split}
    &\opn{(\A^\top \A)\circ \left(\frac{1}{p}(\1\1^\top -\I)\circ\Bm\Bm^\top \right)^{\circ (\ell-1)} } \leq\opn{O_{max}\left(\frac{\log(d)}{\sqrt{d}}\right) \circ \left(\frac{1}{p}(\1\1^\top -\I)\circ\Bm\Bm^\top \right)^{\circ (\ell-1)}} \\
    &+  \opn{\Om{C_X^7C_R  \frac{\plr}{\sqrt{d}}}\circ\left(\frac{1}{p}(\1\1^\top -\I)\circ\Bm\Bm^\top \right)^{\circ (\ell-1)}}.
    \end{split}
\end{equation}

    Plugging in \eqref{eq:BBT3-bound-copy} and using 3. and 5. in Lemma \ref{lem:linalg-aux} we obtain
    \begin{equation}
    \begin{split}\label{eq:gradpart2-aux1}
        \ell&\opn{O_{max}\left(\frac{\log(d)}{\sqrt{d}}\right) \circ \left(\frac{1}{p}(\1\1^\top -\I)\circ\Bm\Bm^\top \right)^{\circ (\ell-1)}} 
        \leq C\frac{\ell}{p^{\ell-1}}\left(n  \frac{\log(d)}{\sqrt{d}}\maxn{\Y^{\circ (\ell-1)}} + \ell e^\ell\sqrt{n} \opn{\Z} \frac{\log(d)}{\sqrt{d}}\maxn{ \Y^{\circ(\ell-2)}}   \right) \\
        &\leq C \frac{\ell}{p^{\ell-1}} \left(C_X^{2(\ell-1)}\frac{\log^{1 + 3  (\ell-1)}(d)}{d^{(\ell-2)/2}} +CC_X^{1+2(\ell-2)}C_Re^\ell\ell\frac{\plr \log^{1+ 3(\ell-2)}(d)}{d^{(\ell-1)/2}}\right)\\
        &\leq CC_X^4 \frac{\log^7(d)}{\sqrt{d}}.
    \end{split}
    \end{equation}
    Similarly, we have 
    \begin{equation}
        \begin{split}\label{eq:gradpart2-aux2}
            &\ell\opn{\Om{C_X^7C_R  \frac{\plr}{\sqrt{d}}}\circ\left(\frac{1}{p}(\1\1^\top -\I)\circ\Bm\Bm^\top \right)^{\circ (\ell-1)}} \\
            &\leq C\frac{\ell}{p^{\ell-1}} \left(\sqrt{n}C_X^7C_R  \frac{\plr}{\sqrt{d}}\maxn{\Y^{\circ (\ell-1)}} +\ell\sqrt{n}C_X^7C_R  \frac{\plr}{\sqrt{d}} \maxn{\Z \circ\Y^{\circ (\ell-2)}} \right)\\
            & \leq  C \frac{\ell}{p^{\ell-1}} \left(C_X^{7+2(\ell-1)}C_R\frac{\plr\log^{3  (\ell-1)}(d)}{d^{(\ell-1)/2}} +CC_X^{7+2(\ell-1)}C_R\ell\frac{(\plr)^2 \log^{3(\ell-1)}(d)}{d^{(\ell-1)/2}}\right)\\
            & \leq \frac{1}{\sqrt{d}}.
        \end{split}
    \end{equation}

Next we have
$$\opn{\text{Diag}\left(\A^\top \A \left( \left(\1\1^\top - \I \right) \circ\frac{1}{p}\Bm\Bm^\top\right)^{\circ \ell}\right)} \leq C \opn{\left( \left(\1\1^\top - \I \right) \circ\frac{1}{p}\Bm\Bm^\top\right)^{\circ \ell}}.$$
Now exactly as in the proof of \eqref{eq:BmBmT-entrywise-bound} we obtain
\begin{equation}\label{eq:gradpart2-aux3}
  C \ell\opn{\left( \left(\1\1^\top - \I \right) \circ\frac{1}{p}\Bm\Bm^\top\right)^{\circ \ell}}  \leq CC_X^6\frac{\log^{10}(d)}{\sqrt{d}}.
\end{equation}
(Note that when writing out the proof, the $\ell$ factor is trivially absorbed for $\ell\geq 4$.)

Finally, we can combine \eqref{eq:gradpart2-aux1}, \eqref{eq:gradpart2-aux2}, \eqref{eq:gradpart2-aux3} to obtain that the RHS of \eqref{eq:gradpart2-lgeq3} is of order  $O(C_X^7\frac{\log^{10}(d)}{\sqrt{d}})$, where we get an extra $C_X$ from bounding the operator norm of $\Bm$. Thus, using that $\sum_{\ell \geq 3} c_\ell^2 < \infty$, we conclude
$$\sum_{\ell \geq 3} c_{\ell}^2\ell\nabla^\ell_{\Bh} = \Om{C_X^7 \frac{\log^{10}(d)}{\sqrt{d}}},$$
which finishes the proof.

\end{proof}

\subsection{GD-analysis and reduction to Gaussian}\label{app:gauss-reduction}
To simplify the notation we will push the time dependence in the subscript, i.e. $\B_t=\B(t)$.
\begin{theorem}[Gaussian recursion]\label{thm:Gauss-GD-min}
If the entries $(\B_0')_{i,j} \sim \mathcal{N}(0, \frac{1}{d}) $ are i.i.d., $\B_0 = \mathrm{proj}(\B_0')$
and
\begin{equation}\label{eq:app1}
\nabla_{\B} \B_t^\top = - \alpha \left(\B_t\B_t^\top + \alpha \I\right)^{-2}\B_t\B_t^\top +\alpha \Diag {\left(\B_t\B_t^\top + \alpha \I\right)^{-2}\B_t\B_t^\top} \B_t\B_t^\top + \tilde{\bE}_t,
\end{equation}
\begin{equation}\label{eq:app2}
\B_t\B_t^\top = \I + \Z_t + \bE_t,
\end{equation}
with $\Z_t = \U (\boldsymbol{\Lambda}_t-\I) \U^\top$, $\U$ a Haar matrix and 
$$    \opn{\tilde{\bE}^t  }\le C_E \left(\frac{\pl{E}}{\sqrt{d}} \cdot \opn{\Z_t}^{1/2}+\opn{\bE_t}^2+\opn{\bE_t}\opn{\Z_t}^{1/2}\right).$$
Consider the GD-min algorithm in \eqref{eq:GDmin-formulas} without noise ($\G_t=\0$ for all $t$) and on the Gaussian objective (i.e., $p=1$). Pick a learning rate $\eta = C/\sqrt{d}$. Then, with probability at least $1- C\expbr{-cd}$, we have that, jointly for all $t \geq 0$,
    \begin{align}\label{eq:assumptionZX}
    \begin{split}
        & \opn{\bE_t} \leq C_{E}e^{-c \eta t}.\frac{\pl{E}}{\sqrt{d}},\\
        &\opn{\Z_t} \leq C_Ze^{-c \eta t}.
    \end{split}
    \end{align}
\end{theorem}
\begin{proof}
   The claim follows from the analysis in Appendix E of \cite{shevchenko2023fundamental}. First, note that here $\tilde{\bE}_t$ and $\bE_t$ respectively correspond to $\bE_t$ and $\X_t$ in (90) in \cite{shevchenko2023fundamental}. Then, the assumptions of our theorem correspond to the conclusion of Lemma E.4 and Lemma E.5. The projection step is handled in Lemma E.6 and the recursion is analysed in Lemma E.7.
\end{proof}

\begin{lemma}[Reduction to Gaussian recursion]\label{lem:reduction-togaussian}
    Fix $t_{\rm max} = T_{\rm max}/\eta, T_{\rm max} \in (0, \infty)$, let $(\B_0')_{i,j} \sim \mathcal{N}(0, \frac{1}{\sqrt{d}}) $ i.i.d., $\B_0 = \mathrm{proj}(\B_0')$ and assume that for $t \leq t_{\rm max}$ we have 
    \begin{equation}\label{eq:reduction-BBT}
    \nabla_\B \B_t^\top +\G_t = - \alpha \left(\B_t\B_t^\top + \alpha \I\right)^{-2}\B_t\B_t^\top +\alpha \Diag {\left(\B_t\B_t^\top + \alpha \I\right)^{-2}\B_t\B_t^\top} \B_t\B_t^\top  + \Om{C_R(T_{\rm max})\frac{\plr}{\sqrt{d}}}.
    \end{equation}
    Consider the GD-min algorithm in \eqref{eq:GDmin-formulas} for any $p\in (0, 1)$. Pick a learning rate $\eta = C/\sqrt{d}$. Then, with probability at least $1-C\expbr{-cd}$, we have that, jointly for all $t\ge 0$, \eqref{eq:app2} holds with 
        \begin{align}\label{eq:assumptionZXreduction}
    \begin{split}
        & \opn{\bE_t} \leq C_Ee^{-c \eta t}.\frac{\pl{E}}{\sqrt{d}},\\
        &\opn{\Z_t } \leq C_Z e^{-c \eta t},
    \end{split}
    \end{align}
    where crucially $C_E, \mathrm{poly}_E , C_Z$ are independent of $d$, and $C_Z$ is independent of $T_{\rm max}$.
\end{lemma}
\begin{proof}
The claim follows from Theorem \ref{thm:Gauss-GD-min} after showing that
\begin{equation}\label{eq:cond1}    
C_R(T_{\rm max})\frac{\plr}{\sqrt{d}}\le C_E \left(\frac{\pl{E}}{\sqrt{d}} \cdot \opn{\Z_t}^{1/2}+\opn{\bE_t}^2+\opn{\bE_t}\opn{\Z_t}^{1/2}\right),
\end{equation}
where $\opn{\bE_t},\opn{\Z_t}$ satisfy \eqref{eq:assumptionZX}. Now, \eqref{eq:cond1} trivially holds if $\opn{\Z_t} \geq c_Z(T_{\rm max})>0$, where $c_Z(T_{\rm max})$ is independent of $d$.

It remains to show the lower bound on $\opn{\Z_t}$. This can be readily seen by analyzing the deterministic recursion of the spectrum of $\Z$ as in Lemma G.3 in \cite{shevchenko2023fundamental}.
First, for sufficiently large $d$, $\eta$ gets arbitrarily small, hence we can approximate such discrete recursion with its continouous analogue. Next, we linearize the continuous evolution since $\Z$ is small (otherwise we already have the desired lower bound). Since the coefficient of the linearization is strictly negative (and, hence, bounded away from $0$), we readily have that $\opn{\Z_t}$ cannot reach $0$ in finite time. 
\end{proof}

For technical reasons, we need the following lemma that shows that the spectrum of $\B$ a priori cannot grow faster than exponentially in the effective time of the dynamics. The proof is a non-tight analog of the analysis done in Lemma E.7 and G.3 in \cite{shevchenko2023fundamental} for $\B$ instead of $\B\B^\top$.

\begin{lemma}[Spectrum evolution of $\B$]\label{lem:B-specevo}
Consider the GD-min algorithm in \eqref{eq:GDmin-formulas} for any $p\in (0, 1)$. Pick a learning rate $\eta = C/\sqrt{d}$. Under the gradient approximation given in \eqref{eq:nablaB-approx-2} with $C_X(t) := \expbr{C \eta t}\opn{\B_0}$, 
we have that, for $t \leq t_{\rm max}$ and $d > d_0(C_X(t_{\max}))$,
$$\B_t = \X_t + \bR_t,$$ where $\X_t$ has the same singular vectors as $\B_0$,
$$\opn{\X_t} \leq C_X(t), \quad \mathrm{and} \quad \opn{\bR_t} \leq CC_X^7(t_{\rm max}) \expbr {C T_{\rm max}}\frac{\log^{10}(d)}{\sqrt{d}},$$
with probability at least $1- C(\eta t_{\rm max}) \frac{1}{d^{3/2}}$.
\end{lemma}

\begin{proof}
        Consider the recursion where the gradient is given below:

\begin{equation}\label{eq:recfake}
        \frac{1}{2}\nabla_{\B} := 
      - \alpha \left(\X\X^\top + \alpha \I\right)^{-2}\X + \alpha\frac{1}{n} \tr{\left(\X\X^\top + \alpha \I\right)^{-2}\X\X^\top} \X .\end{equation}    
      It is evident that this recursion only updates the singular values $s^i$ of $\B$ as the RHS has the same singular vectors as $\B$.
      Furthermore, the update equation for the $s^i$ is given by
      $$s^i_{t+1} = s^i_t - \eta \left(-\alpha \frac{s^i}{(\alpha + (s^i_t)^2)^2} + s^i_t\frac{1}{n} \sum_{i=1}^n\frac{(s^i_t)^2}{(\alpha + (s^i_t)^2)^2} \right).$$
      Note that 
      $$\abs{s^i_{t+1} - s^i_t}\leq C \eta \abs{s^i_t}.$$
      Thus, letting $b_t  = \opn{\B_t}$, the above implies that
        \begin{equation}\label{eq:bt-ode-prelim}
            b_{t+1} \le (1 + C\eta) b_t, 
        \end{equation}
        which by monotonicity gives that $b_{t} \le (1 + C\eta)^{t} b_0$. Hence, if the recursion of the gradient was actually given by \eqref{eq:recfake}, the claim would immediately follow. 

Now, the recursion of the gradient is given by \eqref{eq:nablaB-approx-2}.  Thus, to deal with the error, we can follow the strategy of Lemma E.7 in \cite{shevchenko2023fundamental}. In particular,
    denoting $r_t := \opn{\bR_t}, \epsilon_d :=C_X^{10} \frac{\log^{10}(d)}{\sqrt{d}}$, the evolution of the error is given by 
    $$r_{t+1} \leq (1+C_1\eta ) r_t + C_2\epsilon_d.$$
    By monotonicity, this recursion is upper bounded by the solution of
    $$r_{t+1} = (1+C_1\eta ) r_t + C_2\epsilon_d .$$
    Since the recursion is initialized with $r_0 = 0$, we can unroll it as
    \begin{align*}
    r_{t+1}& = C_2\sum_{i=1}^t (1 + C_1\eta)^i \eta  \epsilon_d \\
        &\leq C_2\sum_{i=1}^t \expbr{C_1\eta i} \eta \epsilon_d\\
        &= C_2\expbr{C_1\eta t} \sum_{i=1}^t \expbr{-C_1 \eta (t - i)}\eta \epsilon_d \\
        &\leq  C_2 \expbr {C_1 \eta t} \frac{1}{1 - \expbr{-C_1 \eta}} \eta \epsilon_d, 
    \end{align*}
    where we have used $1 + x \leq \expbr{x}$. For small enough $\eta$, we have $\frac{1}{1 - \expbr{-C_1 \eta}}  \leq \frac{1}{C_1 \eta} $. Hence,
    \begin{equation}\label{eq:rt-ode}
         r_{t+1} \leq \frac{C_2}{C_1} \expbr {C_1 \eta t} \epsilon_d,
    \end{equation}
which gives that $\opn{\bR_t} \leq C_X^{10}\frac{C_2}{C_1} \expbr {C_1 \eta t}\frac{\log^{10}(d)}{\sqrt{d}}$, as required. Hence, by \eqref{eq:nablaB-approx-2}, $b_t$ is upper bounded by the solution to the recursion    
      $$b_{t+1} = (1 + C_1\eta) b_t + C_2 \eta \expbr {C_3 \eta t} \epsilon_d.$$
      As $\tr{\B\B^\top} = n$, we have that $b_t \geq 1$. Thus,
      $$b_{t+1} \leq \left(1 + \left( C_1 +  C_2 \expbr {C_3 \eta t} \epsilon_d \right) \eta\right) b_t \leq \left(1 + \left( C_1 +  C_2 \expbr {C_3 T_{\rm max}} \epsilon_d \right) \eta\right) b_t.$$
        Taking a sufficiently large $d$ gives that $ C_2 \expbr {C_3 T_{\rm max}} \epsilon_d \leq C_1$, which leads to
        $$b_{t+1} \leq (1 + 2C_1\eta ) b_t. $$
            Using again monotonicity and $1 + x \leq \expbr{x}$, we conclude that
      $C_X:=\opn{\B_{t_{\rm max}}} \leq \expbr{2C_1 \eta t}\opn{\B_0}$. 

      This proves the claim of the lemma for a gradient recursion given exactly by \eqref{eq:nablaB-approx-2}. We note that the GD-min algorithm in \eqref{eq:GDmin-formulas} has two additional steps: 
\emph{(i)} adding noise $\G_t$ at each step, and \emph{(ii)} the projections step, which normalizes the rows of $\B_t$ after the gradient update.

As for \emph{(i)}, let $\G$ be an $n \times d$ matrix with i.i.d.\ $\mathcal{N}(0, \sigma^2)$ entries. Then, by Theorem 4.4.5 in \cite{vershynin2018high} (with $t = \sqrt{d}$), we have that $\opn{\G} \leq C\sigma (\sqrt{n} + \sqrt{d}) \leq C \sigma \sqrt{d}$ with probability at least $1-C\frac{1}{d^2}$. Recall that in \eqref{eq:GDmin-formulas}
we assume that $\sigma \leq C\frac{1}{d}$. Hence, the additional error from the noise is of higher order than all the other error terms and can be neglected.
    By a union bound over $T_{\rm max}/\eta$ steps, the above bound holds for all time steps with probability at least $1- C\frac{1}{d^{3/2}}$.

 As for \emph{(ii)}, a straightforward analysis shows that $\opn{\mathrm{proj}(\B') - \B'} \leq C \eta^2 \opn{\nabla_{\B} + \G}^2$, which is also of higher order. We skip the details here and refer to Lemma E.6 in \cite{shevchenko2023fundamental}. This concludes the proof.
\end{proof}

We are now ready to give the proof of Theorem \ref{thm:GD-min-sparse} by combining the previous results and carrying out an induction over the time steps.

\begin{proof}[Proof of Theorem \ref{thm:GD-min-sparse}]\label{proof:Gd-min-sparse}
We fix $p \in (0, 1)$ and $t_{\rm max} = T_{\rm max}/\eta, T_{\rm max} \in (0, \infty)$. 
We want to show that the assumptions of Lemma \ref{lem:reduction-togaussian} are satisfied, as the conclusion of Lemma \ref{lem:reduction-togaussian} is precisely \eqref{eq:app-thm1}.

By Lemma \ref{lem:B-specevo}, we have that, with probability at least $1 -C(T_{\rm max})\frac{1}{d^{3/2}}$, for all $t \leq t_{\rm max}$, $C_X(t) := \expbr{C \eta t}\opn{\B_0}$, and $C_R(t)= CC_X^{10}(t_{\rm max}) \expbr {C \eta t}\frac{\plr}{\sqrt{d}}.$
Furthermore, by choosing $\delta = d^{-(4 + \gamma_g)}$, we can apply Lemma \ref{lem:G-bound} for each step so that, with probability at least $1- C\frac{1}{d^{3/2}}$, $\min_{i,j} \abs{\B'_{ij}} \geq 2\delta$. Note that the projection step does not change the scale of any entry by more than a factor that converges to $1$ as $d$ grows large (see Lemma E.6 in \cite{shevchenko2023fundamental} for details), so in particular $\min_{i,j} \abs{\B'_{ij}} \geq 2\delta$ implies $\min_{i,j} \abs{\B_{ij}} \geq \delta$.
This gives that, with probability at least $1- C(T_{\rm max}) \frac{1}{d^{3/2}}$, for all $t \leq t_{\rm max}$, the assumptions of Lemma \ref{lem:grad-concentration-part2} are satisfied, hence \eqref{eq:nablaB-approx-1} and \eqref{eq:nablaB-approx-2} hold.

By 2. in Lemma  \ref{lem:lip-conc-application}, at each step with probability $1-C\frac{1}{d^2}$, we have that  $$ \Diag {\left(\B\B^\top + \alpha \I\right)^{-2}\B\B^\top} = \frac{1}{n} \tr{\left(\B\B^\top + \alpha \I\right)^{-2}\B\B^\top}\I + C\expbr{C T_{\rm max}}\opn{\B_0} \frac{\log(d)}{\sqrt{d}},$$ so this holds jointly for all $t \leq t_{\rm max}$ with probability at least $1- C\frac{1}{d^{\frac{3}{2}}}$
Combining this with \eqref{eq:nablaB-approx-1}, we conclude that, with probability at least $1- C(T_{\rm max})\frac{1}{d^{\frac{3}{2}}}$, the assumptions of lemma \ref{lem:reduction-togaussian} hold, which immediately implies
 $$\lim_{d \to \infty} \sup_{t \in [0, t_{\rm max}]} \norm{\bR_t} = 0$$
 and
 $$\opn{\S_t\S_t^\top - \I} \leq C \expbr{-c T_{\rm max}}.$$
 This proves \eqref{eq:app-thm1}.

To prove \eqref{eq:appG1-app}, we note that the combination of \eqref{eq:A-approx}, \eqref{eq:fBh} and \eqref{eq:fBmBm} gives
 \begin{equation}
    \E_{\m}\left[\tr{\A^\top \A f(\hat{\B}\hat{\B}^\top)} - \frac{2}{\sqrt{p}} \tr{\A \hat{\B}}\right] = \tr{\A^\top \A f(\B\B^\top)} - 2 \tr{\A \B^\top} + \Om{C(T_{\rm max})d\frac{\mathrm{poly}(\log(d))}{\sqrt{d}}}.
 \end{equation}
 Since \eqref{eq:GDmin-obj} and \eqref{eq:mse} differ only by a constant and a factor $1/d$, the above implies that, for any $p \in (0,1)$, \eqref{eq:mse}  is close to the Gaussian objective up to an error $C(T_{\rm max})\frac{\mathrm{poly}(\log(d))}{\sqrt{d}}$.
The fact that the 
evolution of $
\B$ matches the Gaussian case is also clear, since the gradient approximation in Lemma \ref{lem:grad-concentration-part2} coincides with the Gaussian recursion in Theorem \ref{thm:Gauss-GD-min}.
\end{proof}

\section{MSE characterizations} \label{appendix:mse_amp_comp}

\subsection{Proof of Proposition \ref{proposition:identity_is_better}}\label{subsec:idbetter}

Denote by $\x^1$ the first iterate of the RI-GAMP algorithm \cite{venkataramanan2022estimation}, as in \eqref{eq:RIGAMPt}. Then, by taking $\sigma$ to be the sign, one can readily verify that
$$
\x^1 = \B^\top \mathrm{sign}(\B\x).
$$
Note that $\B$ is bi-rotationally invariant in law and, as $\x$ has i.i.d.\ components, its empirical distribution converges in Wasserstein-2 distance to a random variable whose law is that of the first component of $\x$, denoted by $x_1$. Therefore, the assumptions of Theorem 3.1 in \cite{venkataramanan2022estimation} are satisfied. Hence, for any $\psi$ pseudo-Lipschitz of order $2$,\footnote{We recall that $\psi:\mathbb R^2\to\mathbb R$ is pseudo-Lipschitz of order $2$ if, for all $\a, \b\in \mathbb R^2$, $|\psi(\a)-\psi(\b)|\le L\|\a-\b\|_2(1+\|\a\|_2+\|\b\|_2)$ for some constant $L>0$.} we have that, almost surely,   
$$
 \lim_{d\rightarrow\infty}\frac{1}{d} 
 \sum_{i=1}^d \psi((\x^1)_i, (\x)_i)=\E[\psi(\mu x_1 + \sigma g, x_1)],
$$ 
where $g\sim\mathcal N(0, 1)$ is independent of $x_1$ and the state evolution parameters $(\mu, \sigma)$ for $r\leq 1$ can be computed as
\begin{equation}\label{eq:state_evol_params_appendix}
    \mu = r\cdot\sqrt{\frac{2\kappa_2}{\pi}} = r \cdot \sqrt{\frac{2}{\pi}}, \quad \sigma^2 = r \cdot \left(\kappa_2 + \kappa_4 \cdot \frac{2}{\pi\kappa_2}\right) = r \cdot \left(1 - r \cdot \frac{2}{\pi}\right),
\end{equation}
that is equation ({\color{mydarkblue} 11}) in \cite{venkataramanan2022estimation}. Here, $\{\kappa_{2k}\}_{k\in\mathbb N}$ denote the rectangular free cumulants of the constant random variable equal to $1$ (since all the singular values of $\B$ are equal to $1$ by assumption).
 Noting that $\psi(x, y)=(x-\alpha\cdot y)^2$ is pseudo-Lipschitz of order $2$, we get
that, almost surely,
$$
 \lim_{d\rightarrow\infty}\frac{1}{d} \cdot \|\x - \alpha\cdot\B^\top \mathrm{sign}(\B^\top \x)\|_2^2 = \E_{x_1, g} [|x_1 - \alpha (\mu x_1 + \sigma g)|_2^2],
$$
 which implies that
$$
 \lim_{d\rightarrow\infty}\frac{1}{d} \cdot \E_{\x}\|\x - \alpha\cdot\B^\top \mathrm{sign}(\B^\top \x)\|_2^2 = \E_{x_1, g}[ |x_1 - \alpha (\mu x_1 + \sigma g)|_2^2].
$$
By expanding the RHS of the last equation and using that $x_1$ has unit second moment by assumption, we get
\begin{align*}
    \E_{x_1, g}[ |x_1 - \alpha (\mu x_1 + \sigma g)|_2^2] &= 
(1-\alpha\mu)^2 \cdot \E [x_1^2] + \alpha^2\sigma^2 \cdot \E [g^2] = (1-\alpha\mu)^2 + \alpha^2\sigma^2\\
&= 1 - 2\alpha\mu + \alpha^2(\mu^2+\sigma^2) = 1 - 2\alpha \cdot r\sqrt{\frac{2}{\pi}} + \alpha^2 r.
\end{align*}
Thus, by minimizing over $\alpha$, we have
$$
\min_{\alpha}\E_{x_1, g} [|x - \alpha (\mu x + \sigma g)|_2^2] = 1 - \frac{2}{\pi} \cdot r,
$$
which concludes the proof of \eqref{eq:Haarloss}.

To prove \eqref{eq:idloss}, a direct calculation gives
\begin{align*}
    \frac{1}{d} \cdot \E_\x\left\|\x - \alpha \cdot \begin{bmatrix}
        \I_n\\
        \boldsymbol{0}_{(d-n)\times n} 
    \end{bmatrix}\mathrm{sign}([\I_{n}, \boldsymbol{0}_{n \times (d-n)}]\x)\right\|_2^2 &= 1 - r + r \cdot \E[ (x_1 - \alpha \mathrm{sign}(x_1))^2] \\
    &= 1 - r + r \cdot (\E[x_1^2] - 2\alpha \cdot \E[|x_1|] + \alpha^2 \cdot \E [\mathrm{sign}^2(x_1)]) \\
    &= 1-r+r\cdot (1 - 2\alpha \cdot \E[|x_1|] + \alpha^2) \\
    &= 1 + r\cdot(\alpha^2-2\alpha \cdot \E[|x_1|]).
\end{align*}
The RHS is minimized by $\alpha=\E[|x_1|]$, which gives
\begin{align*}
   \min_\alpha \frac{1}{d} \cdot \E_\x\left\|\x - \alpha \cdot \begin{bmatrix}
        \I_n\\
        \boldsymbol{0}_{(d-n)\times n} 
    \end{bmatrix}\mathrm{sign}([\I_{n}, \boldsymbol{0}_{n \times (d-n)}]\x)\right\|_2^2  = 1 - r \cdot (\E [|x_1|])^2,
\end{align*}
and the proof is complete. \qed

\subsection{Proof of Proposition \ref{proposition:1}}\label{subsec:MSEden}

Let $\hat \x^1$ be an iterate of the RI-GAMP algorithm \cite{venkataramanan2022estimation}, as in \eqref{eq:RIGAMPt}. Then, by taking $\sigma$ to be the sign and $f_t=f$, one can readily verify that
$$
\hat\x^1 =f( \B^\top \mathrm{sign}(\B\x)),
$$
which is exactly the form of the autoencoder in \eqref{eq:linear_decoding_denoising} that we wish to analyze. Thus, as $f$ is Lipschitz, the assumptions of 
Theorem 3.1 in \cite{venkataramanan2022estimation} are satisfied and, following the same passages as in the proof of Proposition \ref{proposition:identity_is_better}, we have  
\begin{equation}\label{eq:appendix_highdim_limit}
    \lim_{d\rightarrow\infty}\frac{1}{d} \cdot \E_{\x}\|\x - f( \B^\top \mathrm{sign}(\B\x))\|_2^2 = \E_{x_1, g} [|x_1 - f(\mu x_1 + \sigma g)|_2^2],
\end{equation}
where $x_1$ is the first entry of $\x$, $g\sim\mathcal N(0, 1)$ is independent of $x_1$, and $(\mu, \sigma)$ are given by \eqref{eq:state_evol_params_appendix} (which coincides with \eqref{eq:mu_sigma}). This concludes the proof. \qed

\subsection{Proof of Proposition \ref{proposition:sparse_rademacher_id_denoising}}\label{subsec:MSEid}

A direct calculation gives
\begin{align}\label{eq:prop5.3_1}
    \frac{1}{d} \cdot \E_\x\left\|\x - f\left(  \begin{bmatrix}
        \I_n\\
        \boldsymbol{0}_{(d-n)\times n} 
    \end{bmatrix}\mathrm{sign}([\I_{n}, \boldsymbol{0}_{n \times (d-n)}]\x)\right)\right\|_2^2 &= (1-r) \cdot \E\left[(x_1-f(0))^2\right] + r \cdot \E \left[(x_1 - f(\mathrm{sign}(x_1)))^2\right],
\end{align}
where $x_1$ is the first entry of $\x$.
The first term in \eqref{eq:prop5.3_1} is minimized when $f(0) = \E [x] = 0$. Hence, we obtain that, at the optimum,
$$
(1-r) \cdot \E\left[(x_1-f(0))^2\right] = 1 - r,
$$
as $\E [x^2] = 1$. As for the second term in \eqref{eq:prop5.3_1}, we rewrite
\begin{equation}\label{eq:prop5.3_2}
    \E \left[(x_1 - f(\mathrm{sign}(x_1)))^2\right] = \mu_{x_1}(\{0\}) \cdot \frac{1}{2} \cdot (f(1)^2 + f(-1)^2) + \E [\mathbbm{1}_{x_1>0}(x_1 - f(1))^2] + \E [\mathbbm{1}_{x_1<0}(x_1 - f(-1))^2],
\end{equation}
where $\mu_{x_1}$ stands for the measure that corresponds to the distribution of $x_1$, and we use that $\mathrm{sign}(0)$ is a Rademacher random variable by convention. As the distribution of $x_1$ is the same as that of $-x_1$, \eqref{eq:prop5.3_2} is minimized by taking $f(1)=-f(-1)$. Thus, we have that
$$
\min_{f} \eqref{eq:prop5.3_2} = \min_{u \in \mathbb{R}} \E [(x_1 - u \cdot \mathrm{sign}(x_1))^2].
$$
The RHS of this last expression can be further rewritten as
$$
\min_{u \in \mathbb{R}} \E [(x_1 - u\cdot \mathrm{sign}(x_1))^2] =  \E [x_1^2] + \min_{u \in \mathbb{R}} \left\{ u^2 - 2 u \cdot \E |x_1| \right\} = 1 - (\E |x_1|)^2,
$$
which concludes the proof. \qed

\subsection{Computation of $f^*$}\label{appendix:denoiser_computations}

\paragraph{Sparse Gaussian.} Using Bayes rule, the conditional expectation can be expressed as
\begin{equation}\label{eq:appendix_sg_comp_1}
    \E[x|\mu x + \sigma g = y] = \frac{\E_x\left[ x \cdot P(\mu x + \sigma g = y|x)\right]}{\E_x\left[ P(\mu x + \sigma g = y|x)\right]} = \frac{\E_x\left[ x \cdot P(\mu x + \sigma g=y|x)\right]}{P(\mu x + \sigma g=y)}.
\end{equation}
Given that $x\sim {\rm SG}_1(p)$, with probability $p$ we have that $\mu x + \sigma g \sim \mathcal{N}(0, \mu^2/p + \sigma^2)$ as $x \sim \mathcal{N}(0,1/p)$, and with probability $(1-p)$ we have that $x=0$, and, hence, $\mu x + \sigma g = \sigma g \sim \mathcal{N}(0,\sigma^2)$. Combining gives
$$
P(\mu x + \sigma g = y) = p \cdot \frac{\sqrt{p}}{\sqrt{2\pi(\mu^2+p\sigma^2)}} \cdot \exp\left(-\frac{py^2}{2(\mu^2+p\sigma^2)}\right) + (1-p) \cdot \frac{1}{\sqrt{2\pi\sigma^2}} \cdot \exp\left(-\frac{y^2}{2\sigma^2}\right).
$$
Note that due to sparsity, we have that
\begin{equation}\label{eq:appendix_sg_comp_2}
    \E_x\left[ x \cdot P(\mu x + \sigma g = y|x)\right] = p\cdot \E_{x\sim\mathcal{N}(0,1/p)}\left[x \cdot P(\mu x + \sigma g = y|x)\right],
\end{equation}
and, in this case, we conclude that
$$
\mu x + \sigma g | x \sim \mathcal{N}(\mu x, \sigma^2).
$$
Thus, the RHS of \eqref{eq:appendix_sg_comp_2} is a Gaussian integral, which is straight-forward to calculate by ``completing a square''. The computation gives
$$
\E_{x\sim\mathcal{N}(0,1/p)}\left[x \cdot P(\mu x + \sigma g = y|x)\right] = \sqrt{\frac{p}{2\pi}}
\cdot \mu y \cdot \exp\left(-\frac{py^2}{2(\mu^2+p\sigma^2)}\right) \cdot \frac{1}{(\mu^2 + p\sigma^2)^{3/2}}.
$$
Note that, when $p=1$, i.e., $\x$ is an isotropic Gaussian vector, $f^*$ is just a rescaling by a constant factor, i.e., $f^*(y) = \mathrm{const(\mu, \sigma)} \cdot y$.

\paragraph{Sparse Laplace.} The sparse Laplace distribution with sparsity level $(1-p)$ has the following law 
\begin{equation}\label{eq:sLap}
    (1 - p) \cdot \delta_0 + p \cdot \sqrt{\frac{p}{2}} \cdot \exp\left(-\sqrt{2p} \cdot |x|\right),
\end{equation}
where $\delta_0$ stands for the delta distribution centered at $0$. The scaling for different $p$ is chosen to ensure a unit second moment.

First, we derive the expression for the conditional expectation for $p = 1$. For $p\neq 1$ we elaborate later how a simple change of variables allows to obtain closed-form expressions of the corresponding expectations via the case $p=1$. For $p=1$, the denominator in \eqref{eq:appendix_sg_comp_1} is equivalent to
\begin{align}\label{eq:appendix_sg_comp_3}
    \int_{\mathbb{R}} p(x) p(\mu x + \sigma g = y|x) \mathrm{d}x = \frac{1}{\sqrt{4\pi\sigma^2}}\int_{\mathbb{R}} \exp\left(-\sqrt{2}\cdot |x|\right) \exp\left(-\frac{(y-\mu x)^2}{2\sigma^2}\right)\mathrm{d}x.
\end{align}
By considering two cases, i.e., $x < 0$ and $x \geq 0$, for the limits of integration and for each of them ``completing a square'', we obtain
\begin{align*}
    &\int_{\mathbb{R}_{+}} \exp\left(-\sqrt{2}\cdot x\right) \exp\left(-\frac{(y-\mu x)^2}{2\sigma^2}\right)\mathrm{d}x = \left(1 + \mathrm{erf}\left(\frac{\sqrt{2}\mu y -2\sigma^2}{2\mu\sigma}\right)\right) \cdot \exp\left(\frac{\sigma^2-\sqrt{2}\mu y }{\mu^2}\right) \cdot \sqrt{\frac{\pi}{2}} \cdot \frac{\sigma}{\mu},
    \\
    &\int_{\mathbb{R}_{-}} \exp\left(\sqrt{2}\cdot x\right) \exp\left(-\frac{(y-\mu x)^2}{2\sigma^2}\right)\mathrm{d}x = \mathrm{erfc}\left(\frac{\sqrt{2}\mu y +2\sigma^2}{2\mu\sigma}\right) \cdot \exp\left(\frac{\sigma^2+\sqrt{2}\mu y }{\mu^2}\right) \cdot \sqrt{\frac{\pi}{2}} \cdot \frac{\sigma}{\mu},
\end{align*}
where $\mathrm{erf}(\cdot)$ stands for the Gaussian error function, and $\mathrm{erfc}(\cdot)$ for its complement. For the case of $p\neq1$, we get that the RHS of \eqref{eq:appendix_sg_comp_3} becomes
$$
(1-p) \cdot \frac{1}{\sqrt{2\pi\sigma^2}}\cdot\exp\left(-\frac{y^2}{2\sigma^2}\right) + p\cdot \sqrt{\frac{p}{{4\pi\sigma^2}}} \cdot \int_{\mathbb{R}} \exp\left(-\sqrt{2p}\cdot |x|\right) \exp\left(-\frac{(y-\mu x)^2}{2\sigma^2}\right)\mathrm{d}x.
$$
The change in normalization constant of the second term is then trivial. For the integral itself, consider the change of variables $\tilde{x} = x \cdot \sqrt{p}$:
\begin{align*}
    \int_{\mathbb{R}} \exp\left(-\sqrt{2p}\cdot |x|\right) \exp\left(-\frac{(y-\mu x)^2}{2\sigma^2}\right)\mathrm{d}x &=  \frac{1}{\sqrt{p}} \cdot \int_{\mathbb{R}} \exp\left(-\sqrt{2}\cdot |\tilde{x}|\right) \exp\left(-\frac{(y-\frac{\mu}{\sqrt{p}} \cdot\tilde{x})^2}{2\sigma^2}\right)\mathrm{d}\tilde{x} \\
    &= \frac{1}{\sqrt{p}} \cdot \int_{\mathbb{R}} \exp\left(-\sqrt{2}\cdot |\tilde{x}|\right) \exp\left(-\frac{(y-\tilde{\mu} \cdot\tilde{x})^2}{2\sigma^2}\right)\mathrm{d}\tilde{x},
\end{align*}
which is exactly the previous integral in \eqref{eq:appendix_sg_comp_3} but with $\tilde{\mu} = \mu / \sqrt{p}$ and an additional scaling factor in front.

Consider the numerator of \eqref{eq:appendix_sg_comp_1} for $p = 1$. For this case, the computation reduces to evaluating:
\begin{align}\label{eq:appendix_sg_comp_4}
    \int_{\mathbb{R}} x \cdot p(x) p(\mu x + \sigma g = y|x) \mathrm{d}x = \frac{1}{\sqrt{4\pi\sigma^2}}\int_{\mathbb{R}} x \cdot \exp\left(-\sqrt{2}\cdot |x|\right) \exp\left(-\frac{(y-\mu x)^2}{2\sigma^2}\right)\mathrm{d}x.
\end{align}
Reducing to cases again and ``completing a square'' gives
\begin{align*}
     &\int_{\mathbb{R}_{+}} x\cdot \exp\left(-\sqrt{2}\cdot x\right) \exp\left(-\frac{(y-\mu x)^2}{2\sigma^2}\right)\mathrm{d}x \\ &= \exp\left(-\frac{y^2}{2\sigma^2}\right) \cdot \left[ 
     \frac{\sigma^2}{\mu^2} 
      + \frac{\sqrt{\pi}\sigma \cdot (\sqrt{2} \mu y - 2 \sigma^2) \cdot e^{\frac{(\mu y - \sqrt{2} \sigma^2)^2}{2\mu^2\sigma^2}} \cdot \left(1 + \mathrm{erf}\left(\frac{y}{\sqrt{2}\sigma}-\frac{\sigma}{\mu}\right)\right)}{2\mu^3}\right],
    \\
    &\int_{\mathbb{R}_{-}} x\cdot\exp\left(\sqrt{2}\cdot x\right) \exp\left(-\frac{(y-\mu x)^2}{2\sigma^2}\right)\mathrm{d}x \\
    &= \exp\left(-\frac{y^2}{2\sigma^2}\right) \cdot \left[-\frac{\sigma^2}{\mu^2}
    + \frac{\sqrt{\pi}\sigma \cdot (\sqrt{2} \mu y + 2 \sigma^2) \cdot e^{\frac{(\mu y + \sqrt{2} \sigma^2)^2}{2\mu^2\sigma^2}} \cdot \mathrm{erfc}\left(\frac{y}{\sqrt{2}\sigma}+\frac{\sigma}{\mu}\right)}{2\mu^3}
    \right].
\end{align*}
The derivation for the case $p\neq 1$ can be obtained analogously, by noting that \eqref{eq:appendix_sg_comp_4} in this case is written as
\begin{align*}
    p\cdot \sqrt{\frac{p}{{4\pi\sigma^2}}} \cdot \int_{\mathbb{R}} x \cdot \exp\left(-\sqrt{2p}\cdot |x|\right) \exp\left(-\frac{(y-\mu x)^2}{2\sigma^2}\right)\mathrm{d}x.
\end{align*}

\paragraph{Sparse Rademacher.} The sparse Rademacher distribution with sparsity level $(1-p)$ has the following law
$$
(1-p) \cdot \delta_0 + \frac{p}{2} \cdot \left(\delta_{1/\sqrt{p}} + \delta_{-1/\sqrt{p}}\right).
$$
The denominator in \eqref{eq:appendix_sg_comp_1} reduces to
$$
(1-p) \cdot \frac{1}{\sqrt{2\pi\sigma^2}}\cdot\exp\left(-\frac{y^2}{2\sigma^2}\right) + \frac{p}{2} \cdot \frac{1}{\sqrt{2\pi\sigma^2}} \cdot \left[\exp\left(-\frac{(y-\mu/\sqrt{p})^2}{2\sigma^2}\right) + \exp\left(-\frac{(y+\mu/\sqrt{p})^2}{2\sigma^2}\right)\right].
$$
Moreover, it is easy to see that the enumerator of \eqref{eq:appendix_sg_comp_1} reduces to
$$
\frac{\sqrt{p}}{2}\cdot\left[\exp\left(-\frac{(y-\mu/\sqrt{p})^2}{2\sigma^2}\right) - \exp\left(-\frac{(y+\mu/\sqrt{p})^2}{2\sigma^2}\right)\right].
$$

\section{Experimental details and additional numerical results}

\subsection{Numerical setup}\label{app:num-setup}

\paragraph{Activation function and reparameterization of the weight matrix $\B$.} Since the sign activation has derivative zero almost everywhere, it is not directly suited for gradient-based optimization. To overcome this issue for SGD training of the models described in the main body, we use a ``straight-through'' (see for example \cite{yin2019understanding}) approximation of it. In details, during the forward pass the activation of the network $\sigma(\cdot)$ is treated as a sign activation. However, during the backward pass (gradient computation) the derivatives are computed as if instead of $\sigma(\cdot)$ its relaxed version is used, namely, the tempered hyperbolic tangent:
$$
\sigma_{\tau}(x) = \tanh\left(\frac{x}{\tau}\right).
$$
We also note that such approximation is pointwise consistent except zero: 
$$
\lim_{\tau\rightarrow0} = \sigma(x), \quad \forall \ \x \in \mathbb{R} \setminus \{0\}.
$$
For the experiments we fix the temperature $\tau$ to the value of $0.1$. Refining the approximation further, i.e., making $\tau$ smaller, does not affect the end result, but it makes numerics a bit less stable due to the increased variance of the derivative.

To ensure consistency of the ``straight-though'' approximation, we enforce the condition $\B_{i,:} \in \mathbb{S}^{d-1}$ via a simple differentiable reparameterization. Let $\B \in \mathbb{R}^{n\times d}$ be trainable network parameters, then
$$
\hat{\B}_{i,:} = \frac{\B_{i,:}}{\|\B_{i,:}\|_2}.
$$
It should be noted that it is not clear whether this constraint is necessary, since during the forward pass we use directly $\sigma(\cdot)$, which is agnostic to the row scaling of $\B$.

\paragraph{Augmentation and whitening.} For the natural image experiments in Figures \ref{fig:noniso_exps}, \ref{fig:noniso_exps_l} and \ref{fig:noniso_exps_2}, we use data augmentation to bring the amount of images per class to the initial dataset scale. This step is crucial to simulate the minimization of the population risk and not the empirical one, when the number of samples per class is insufficient. We augment each image $15$ times for CIFAR-10 data and $10$ times for MNIST data. We note that the described amount of augmentation is sufficient:  increasing it further does not change the results of the numerical experiments and only increases computational cost.

The whitening procedure corresponds to the matrix multiplication of each image by the inverse square root of the empirical covariance of the data. This is done to ensure that the data is isotropic (to be closer to the i.i.d.\ data assumption needed for the theoretical analysis). More formally, let $\X \in \mathbb{R}^{\mathrm{n_{\rm samples}}\times d}$ be the augmented data that is centered, i.e., the data mean is subtracted. Its empirical covariance is then given by
$$
\hat{\bSigma} = \frac{1}{\mathrm{n_{\rm samples}} - 1} \cdot \sum_{i=1}^{\mathrm{n_{\rm samples}}} \X_{i,:} \X_{i,:}^\top.
$$
In this view, the whitened data $\hat{\X}\in \mathbb{R}^{\mathrm{n_{\rm samples}}\times d}$ is obtained from the initial data $\X$ as follows
$$
\hat{\X}_{i,:} = \hat{\bSigma}^{-\frac{1}{2}} \X_{i,:},
$$
where $\X_{i,:}$ defines the $i$-th data sample.

\subsection{Phase transition and staircase in the learning dynamics for the autoencoder in \eqref{eq:linear_decoding}}\label{appendix:pt_sprase_gaussian_mixture}

First, we provide an additional numerical simulation similar to the one in Figure \ref{fig:sgd_rademacher} for the case of non-sparse Rademacher data, i.e., $p=1$. Since condition \eqref{eq:phase_transition} holds, we expect the minimizer to be a permutation of the identity, and the corresponding SGD dynamics to experience a staircase behaviour, as discussed in Section \ref{sec:4}. Namely, the SGD algorithm first finds a random rotation that achieves Gaussian performance (indicated by the orange dashed line). Next, it searches a direction towards a sparse solution given by a permutation of the identity, and the corresponding loss remains at the plateau. Finally, the correct direction is found, and SGD quickly converges to the optimal solution.

\begin{figure}[H]
    \centering
    \subfloat{\includegraphics[width=0.52\columnwidth]{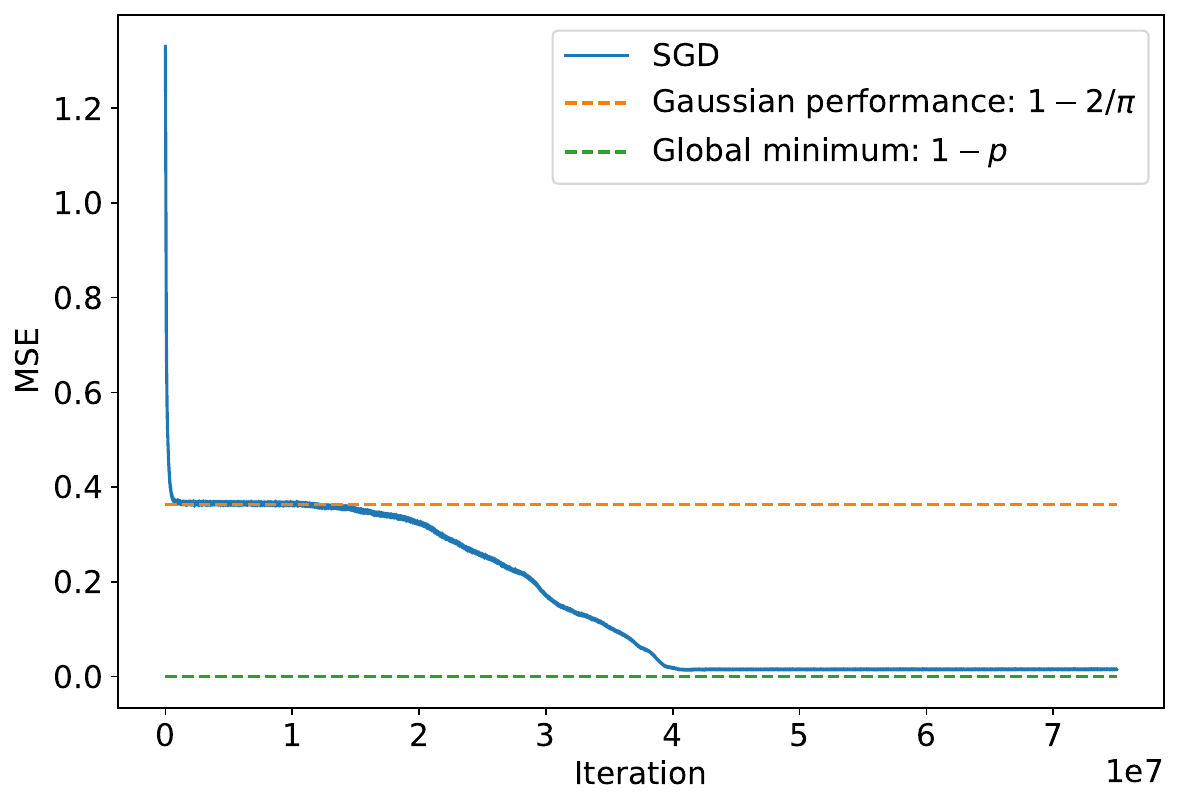}}\hspace{0.2em}
\caption{Compression of Rademacher data ($p=1$) via the autoencoder in \eqref{eq:linear_decoding}. We set $d=200$ and $r=1$. The MSE is plotted as a function of the number of iterations, and it displays a staircase behavior.}\label{fig:sgd_rademacher_2}
\end{figure}

Next, we consider the compression of $\x$ with i.i.d.\ components distributed according to the following sparse mixture of Gaussians:
$$
x_i \sim p \cdot \left(\frac{1}{2} \cdot \mathcal{N}\left(1, \frac{1-p}{p}\right) + \frac{1}{2} \cdot \mathcal{N}\left(-1, \frac{1-p}{p}\right)\right) + (1-p) \cdot \delta_0.
$$
It is easy to verify that $\E [x_i^2] = 1$. In order to compute the transition point we need to access the first absolute moment of $x_i$, i.e., $\E |x_i|$. Using the result in \cite{winkelbauer2012moments}, we are able to claim that
\begin{equation}\label{eq:bulbazaur_abs_moment}
   \E_{x \sim \mathcal{N}(\pm 1, \sigma^2)} |x| = \sigma \sqrt{\frac{2}{\pi}} \cdot \Phi\left(-\frac{1}{2}, \frac{1}{2}, -\frac{1}{2\sigma^2}\right), 
\end{equation}
where $\Phi(a,b,c)$ stands for Kummer's confluent hypergeometric function:
$$
\Phi(a,b,c) = \sum_{n=1}^\infty \frac{a^{\overline{n}}}{b^{\overline{n}}} \cdot \frac{c^n}{n!},
$$
with $x^{\overline{n}}$ denoting the rising factorial, i.e.,
$$
x^{\overline{n}} = z \cdot (z+1) \cdot \dots \cdot (z+n-1), \quad n \in \mathbb{N}_{0}.
$$
We use \texttt{scipy.special.hyp1f1} to evaluate numerically $\Phi\left(-\frac{1}{2}, \frac{1}{2}, -\frac{1}{2\sigma^2}\right)$, where $\sigma^2 = (1-p) / p$. Likewise, to find $p_{\mathrm{crit}}$
at which
$
\E |x_i| = \sqrt{\frac{2}{\pi}}
$
we rely on numerics. The results are presented in Figure \ref{fig:spare_mixture_of_gaussians_phase_transition}.

\begin{figure*}[h]
    \centering
    \subfloat{\includegraphics[width=0.402\columnwidth]{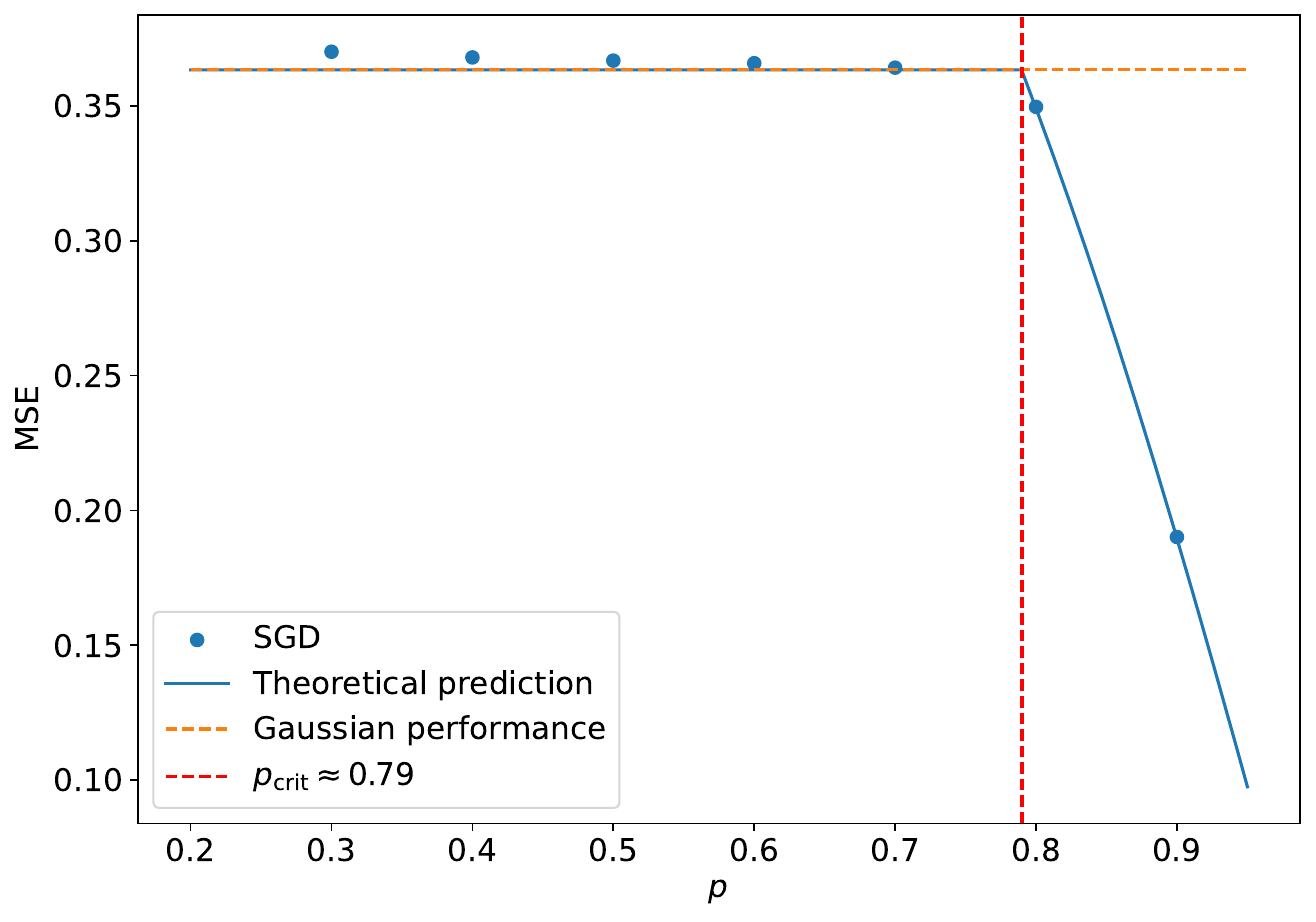}}
    \hspace{-0.1em}
    \subfloat{\includegraphics[width=0.285\columnwidth]{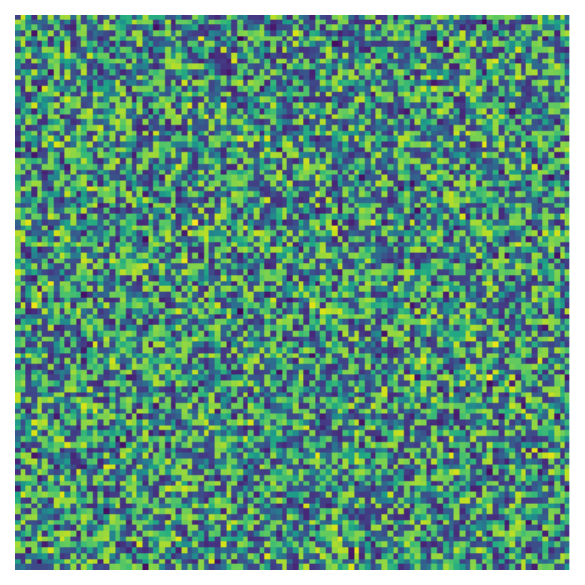}}
    \subfloat{\includegraphics[width=0.285\columnwidth]{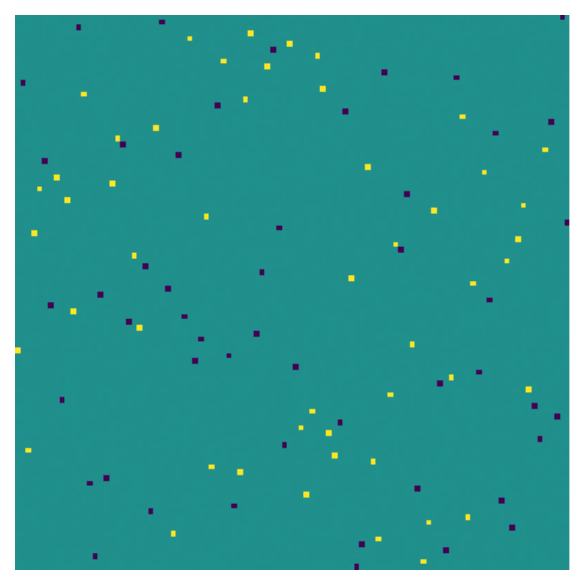}}\hspace{0.2em}
\vspace{-0.5em}\caption{Compression of data whose distribution is given by a sparse mixture of Gaussians via the autoencoder in \eqref{eq:linear_decoding}. We set $d=100$ and $r=1$. \emph{Left.} MSE achieved by SGD at convergence, as a function of the sparsity level $p$. The empirical values (dots) match our theoretical prediction (blue line): for $p<p_{\mathrm{crit}}$, the loss is equal to the value obtained for Gaussian data, i.e., $1-2r/\pi$; for $p> p_{\mathrm{crit}}$, the loss is smaller, and it is equal to $1-r\cdot (\E|x_1|)^2$. \emph{Center.} Encoder matrix $\B$ at convergence of SGD when $p=0.6<p_{\mathrm{crit}}$: the matrix is a random rotation. \emph{Right.} Encoder matrix $\B$ at convergence of SGD when $p=0.9\ge p_{\mathrm{crit}}$. The negative sign in part of the entries of $\B$ is cancelled by the corresponding sign in the entries of $\A$. Hence, $\B$ is equivalent to a permutation of the identity.}\label{fig:spare_mixture_of_gaussians_phase_transition}
\end{figure*}

We remark that the first absolute moment can always be estimated via Monte-Carlo sampling if a functional expression such as \eqref{eq:bulbazaur_abs_moment} is out of reach. We also note that the behaviour of the predicted curve after the transition point $p_{\mathrm{crit}}$ can be arbitrary. In particular, it is not always linear like in the case of sparse Rademacher data in Figure \ref{fig:rademacher_phase_transition}. For instance, in the case of the sparse Gaussian mixture of Figure \ref{fig:spare_mixture_of_gaussians_phase_transition}, the shape is clearly of non-linear nature.

\begin{figure}[H]
    \centering
    \subfloat{\includegraphics[width=0.52\columnwidth]{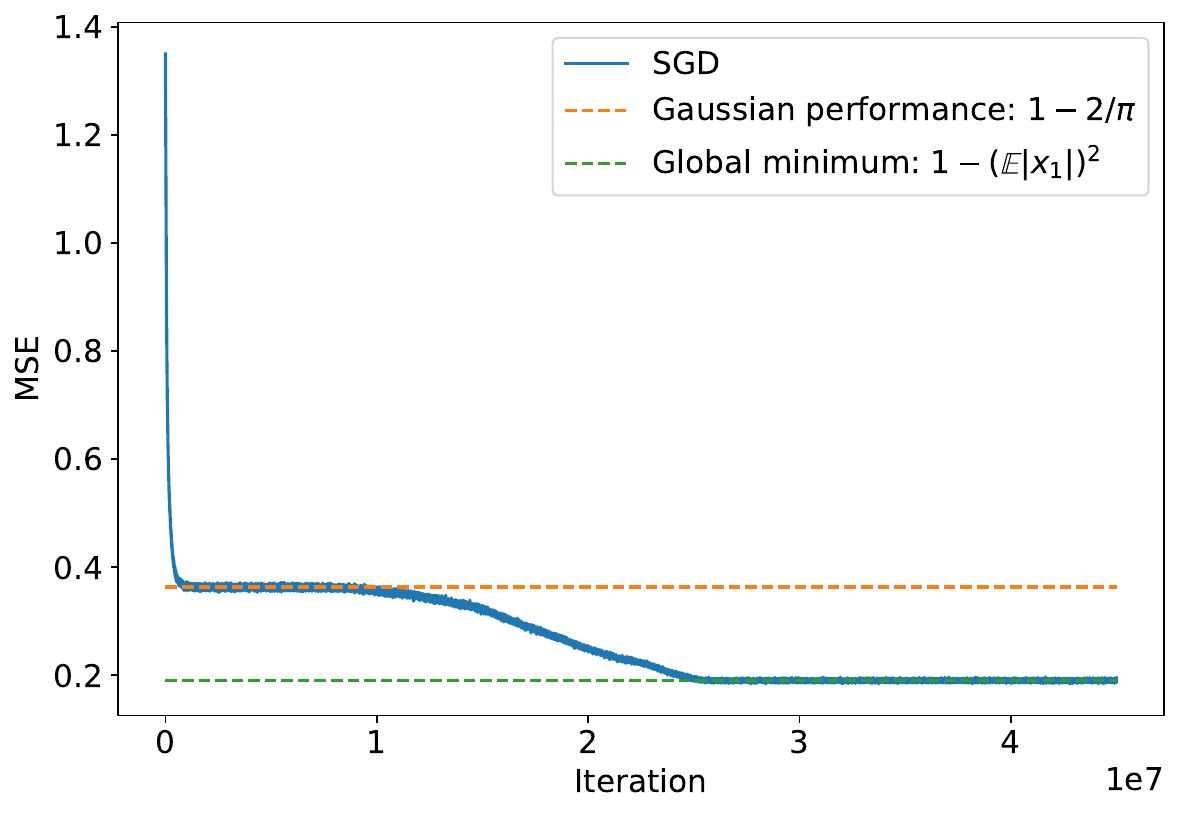}}\hspace{0.2em}
\caption{Compression of data whose distribution is given by a sparse mixture of Gaussians via the autoencoder in \eqref{eq:linear_decoding}. We set $d=100$, $r=1$, and $p=0.9$. The MSE is plotted as a function of the number of iterations and, as $p> p_{\mathrm{crit}}$, it displays a staircase behavior.}\label{fig:sgd_pt_sparse_gauss_2}
\end{figure}

In Figure \ref{fig:sgd_pt_sparse_gauss_2}, we provide an experiment similar to that of Figure \ref{fig:sgd_rademacher}, but for the compression of a sparse mixture of Gaussians with $p=0.9$ at $r=1$. We can clearly see that Figure \ref{fig:sgd_pt_sparse_gauss_2} again indicates the emergent staircase behaviour of the SGD loss for $p>p_{\mathrm{crit}}$.

\subsection{MNIST experiment}\label{appendix:masked_mnist}

In this subsection, we provide additional numerical evidence complementing the results presented in Figure \ref{fig:noniso_exps}. Namely, we provide a similar evaluation on Bernoulli-masked whitened MNIST data. As for the experiment in Figure \ref{fig:noniso_exps}, the sparsity level $p$ is set to $0.7$.
\begin{figure}[h]
\hspace{4em}\begin{tabular}{@{}cc@{}}
    \raisebox{-0.895\height}{\includegraphics[width=0.6\textwidth]{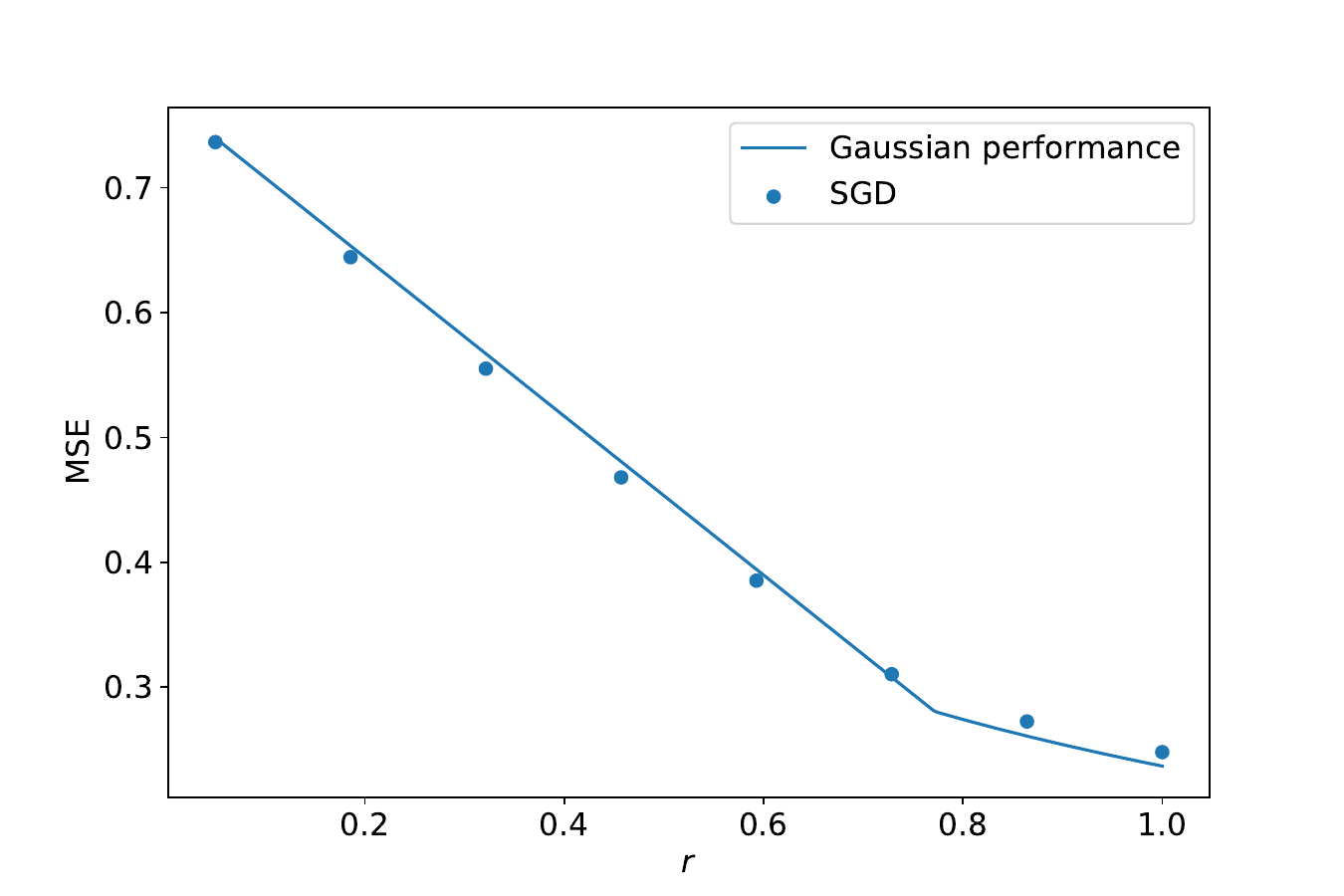}} & 
    \begin{tabular}[t]{@{}cc@{}}
        \raisebox{-\height}{\includegraphics[width=0.17\textwidth]{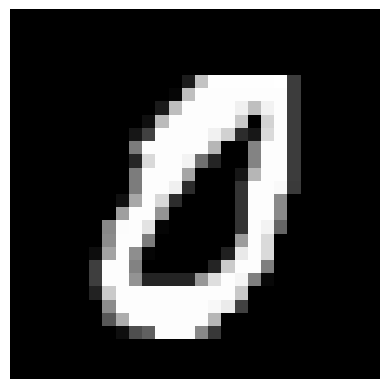}} &  \\[1.8cm]
        \raisebox{-\height}{\includegraphics[width=0.17\textwidth]{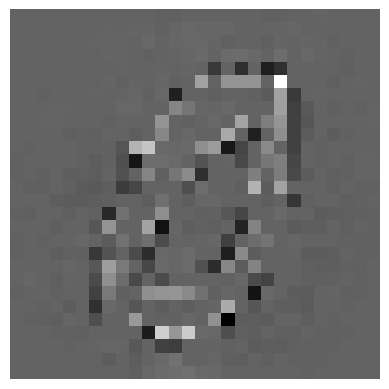}} & 
    \end{tabular}
\end{tabular}
\vspace{-0em}\caption{Compression of masked and whitened MNIST images that correspond to digit ``zero'' via the two-layer autoencoder in \eqref{eq:linear_decoding}. First, the data is whitened so that it has identity covariance (as in the setting of Theorem \ref{thm:GD-min-sparse-body}). Then, the data is masked by setting each pixel independently to $0$ with probability $p=0.7$. An example of an original image is on the top right, and the corresponding masked and whitened image is on the bottom right. The SGD loss at convergence (dots) matches the solid line, which corresponds to the prediction in \eqref{eq:gaussian_val} for the compression of standard Gaussian data (with no sparsity).\vspace{-5mm}}\label{fig:noniso_exps_2} .
\end{figure}

Note that the eigen-decomposition of the covariance of MNIST data has zero eigenvalues. In this case, we need to apply the lower bound from \cite{shevchenko2023fundamental} that accounts for a degenerate spectrum. The corresponding result is stated in Theorem {\color{mydarkblue} 5.2} of \cite{shevchenko2023fundamental}. In particular, the number of zero eigenvalues $n_0$ is equal to $179$, which means that at the value of the compression rate $r$ given by
$$
r = \frac{d - n_0}{d} = \frac{28^2-179}{28^2} \approx 0.77
$$
the derivative of the lower bound experiences a jump-like behavior, as described in 
\cite{shevchenko2023fundamental}.

\subsection{CIFAR-10: Laplace approximation of pixel distribution}\label{appendix:laplace_approx}

\begin{figure}[h!]
    \centering
    \subfloat{\includegraphics[width=0.48\columnwidth]{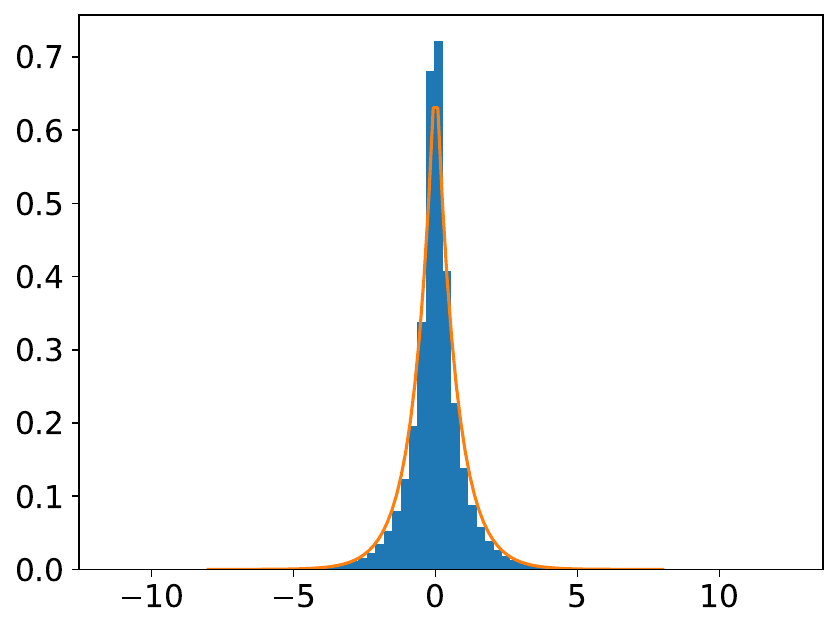}}
\vspace{-3mm}\caption{Empirical distribution of whitened CIFAR-10 image pixels (blue histogram), and its approximation via a Laplace distribution with unit second moment (orange curve).\vspace{-3mm}}\label{fig:nlaplace_approx_appendix}
\end{figure}

Figure \ref{fig:nlaplace_approx_appendix} demonstrates the quality of the Laplace approximation for whitened CIFAR-10 images. Namely, we note that 
the empirical distribution of the image pixels after whitening is well approximated by a Laplace random variable with unit second moment.

\subsection{Provable benefit of nonlinearities for the compression of  sparse Gaussian data}\label{app:prov}

\begin{figure*}[h]
    \centering
    \subfloat{\includegraphics[width=0.55\columnwidth]{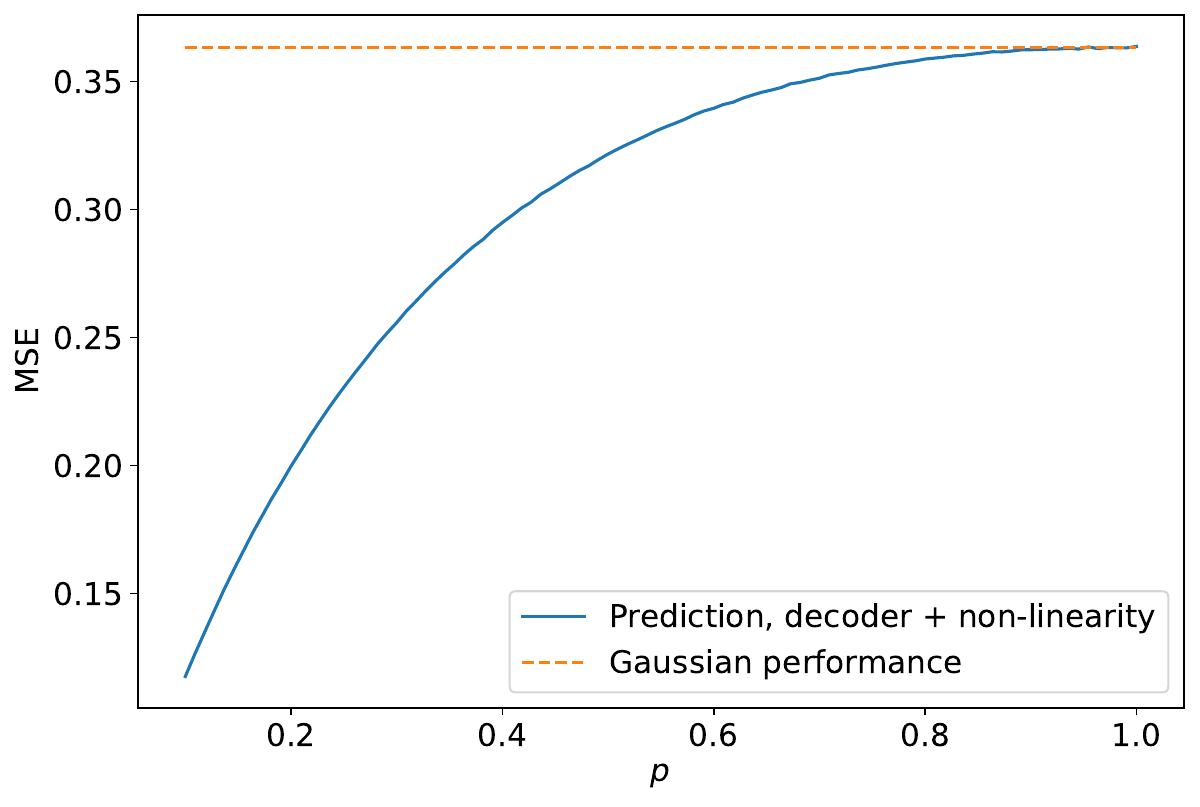}}
\vspace{-0.5em}\caption{Compression of sparse Gaussian data. We set $r=1$. The solid blue line corresponds to the MSE in  \eqref{eq:1RIGAMP} with $f=f^*$ (defined in \eqref{eq:pmean}), for different values of $p$; the dashed orange line corresponds to the Gaussian performance in \eqref{eq:gaussian_val}, which is achieved by the autoencoder in \eqref{eq:linear_decoding}.\vspace{-0mm}}\label{fig:MSEcomp} 
\end{figure*}

Figure \ref{fig:MSEcomp} considers the compression of sparse Gaussian data, and it shows that the MSE achieved by the autoencoder in \eqref{eq:linear_decoding_denoising} with the optimal choice of $f$ (namely, the RHS of \eqref{eq:1RIGAMP} with $f=f^*$) is strictly lower than the MSE \eqref{eq:gaussian_val} achieved by the autoencoder in \eqref{eq:linear_decoding}, for any sparsity level $p\in (0, 1)$. The conditional expectation $\E[x_1|\mu x_1 + \sigma g]$ (cf.\ the definition of $f^*$ in \eqref{eq:pmean}) is computed numerically via a Monte-Carlo approximation.

\subsection{Phase transition and staircase in the learning dynamics for the autoencoder in \eqref{eq:linear_decoding_denoising}}\label{appendix:denoiser_pt}

For sparse Rademacher data, the optimal $f^*$ given by \eqref{eq:pmean} is computed explicitly in Appendix \ref{appendix:denoiser_computations} and plotted in Figure \ref{fig:sparse_rademacher_denoiser_function}. We note that functions of the form in  \eqref{eq:parametric_denoiser} are unable to approximate $f^*$ well. Thus, in the experiments we use a different parametric function for $f$ given by the following mixture of hyperbolic tangents:
\begin{equation}\label{eq:sparse_rademacher_parametric_denoiser}
    f(x) = \mathbbm{1}_{x\geq0} \cdot (\gamma_1 \cdot \tanh(\varepsilon_1 \cdot x - \alpha_1) + \beta_1) + \mathbbm{1}_{x<0} \cdot (\gamma_2 \cdot \tanh(\varepsilon_2 \cdot x - \alpha_2) + \beta_2).
\end{equation}

\begin{figure*}[h]
    \centering
    \subfloat{\includegraphics[width=0.50\columnwidth]{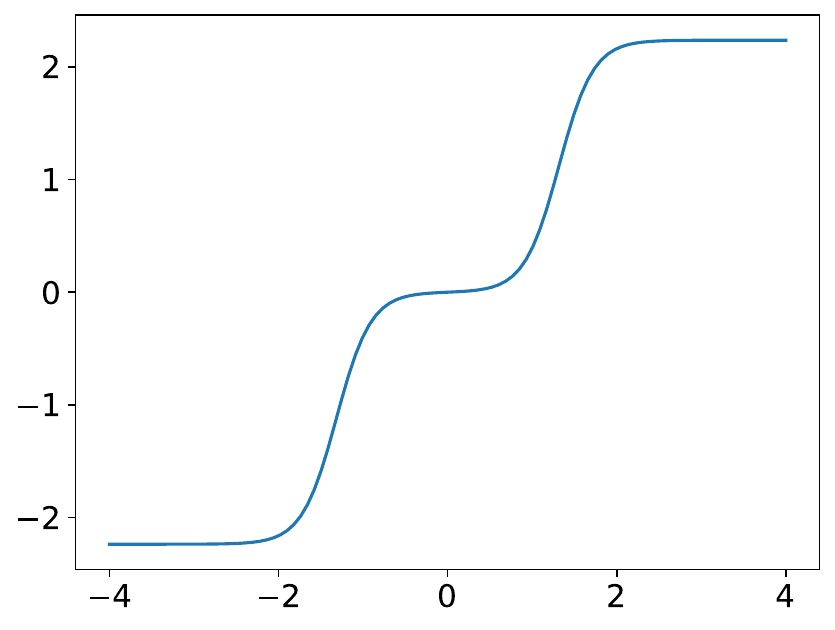}}
\vspace{-0.5em}\caption{Optimal $f^*$ in \eqref{eq:pmean} when $x_1$ is a sparse Rademacher random variable. We set $r=1$ and $p=0.2$.\label{fig:sparse_rademacher_denoiser_function}
}
\end{figure*}

The numerical evaluation of the autoencoder in \eqref{eq:linear_decoding_denoising} with $f$ of the form in \eqref{eq:sparse_rademacher_parametric_denoiser} for the compression of sparse Rademacher data is provided in Figure \ref{fig:rademacher_denoising_phase_transition}. We set $r=1$ and $d=200$. The solid blue line corresponds to the prediction of Proposition \ref{proposition:1}, obtained for random Haar $\B$; the solid orange line corresponds to the prediction of Proposition \ref{proposition:sparse_rademacher_id_denoising}, obtained for $\B$ equal to the identity. The blue dots correspond to the performance of SGD, and they exhibit the transition in the learnt $\B$ from a random Haar matrix ($p<p_{\mathrm{crit}}$) to a permutation of the identity ($p>p_{\mathrm{crit}}$). The critical value $p_{\mathrm{crit}}$ is obtained from the intersection between the blue curve and the orange curve. For all values of $p$, the autoencoder in 
\eqref{eq:linear_decoding_denoising} outperforms the Gaussian MSE \eqref{eq:gaussian_val} (green dashed line) and, hence, it is able to exploit the structure in the data. 

For $p>p_{\mathrm{crit}}$, the staircase behavior of the SGD training dynamics is presented in Figure \ref{fig:sgd_pt_sparse_rademacher_transition}. 

\begin{figure*}[h]
    \centering
    \subfloat{\includegraphics[width=0.492\columnwidth]{pics/sparse_rademacher_denoising_total.pdf}}
    \hspace{-1.6em}
    \subfloat{\includegraphics[width=0.26\columnwidth]{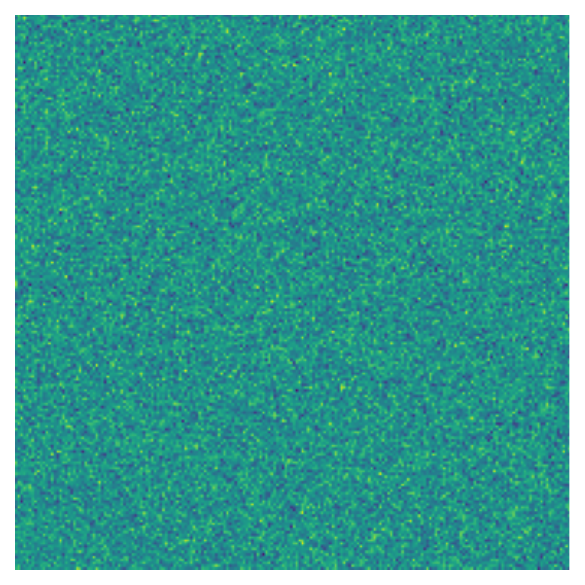}}
    \subfloat{\includegraphics[width=0.26\columnwidth]{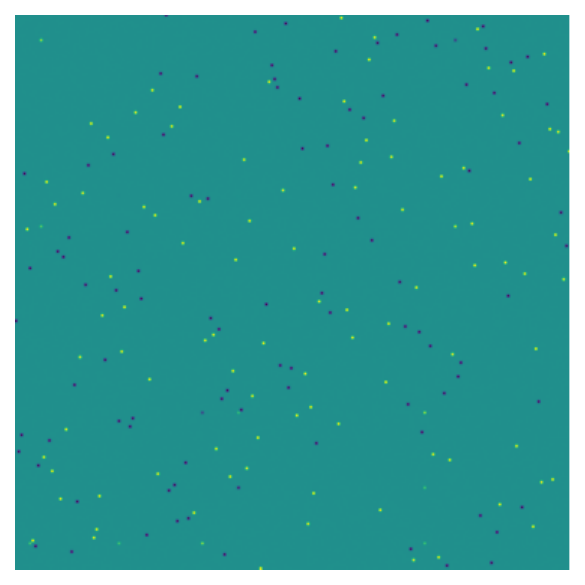}}\hspace{0.2em}
\vspace{-0.5em}\caption{Compression of sparse Rademacher data via the autoencoder in \eqref{eq:linear_decoding_denoising} with $f$ of the form in \eqref{eq:sparse_rademacher_parametric_denoiser}. We set $d=200$ and $r=1$. \emph{Left.} MSE achieved by SGD at convergence, as a function of the sparsity level $p$. The empirical values (dots) match our theoretical prediction (blue line). For $p<p_{\mathrm{crit}}$, the loss is given by Proposition \ref{proposition:1} for $\B$ sampled from the Haar distribution; for $p\ge p_{\mathrm{crit}}$, the loss is given by Proposition \ref{proposition:sparse_rademacher_id_denoising} for $\B$ equal to the identity. \emph{Center.} Encoder matrix $\B$ at convergence of SGD when $p=0.3<p_{\mathrm{crit}}$: the matrix is a random rotation. \emph{Right.} Encoder matrix $\B$ at convergence of SGD when $p=0.7\ge p_{\mathrm{crit}}$. The negative sign in part of the entries of $\B$ is cancelled by the corresponding sign in the entries of $\A$. Hence, $\B$ is equivalent to a permutation of the identity.}\label{fig:rademacher_denoising_phase_transition}
\end{figure*}

\begin{figure}[H]
    \centering
    \subfloat{\includegraphics[width=0.52\columnwidth]{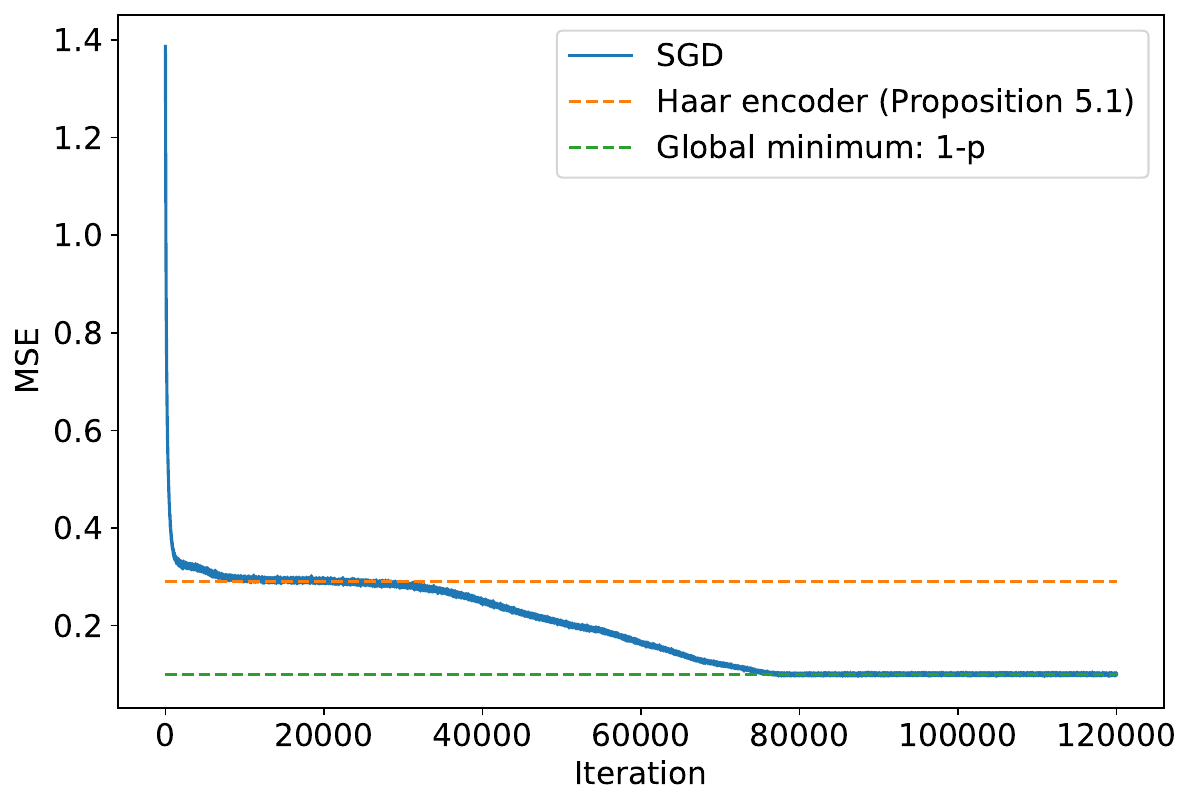}}\hspace{0.2em}
\caption{Compression of sparse Rademacher data via the autoencoder in \eqref{eq:linear_decoding_denoising}. We set $d=200$, $r=1$, and $p=0.9$. The MSE is plotted as a function of the number of iterations and, as $p> p_{\mathrm{crit}}$, it displays a staircase behavior.}\label{fig:sgd_pt_sparse_rademacher_transition}
\end{figure}

\subsection{Discussion on multi-layer decoder}\label{appendix:rigamp_like_nn}

First, let us elaborate on some design points for the network in \eqref{eq:rigamp_like_decoder}. The merging operations $\oplus_2$ and $\oplus_3$ play the role of the correction terms $- \sum_{i=1}^{t-1} \beta_{t,i} \hat{\bx}^i$ and $- \sum_{i=1}^{t} \alpha_{t,i} \hat{\bz}^i$ in the RI-GAMP iterates in \eqref{eq:RIGAMPt}. Furthermore, the composition of $\oplus_3$ and $f_2(\cdot)$ in $\hat{\bx}_2$ approximates taking the posterior mean in  \eqref{eq:RIGAMPt}.
We note that the network \eqref{eq:rigamp_like_decoder} can be generalized to emulate more RI-GAMP iterations, at the cost of additional layers and skip connections (induced by the merging operations $\oplus_k$).

In the rest of this appendix, we discuss how to obtain the orange curve in the right plot of Figure \ref{fig:rigamp_plots}, which corresponds to the Bayes-optimal MSE when $\B$ is sampled from the Haar distribution. This optimal MSE is achieved by the fixed point of the VAMP algorithm proposed in \cite{rangan2019vector}. Thus, we implement the state evolution recursion from \cite{rangan2019vector}, in order to evaluate the fixed point. 

As the specific setting considered here ($\x\sim {\rm SG}_d(p)$, $\B$ a Haar matrix, and a generalized linear model with $\rm sign$ activation) is not considered in \cite{rangan2019vector}, we provide explicit expressions for the recursion leading to the desired MSE. 
\begin{figure*}[t!]
    \centering
    \subfloat{\includegraphics[width=0.55\columnwidth]{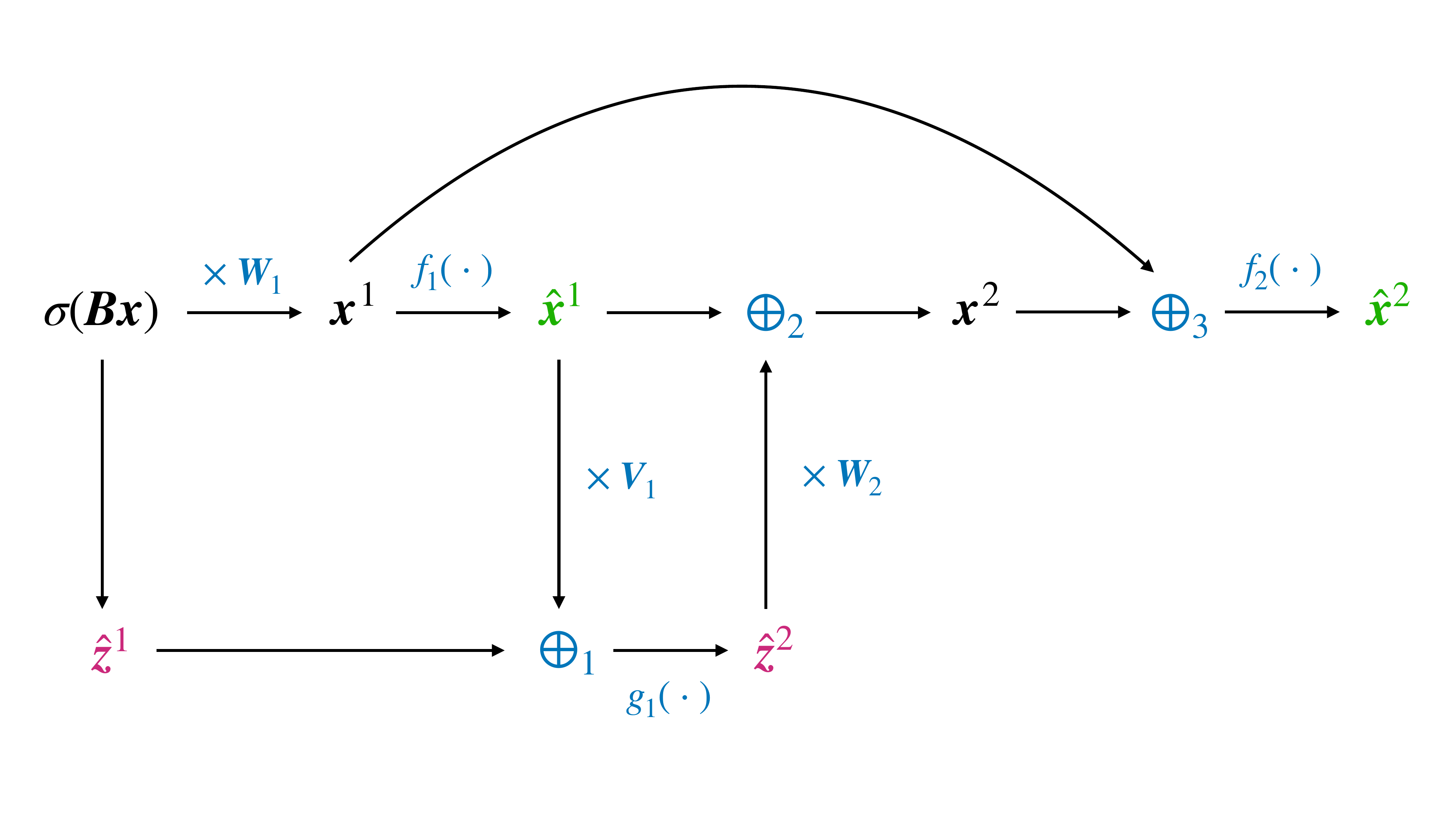}}\hspace{0.2em}
\vspace{-4mm}\caption{Block diagram of the decoder in \eqref{eq:rigamp_like_decoder}.}\label{fig:rigamp_decoder_block_diagram}
\end{figure*}

\paragraph{First state evolution function - $\mathcal{E}_1(\gamma_1)$.} We start with the state evolution function that is equal to the following expected value of the derivative of the conditional expectation
\begin{equation}\label{eq:appendix_e_cond}
    \mathcal{E}_1(\gamma_1) = \E_{R_1}\left[\frac{\partial}{\partial R_1} \mathbb{E}[X|R_1 = X + P]\right], \quad X \sim {\rm SG}_1(p),\quad P \sim \mathcal{N}(0, \gamma_1^{-1}).
\end{equation}
For completeness, we note that the quantity
$$
\frac{\partial}{\partial R_1} \mathbb{E}[X|R_1 = X + P]
$$
is in fact the conditional variance $\mathrm{Var}[X|R_1 = X + P]$ up to a scaling \cite{dytso2020general}, which is related to the optimal MSE.

 Modulo the scalings, the computation of $\mathbb{E}[X|R_1 = X + P]$ is similar to the computation performed in Section \ref{appendix:denoiser_computations}. For brevity, we just state the final result:
\begin{equation}\label{appendix:vamp_cond}
    \mathbb{E}[X|R_1 = X + P] = \frac{p\cdot\frac{R_1}{\sqrt{2\pi p^{-1}}}\cdot\exp\left(-\frac{pR_1^2}{2(p \gamma_1^{-1} + 1)}\right) \cdot \frac{1}{(p \gamma_1^{-1} + 1)^{3/2}}}{p \cdot \frac{1}{\sqrt{2 \pi (p^{-1} + \gamma_1^{-1})}} \cdot \exp\left(-\frac{pR_1^2}{2(p\gamma_1^{-1} + 1)}\right) + (1-p) \cdot \frac{1}{\sqrt{2\pi\gamma_1^{-1}}} \cdot \exp\left(-\frac{R_1^2}{2\gamma_1^{-1}}\right)}:= \frac{E(R_1)}{p(R_1)}.
\end{equation}
Taking the partial derivative in $R_1$ and substituting in \eqref{eq:appendix_e_cond} yields:
\begin{equation}\label{appendix:vamp_cond_prime_1}
\begin{split}
    \mathcal{E}_1(\gamma_1) &= \gamma_1^{-1}\int_{\R}  \frac{\partial}{\partial R_1}\mathbb{E}[X|R_1 = X + P] \cdot p(R_1)\mathrm{d}R_1 = \gamma_1^{-1}\int_{\R} \frac{E'(R_1) p(R_1) - E(R_1)p'(R_1)}{p^2(R_1)} p(R_1)\mathrm{d}R_1 \\
    &= \gamma_1^{-1} \int_{\R} \left(E'(R_1) - E(R_1) \cdot \frac{\partial}{\partial R_1}\log p(R_1)\right) \mathrm{d}R_1.
\end{split}
\end{equation}
We can readily verify that
$$
\int_{\mathbb{R}} E'(R_1)\mathrm{d}R_1 = \lim_{\mathrm{ext}\rightarrow\infty}E(R_1)\Bigg|_{-\mathrm{ext}}^{+\mathrm{ext}} = 0.
$$
An integration by parts for the remaining term in \eqref{appendix:vamp_cond_prime_1} gives:
\begin{equation}\label{appendix:vamp_cond_prime_2}
    \mathcal{E}_1(\gamma_1) = \gamma_1^{-1}\lim_{\mathrm{ext}\rightarrow\infty}E(R_1)\log p(R_1)\Bigg|_{-\mathrm{ext}}^{+\mathrm{ext}} - \gamma_1^{-1}\int_{\mathbb{R}} E'(R_1) \log p(R_1) \mathrm{d}R_1 = - \gamma_1^{-1}\int_{\mathbb{R}} E'(R_1) \log p(R_1) \mathrm{d}R_1.
\end{equation}
The RHS of \eqref{appendix:vamp_cond_prime_2} is then evaluated via numerical integration. For completeness, the derivative $E'(R_1)$ has the following form:
\begin{equation*}
    \begin{split}
        E'(R_1) &= p\cdot\frac{1}{\sqrt{2\pi p^{-1}}}\cdot\exp\left(-\frac{pR_1^2}{2(p \gamma_1^{-1} + 1)}\right) \cdot \frac{1}{(p \gamma_1^{-1} + 1)^{3/2}} \\&- 
        p^2\cdot\frac{R_1^2}{\sqrt{2\pi p^{-1}}}\cdot\exp\left(-\frac{pR_1^2}{2(p \gamma_1^{-1} + 1)}\right) \cdot \frac{1}{(p \gamma_1^{-1} + 1)^{5/2}}.
    \end{split}
\end{equation*}

\paragraph{Second state evolution function - $\mathcal{E}_2(\tau_2, \gamma_2)$.} This function is defined in terms of spectrum of $\B^\top\B \in \mathbb{R}^{d\times d}$. Namely, for $r\leq 1$, the distribution of the eigenvalues of $\B^\top\B$ obeys the following law
$$
\rho_{S} = r \cdot \delta_1 + (1-r) \cdot \delta_0.
$$
The state evolution function $\mathcal{E}_2(\tau_2,\gamma_2)$ is then defined as follows
$$
\mathcal{E}_2(\tau_2,\gamma_2) := \E_{S \sim \rho_S} \left[\frac{1}{\tau_2 \cdot S^2 + \gamma_2}\right] = r \cdot \frac{1}{\tau_2 + \gamma_2} + (1-r) \cdot \frac{1}{\gamma_2}.
$$

\paragraph{Third state evolution function - $\mathcal{B}_2(\tau_2,\gamma_2)$.} The computation is similar to the case of the second state evolution function. Namely, the third state evolution function is defined as follows:
$$
\mathcal{B}_2(\tau_2,\gamma_2) = \frac{1}{r} \cdot \E_{S \sim \rho_S} \left[\frac{\tau_2 S^2}{\tau_2 S^2 + \gamma_2}\right] = \frac{1}{r} \cdot r \cdot \frac{\tau_2}{\tau_2 + \gamma_2} = \frac{\tau_2}{\tau_2+\gamma_2}.
$$

\paragraph{Fourth state evolution function - $\mathcal{B}_1(\tau_1)$.} The last state evolution function is defined similarly to $\mathcal{E}_1(\gamma_1)$, namely
\begin{equation}\label{appendix:den4_1}
    \mathcal{B}_1(\tau_1) = \E_{P_1,Y}\left[ \frac{\partial}{\partial P_1} \E[Z|P_1,Y]\right].
\end{equation}
Here, $Z\sim \mathcal{N}(0,1)$ has variance one (since the spectrum of $\B$ has unit variance), $Y = \mathrm{sign}(Z)$ and $P_1 = b \cdot Z + a \cdot G$, where $G\sim \mathcal{N}(0,1)$ is independent of $Z$ and
$$
b = 1 - \tau_1^{-1}, \quad a = \sqrt{b \cdot (1-b)}.
$$
The outer expectation in \eqref{appendix:den4_1} is estimated via Monte-Carlo. We now compute the conditional expectation. First note that the following decomposition (depending on the sign of $Y$) holds:
\begin{equation}\label{appendix:den4_2}
    \E [Z|P_1, Y] = \E [Z'|P_1'],
\end{equation}
where $Z' = \mathbbm{1}_{ZY \geq 0} \cdot Z$ and $P'_1 = b\cdot Z'+a \cdot G$. Using Bayes formula, we get that
\begin{equation}\label{appendix:den4_3}
    \E [Z'|P'_1] = \frac{\int_{ZY\geq0} Z \exp\left(-\frac{Z^2}{2}\right) \exp\left(-\frac{(P_1-bZ)^2}{2a^2}\right)\mathrm{d}Z}{\int_{ZY\geq0} \exp\left(-\frac{Z^2}{2}\right) \exp\left(-\frac{(P_1-bZ)^2}{2a^2}\right)\mathrm{d}Z}.
\end{equation}
Completing the square in the exponents gives
\begin{equation*}
    \frac{Z^2a^2 + (P_1-bZ)^2}{2a^2} = \frac{bZ^2 - 2bZP_1 + P_1^2}{2b(1-b)} = \frac{(Z-P_1)^2}{2(1-b)} + \frac{P_1^2}{2b},
\end{equation*}
which after substitution in \eqref{appendix:den4_3} results in
\begin{equation}\label{appendix:den4_4}
    \E [Z'|P'_1] = \frac{\int_{ZY\geq0} Z \exp\left(-\frac{(Z-P_1)^2}{2\tau_1^{-1}}\right) \mathrm{d}Z}{\int_{ZY\geq0} \exp\left(-\frac{(Z-P_1)^2}{2\tau_1^{-1}}\right) \mathrm{d}Z}.
\end{equation}
Note that the denominator of \eqref{appendix:den4_4} is easy to access via the standard Gaussian CDF $\Psi(\cdot)$ as follows
\begin{equation}\label{appendix:den4_denom}
\begin{split}
    \frac{1}{\sqrt{2\pi\tau_1^{-1}}}\int_{ZY\geq0} \exp\left(-\frac{(Z-P_1)^2}{2\tau_1^{-1}}\right) \mathrm{d}Z &= \mathbbm{1}_{Y\geq 0} \cdot \left[1-\Psi\left(-\frac{P_1}{\tau_1^{-1/2}}\right)\right] + \mathbbm{1}_{Y < 0} \cdot \Psi\left(-\frac{P_1}{\tau_1^{-1/2}}\right) \\
    &=\Psi\left(\frac{YP_1}{\tau_1^{-1/2}}\right),
\end{split}
\end{equation}
where for the last equality we use that $\Psi(x) = 1 - \Psi(-x)$ and $Y\in\{-1,+1\}$. For the numerator of \eqref{appendix:den4_4}, we get
\begin{equation}\label{appendix:den4_5}
    \frac{1}{\sqrt{2\pi\tau_1^{-1}}}\int \mathbbm{1}_{YZ\geq 0} \cdot Z\exp\left(-\frac{(Z-P_1)^2}{2\tau_1^{-1}}\right) \mathrm{d}Z.
\end{equation}
Let us denote the PDF of $\mathcal{N}(\mu,\sigma^2)$ by $\rho_{\mu,\sigma^2}$, and use the  shorthand $\rho(\cdot)$ for $\rho_{0,1}(\cdot)$. Note that $\rho_{x,\sigma^2}(0) = \sigma^{-1}\rho(x/\sigma)$. Then, by Stein's identity, we have
$$
\E \left[\mathbbm{1}_{YZ\geq 0} \cdot (Z-P_1)\right] = \tau_1^{-1} \cdot \E [Y\cdot\delta_0(Z)] = Y\tau_{1}^{-1} \cdot \rho_{P_1,\tau_1^{-1}}(0) = Y\tau_{1}^{-1/2} \cdot \rho\left(\frac{P_1}{\tau_1^{-1/2}}\right),
$$
as the weak derivative of $\mathbbm{1}_{YZ\geq 0}$ is well-defined and equal to $Y\cdot\delta_0(Z)$. Noting that similarly to \eqref{appendix:den4_denom}
$$
\E \left[\mathbbm{1}_{YZ\geq 0} \cdot P_1\right] = P_1 \cdot \Psi\left(\frac{YP_1}{\tau_1^{-1/2}}\right),
$$
we conclude that
\begin{equation}\label{appendix:den4_6}
    \eqref{appendix:den4_5} = P_1 \cdot \Psi\left(\frac{YP_1}{\tau_1^{-1/2}}\right) + Y\tau_{1}^{-1/2} \cdot \rho\left(\frac{P_1}{\tau_1^{-1/2}}\right).
\end{equation}
Combining the results gives
\begin{equation}\label{appendix:4_allinall}
    \E[Z_1'|P_1'] = \frac{P_1 \cdot \Psi\left(\frac{YP_1}{\tau_1^{-1/2}}\right) + Y\tau_{1}^{-1/2} \cdot \rho\left(\frac{P_1}{\tau_1^{-1/2}}\right)}{\Psi\left(\frac{YP_1}{\tau_1^{-1/2}}\right)} = P_1 + Y\tau_{1}^{-1/2} \cdot \frac{\rho\left(\frac{P_1}{\tau_1^{-1/2}}\right)}{{\Psi\left(\frac{YP_1}{\tau_1^{-1/2}}\right)}}.
\end{equation}
It now remains to take the derivative in $P_1$. We get that
\begin{equation}
    \begin{split}
        \mathcal{B}_1(\tau_1) = 1 - \frac{YP_1\sqrt{\tau_1} \cdot \rho\left(\frac{P_1}{\tau_1^{-1/2}}\right) \cdot \Psi\left(\frac{YP_1}{\tau_1^{-1/2}}\right) + \rho\left(\frac{P_1}{\tau_1^{-1/2}}\right)^2}{\Psi\left(\frac{YP_1}{\tau_1^{-1/2}}\right)^2},
    \end{split}
\end{equation}
where we used that $Y^2=1$ and that $\frac{\partial}{\partial x}\Psi(x) = \rho(x)$.

\paragraph{State evolution recursion.} At this point, we are ready to present the state evolution recursion, which reads
\begin{equation}
    \begin{split}
        \gamma_{2,k} &= \gamma_{1,k} \cdot \frac{1-\mathcal{E}_1(\gamma_{1,k})}{\mathcal{E}_1(\gamma_{1,k})},\\
        \tau_{2,k} &= \tau_{1,k} \cdot \frac{1 - \mathcal{B}_1(\tau_{1,k})}{\mathcal{B}_1(\tau_{1,k})}, \\
        \gamma_{1,k+1} &= \gamma_{2,k} \cdot \frac{1-\mathcal{E}_2(\tau_{2,k},\gamma_{2,k})}{\mathcal{E}_2(\tau_{2,k},\gamma_{2,k})} = \gamma_{2,k} \cdot \frac{r\cdot\tau_{2,k}}{(1-r)\cdot\tau_{2,k}+\gamma_{2,k}},\\
        \tau_{1,k+1} &= \tau_{2,k} \cdot \frac{1 - \mathcal{B}_2(\tau_{2,k},\gamma_{2,k})}{\mathcal{B}_2(\tau_{2,k},\gamma_{2,k})} = \gamma_{2,k}.
    \end{split}
\end{equation}
The initialization $\gamma_{1,0}$ and $\tau_{1,0}$ can be set to a small strictly positive number. For the experiments, we choose the value of $10^{-6}$.

\paragraph{MSE from the state evolution parameter $\gamma_{1,k+1}$.} The MSE after $k$ steps of the recursion can be accessed via the function previously computed in \eqref{appendix:vamp_cond}. Namely, let $\x \sim {\rm SG}_d(p)$ and $\boldsymbol{r}_1 = \x + \boldsymbol{p}$, where $\boldsymbol{p}$ has i.i.d.\ entries with distribution $ \mathcal{N}(0,\gamma_{1,k+1}^{-1})$. Define
$$
g(\boldsymbol{r}_1) = \E[\x|\r_{1} = \x + \boldsymbol{p}].
$$
By the tower property of the conditional expectation, we claim that the following holds
$$
\E[\E[X|Y] \cdot X] = \E[\E[\E[X|Y] \cdot X| Y]] = \E\left[(\E[X|Y])^2\right],
$$
where we use that $\E[X|Y]$ is measurable w.r.t. $Y$.
Thus, we have that
$$
\E \langle g(\boldsymbol{r}_1), \x \rangle = d \cdot \E \left[(g(\boldsymbol{r}_1)_1)^2\right], 
$$
where $g(\boldsymbol{r}_1)_1$ denotes the first entry of the vector $g(\boldsymbol{r}_1)$.
Finally, the desired MSE after $k$ steps of the recursion is equal to
\begin{equation}\label{eq:vamp_fixed}
    \begin{split}
        d^{-1} \cdot \E \|\x - g(\boldsymbol{r}_1)\|_2^2 = 1 - \E \left[(g(\boldsymbol{r}_1)_1)^2\right].
    \end{split}
\end{equation}
We evaluate \eqref{eq:vamp_fixed} for $k$ large enough, so that the MSE has converged.
For the experiment in Figure \ref{fig:rigamp_plots}, we use $k=15$, as for $k\geq 15$ the MSE value in \eqref{eq:vamp_fixed} is stable.

\end{document}